\documentclass{article} % For LaTeX2e
\usepackage{iclr2027_conference,times}
\iclrfinalcopy 
\usepackage{amsmath,amsfonts,bm}

\def\eqref#1{equation~\ref{#1}}
\def\1{\bm{1}}

\DeclareMathAlphabet{\mathsfit}{\encodingdefault}{\sfdefault}{m}{sl}
\SetMathAlphabet{\mathsfit}{bold}{\encodingdefault}{\sfdefault}{bx}{n}

\usepackage{hyperref}
\usepackage{url}
\usepackage{amsmath,mathtools}
\usepackage{amssymb}
\usepackage{amsthm}
\usepackage{algorithm}
\usepackage{algpseudocode}
\usepackage{booktabs}
\usepackage{tcolorbox}
\tcbset{promptbox/.style={colback=gray!5!white, colframe=blue!75!black, boxrule=0.5pt, arc=2mm, left=1mm, right=1mm, top=1mm, bottom=1mm}}
\usepackage{multirow}
\usepackage{graphicx}
\usepackage{colortbl}
\usepackage{adjustbox}
\usepackage{makecell}
\usepackage{subcaption}
\usepackage{booktabs}
\usepackage{tabularx}
\usepackage{array}
\usepackage{makecell}
\usepackage{xcolor}
\newtheorem{appproposition}{Proposition}[section]

\usepackage{booktabs}
\usepackage{multirow}
\usepackage{graphicx}
\usepackage[table]{xcolor}
\usepackage{booktabs}
\usepackage[table]{xcolor}
\usepackage{siunitx}

\title{EnGRICH: Enhancing Generative Reward Modeling with Critiques from Humans}

\author{
Xuancheng Li\thanks{This work was done during an internship at Tencent.}\textsuperscript{1}
\quad
Beining Wang\textsuperscript{1}
\quad
Haitao Li\textsuperscript{1}
\quad
Heng Wang\textsuperscript{2}
\quad
Yujia Zhou\textsuperscript{1}
\\
\textbf{
Qingyi Pan\textsuperscript{1}
\quad
Blaze Chen
\quad
Yiqun Liu\textsuperscript{1}
\quad
Min Zhang\textsuperscript{1}
\quad
Qingyao Ai\textsuperscript{1}
}
\\[0.6em]
\textsuperscript{1}Department of Computer Science and Technology, Tsinghua University, Beijing, China
\\
\textsuperscript{2}Tencent, Beijing, China
\\
\texttt{lixuanch23@mails.tsinghua.edu.cn}
}

\begin{document}

\maketitle
\lhead{}
% \fancyhead{}
% \renewcommand{\headrulewidth}{0pt}

\begin{abstract}
Generative reward models (GRMs) are important for LLM optimization. Unlike scalar reward models, GRMs generate natural-language critiques alongside preference judgments, providing finer-grained evaluation signals. Their effectiveness depends heavily on critique reliability.
However, existing GRM training typically uses final preference correctness as outcome supervision. Because the preference outcome space is highly constrained, unreliable critiques can still yield correct outcomes and thus be reinforced.
Recent work leverages human critiques for process supervision, but such
critiques are scarce and are often reduced to scalar rewards, leaving
their fine-grained evaluative information underutilized.
We argue that evaluative criteria learned from human critiques can be
generalized to broader outcome-only preference data.
To this end, we propose \textbf{EnGRICH}, a GRM training framework that
pairs the GRM with a training-time MetaCritic learned from a small set
of human critiques.
MetaCritic constructs response-specific rubrics and uses them to evaluate
the evidence coverage and correctness of generated critiques.
The resulting signals provide both process rewards for fine-grained
credit assignment and structured guidance for exploring better critiques.
During GRM training, MetaCritic is further optimized to generalize
human-grounded evaluative criteria to outcome-only data.
At inference, the trained GRM operates
independently.
Experiments across seven reward-model benchmarks show that EnGRICH consistently improves over competitive baselines, while further analyses validate the effectiveness of its core mechanisms.\footnote{We release our code at \url{https://anonymous.4open.science/r/ENGRICH}.}
\end{abstract}

\section{Introduction}
\label{sec:introduction}
Reward models (RMs) play an important role across LLM training and inference
\citep{lambert2025rewardbench,zhang2025generativeverifiersrewardmodeling}.
RMs typically take one of two forms: scalar reward models, which directly assign numerical scores to responses, and generative reward models (GRMs), which compare candidate responses by generating language-based critiques and predicting the preferred response \citep{ye2024scalarrewardmodellearning,lai2026ralignenhancinggenerativereward}.
By verbalizing the evaluation process, GRMs can leverage the reasoning capabilities of pretrained LLMs and provide richer evaluation signals than scalar RMs
\citep{chen2025reasongrmenhancinggenerativereward}.
This advantage is particularly valuable in non-verifiable domains such as creative writing and open-ended instruction following, where evaluation often requires nuanced semantic judgment
\citep{xu2026alternatingreinforcementlearningrubricbased}.

A common approach to training GRMs is reinforcement learning with verifiable rewards (RLVR), which rewards agreement between the model's final judgment and the preference label
\citep{liu2025inference,chen2026rmr1rewardmodelingreasoning}.
In conventional RLVR tasks such as mathematics and coding, outcome rewards provide strong supervision because their large solution spaces make correct solutions difficult to obtain by chance
\citep{wang2026rewardmodelingnaturallanguage}.
However, pairwise GRMs operate over a much more restricted outcome space, where the final judgment reduces to choosing between two responses.
Consequently, a GRM may reach the correct preference judgment even when its critique is unreliable
\citep{lai2026ralignenhancinggenerativereward}.
Such spurious successes can still be reinforced during training and subsequently affect downstream tasks such as policy optimization and candidate selection.
Therefore, reliable GRM training requires supervision of both the final preference judgment and the quality of its supporting critique.

Motivated by this limitation, recent work has explored finer-grained supervision for GRM critiques along two directions.
One line derives process supervision from human critiques: RM-NLHF \citep{wang2026rewardmodelingnaturallanguage} uses critique similarity as a process reward and extends it to outcome-only data, while RC-GRPO \citep{wang2026outcomeaccuracyenoughaligning} rewards rationale consistency with human critiques.
However, these methods reduce human critiques to scalar process rewards, leaving their fine-grained diagnostic information underutilized beyond scoring.
Another line uses explicit rubrics to structure evaluation: Rubric-ARM \citep{xu2026alternatingreinforcementlearningrubricbased} jointly optimizes a rubric generator and judge through alternating RL. However, the rubric generator is supervised only by preference correctness rather than human critiques, and both modules are required at inference.

Prior work highlights the value of human critiques for fine-grained
supervision and rubrics as explicit evaluative criteria.
However, since high-quality human critiques are costly to collect,
only a small fraction of preference data is accompanied by them~\citep{akyürek2023rl4fgeneratingnaturallanguage}.
We therefore seek to learn from these limited critiques what evidence
a well-grounded GRM critique should cover and extend such evaluative
criteria to outcome-only preference data.
Rather than reducing this supervision to scalar rewards alone, we
retain its diagnostic structure, allowing the same human-grounded
process supervision to support both process credit assignment and
diagnosis-guided exploration
\citep{li2026mulferlenhancingreinforcementlearning}.

To this end, we propose \textbf{EnGRICH}
(\textbf{En}hancing \textbf{G}enerative \textbf{R}eward Modeling
w\textbf{I}th \textbf{C}ritiques from \textbf{H}umans), a training
framework that pairs a GRM with a MetaCritic.
MetaCritic is initialized from human critiques to construct
response-specific rubrics and diagnose the evidence coverage and
correctness of GRM critiques.
These diagnoses support GRM optimization in two complementary ways:
they provide process rewards for finer-grained credit assignment and
structured guidance for exploring better critiques.
During GRM training, MetaCritic is further optimized to generalize
human-grounded evaluative criteria to outcome-only data.
MetaCritic is used only during training; inference requires only the
standalone GRM.

Using EnGRICH, we train GRMs on limited human critiques together with
outcome-only preference data and evaluate them across diverse
reward-model benchmarks.
EnGRICH consistently outperforms both outcome-supervised and prior
process-supervised baselines, while further analyses validate the
effectiveness of its process supervision and exploration mechanisms.

\section{Related Work}
\label{sec:related_work}

\paragraph{Generative Reward Models and Fine-Grained Supervision.}
Scalar reward models learn scalar rewards from pairwise preferences
\citep{christiano2017deep,ouyang2022training}.
LLM evaluators such as JudgeLM
\citep{zhu2025judgelmfinetunedlargelanguage},
Prometheus~\citep{kim2024prometheusinducingfinegrainedevaluation},
and Critic-RM~\citep{yu2025criticrm} extend reward modeling with
natural-language critiques, while recent GRMs such as
GenRM~\citep{mahan2024generative},
DeepSeek-GRM~\citep{liu2025inference}, and
RM-R1~\citep{chen2026rmr1rewardmodelingreasoning} generate reasoning before predicting
preferences.
Yet their training still largely relies on final preference correctness,
which may reinforce unreliable critiques that yield correct judgments.
RC-GRPO~\citep{wang2026outcomeaccuracyenoughaligning} directly
derives scalar rationale-consistency rewards from human critiques.
RM-NLHF~\citep{wang2026rewardmodelingnaturallanguage} further learns
a scalar MetaRM to extend critique-based process rewards to
outcome-only data, while R-Align~\citep{lai2026ralignenhancinggenerativereward} similarly
uses a Meta-Judge to assess rationale alignment with gold rationales.
These methods improve process supervision but largely compress
fine-grained diagnostic information into scalar or holistic signals.
Another line uses explicit rubrics to structure evaluation.
Rubric-ARM~\citep{xu2026alternatingreinforcementlearningrubricbased}
jointly learns a rubric generator and a judge, but optimizes rubric
generation through preference correctness without explicit supervision
from human critiques, and requires both modules at inference.
EnGRICH instead learns human-grounded, instance-specific rubrics from
limited human critiques, generalizes them to outcome-only preference
data, and uses the resulting diagnoses both as process rewards and as
feedback for targeted revision.

% RLCER~\citep{sheng2026reinforcingchainofthoughtreasoningselfevolving} uses self-evolving rubrics to provide process supervision for reasoning trajectories.
% Rubric-ARM~\citep{xu2026alternatingreinforcementlearningrubricbased}
% jointly learns a rubric generator and a judge, but optimizes rubric
% generation through preference correctness without explicit supervision
% from human critiques, and requires both modules at inference.
% EnGRICH instead learns human-grounded, instance-specific rubrics from
% limited human critiques, generalizes them to outcome-only preference
% data, and uses the resulting diagnoses both as process rewards and as
% feedback for targeted revision.

\paragraph{Learning from Natural-Language Feedback.}
Our use of fine-grained diagnostic signals to guide evaluation
trajectories is inspired by broader work on learning from
natural-language feedback.
ILF~\citep{scheurer2024traininglanguagemodelslanguage} trains on feedback-conditioned
refinements, while Self-Refine~\citep{madaan2023selfrefineiterativerefinementselffeedback},
Reflexion~\citep{shinn2023reflexionlanguageagentsverbal}, and
RL4F~\citep{akyürek2023rl4fgeneratingnaturallanguage} use critiques or reflections to revise
model outputs.
More recent methods incorporate such refinement into online
reinforcement learning.
Critique-GRPO~\citep{zhang2026critiquegrpoadvancingllmreasoning} learns from both initial
solutions and critique-guided refinements, while
MulFeRL~\citep{li2026mulferlenhancingreinforcementlearning} invokes
verbal feedback on failed rollouts and assigns credit to
verifier-confirmed progress.
These methods primarily improve problem-solving outputs or policies.
We argue that fine-grained supervision can serve two complementary roles
in GRMs: providing scalar process rewards for process credit assignment
and keypoint-level diagnosis for guided exploration.

\begin{figure*}[t]
    \centering
    \includegraphics[
        width=\textwidth
    ]{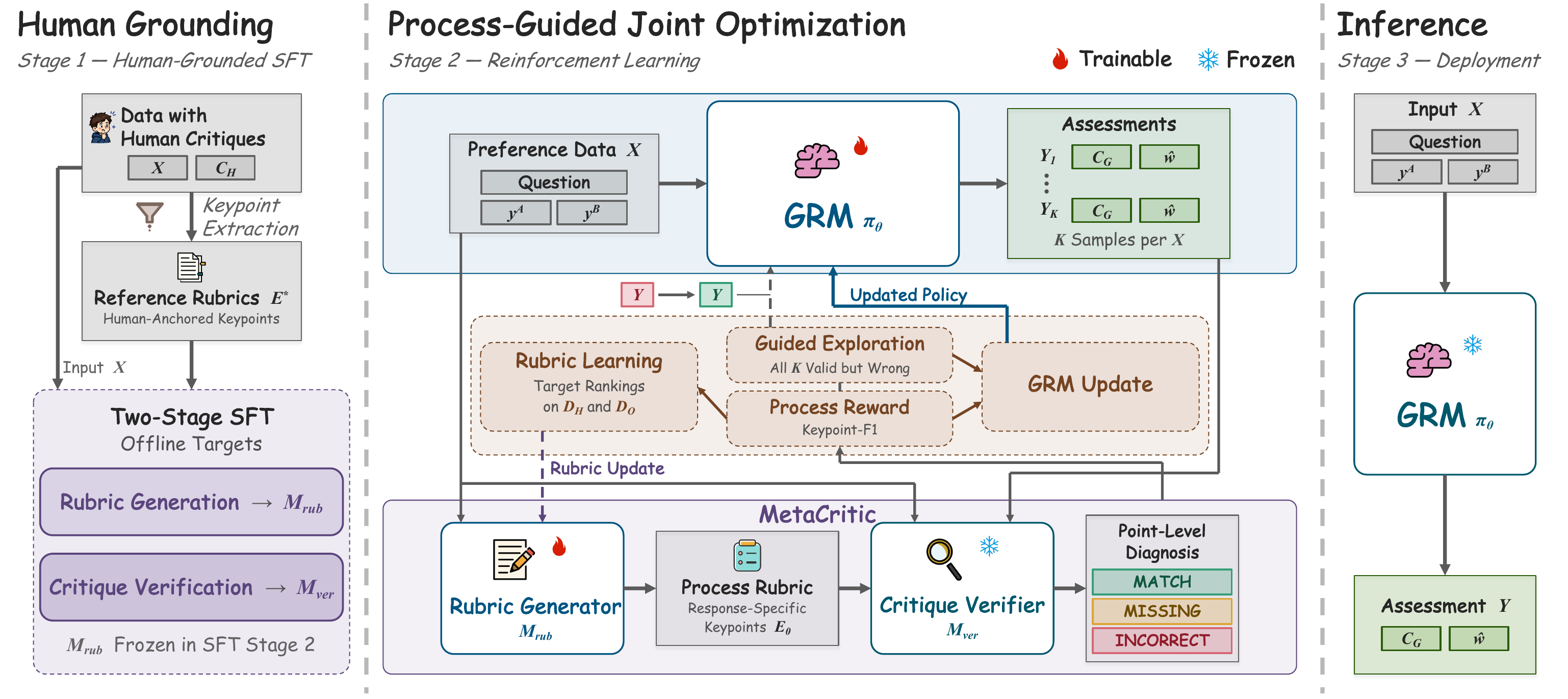}
    \caption{
        Illustration of EnGRICH.
    }
    \label{fig:main}
    \vspace{-1.0em}
\end{figure*}

\section{Method}
\label{sec:method}

\subsection{Problem Formulation}
\label{sec:problem}

A GRM evaluates two candidate responses by generating a comparative
critique followed by a preference judgment.
Given an input $X=(x,y^A,y^B)$, where $x$ is the question and
$y^A,y^B$ are the candidate responses, the GRM $\pi_\theta$
produces an assessment $Y=(C_G,\hat{w})$, where $C_G$ is the generated
critique and $\hat{w}\in\{A,B\}$ is the predicted preference.
Given the cost of collecting human critiques~\citep{akyürek2023rl4fgeneratingnaturallanguage},
we consider a training setting where only a small fraction of the data
is accompanied by human critiques, while the majority provides only
preference labels.
We therefore define a human-critique subset
$\mathcal D_H=\{(X,w,C_H)\}$ and a larger outcome-only subset
$\mathcal D_O=\{(X,w)\}$, where $w$ denotes the gold preference and
$C_H$ is the corresponding human critique.

Our goal is to learn from the limited human critiques the key evaluation
points that characterize a well-grounded critique, and to generalize these
evaluative criteria to supervise GRM training on broader preference data.
This enables the GRM to produce reliable critiques and accurate preference
judgments across diverse tasks.

\subsection{Overview}
\label{sec:overview}

As illustrated in Figure~\ref{fig:main}, \textbf{EnGRICH} turns a small
set of human critiques into reusable process supervision for GRM
optimization.
MetaCritic uses the rubric generator $M_{\mathrm{rub}}$ to construct
response-specific rubrics and the verifier $M_{\mathrm{ver}}$ to
diagnose the evidence coverage and correctness of GRM critiques.
The resulting signals support GRM optimization in two complementary
ways: they provide process rewards for fine-grained credit assignment
and structured guidance for exploring better critiques when outcome
supervision provides limited learning signal.
During GRM training, MetaCritic is further optimized with target
rankings on $\mathcal D_H$ and $\mathcal D_O$ to generalize
human-grounded evaluative criteria to outcome-only data.
MetaCritic is used only during training, while the GRM operates
independently at inference.
We next describe the core mechanisms of EnGRICH in detail.

\subsection{From Human Critiques to Process Supervision}
\label{sec:process_supervision}

Human critiques reveal the keypoints that support a preference judgment.
A reliable GRM assessment should predict the correct preference and accurately address the keypoints.
The challenge is therefore twofold: transforming limited human critiques
into keypoint-level process supervision, and generalizing the resulting
human-grounded evaluative criteria to broader outcome-only data.
We introduce MetaCritic to construct response-specific keypoints, assess
GRM critiques against them, and extend this supervision beyond the
human-critiqued data.

MetaCritic consists of a rubric generator $M_{\mathrm{rub}}$ and a
critique verifier $M_{\mathrm{ver}}$.
Given an input $X$, $M_{\mathrm{rub}}$ generates a rubric $\mathcal E$
consisting of the keypoints that a well-grounded critique of $X$ should
cover.
Given $X$, a rubric $\mathcal E$, and a generated critique $C_G$,
$M_{\mathrm{ver}}$ labels each keypoint as \texttt{MATCH} if
$C_G$ covers it correctly, \texttt{MISSING} if $C_G$ omits it
or covers it incompletely, and \texttt{INCORRECT} if $C_G$
misstates it.
Let $m$, $u$, and $e$ denote the numbers of keypoints with each label,
respectively.

This design turns the same keypoint verification into two complementary
training signals.
We aggregate the keypoint labels into a scalar process reward that
quantifies critique quality beyond preference correctness, while
retaining the rubric and its labels as structured feedback to guide
exploration toward better critiques.
To make the process reward account for both keypoint coverage and
correctness, we draw on precision and recall in information
retrieval~\citep{manning2008introduction} and define it as
\emph{Keypoint-F1}:
\begin{equation}
P=\frac{m}{m+e},\qquad
R=\frac{m}{m+u+e},\qquad
q(\mathcal E,C_G)=\frac{2PR}{P+R}
=\frac{2m}{2m+u+2e}.
\label{eq:process_score}
\end{equation}
where \texttt{INCORRECT} keypoints reduce both precision and recall,
whereas \texttt{MISSING} keypoints reduce only recall\footnote{We omit the fixed input $X$ from the notation of
$q(\mathcal E,C_G)$ for brevity.}.

% \begin{equation}
% q(\mathcal E,C_G)=\frac{2m}{2m+u+2e},
% \label{eq:process_score}
% \end{equation}
% where \texttt{INCORRECT} keypoints reduce both precision and recall,
% whereas \texttt{MISSING} keypoints reduce only recall.

\paragraph{Human-grounded initialization of MetaCritic.}
We initialize MetaCritic on $\mathcal D_H$ in two SFT stages.
First, $M_{\mathrm{rub}}$ learns to generate rubrics from inputs
$X$ in $\mathcal D_H$ using offline targets that a frozen external
LLM $\mathcal A$ derives from the human critique $C_H$.
We then freeze $M_{\mathrm{rub}}$ and train $M_{\mathrm{ver}}$ to label
critiques sampled from the initial GRM against rubrics generated by
$M_{\mathrm{rub}}$, matching its inputs during GRM training.
The corresponding keypoint-level labels are also constructed by $\mathcal A$.
Both components are implemented as independent LoRA adapters over a
shared frozen backbone with the same architecture as the GRM.

\subsection{Learning with Human-Grounded Process Supervision}
\label{sec:grm_optimization}

Keypoint verification supervises the GRM in two ways: scalar process
rewards distinguish critique quality among correct assessments, while
keypoint-guided exploration recovers learning signal in groups where
all assessments fail.
Alongside GRM optimization, we train $M_{\mathrm{rub}}$ to extend
the evaluative criteria learned from human critiques to outcome-only
data.

\paragraph{Optimizing the GRM with process supervision.}
For each $X$, we sample $K$ assessments
$Y_k=(C_{G,k},\hat w_k)$ from the GRM $\pi_\theta$ and a rubric
$\mathcal E_0\sim M_{\mathrm{rub}}(\cdot\mid X)$, and combine preference
correctness with Keypoint-F1:
\begin{small}
\begin{equation}
R_k=
\begin{cases}
-1, & \text{invalid format},\\
0, & \text{valid format},\ \hat w_k\neq w,\\
1+\lambda\, q(\mathcal E_0,C_{G,k}),
& \text{valid format},\ \hat w_k=w,
\end{cases}
\label{eq:grm_credit}
\end{equation}
\end{small}
where $\lambda\geq0$ weights the process reward.
We grant process credit only to correct predictions, since rewarding
process quality for outcome-incorrect assessments can decouple process
from outcome and invite reward hacking~\citep{sun2025stabilizinglongtermmultiturnreinforcement}.
Thus, Keypoint-F1 provides bounded process-level differentiation among
correct assessments.

For failed assessments, keypoint verification instead guides exploration
to recover the learning signal.
When all $K$ assessments have valid formats but incorrect predictions,
their rewards tie at zero and group-relative advantages vanish.
Yet these assessments need not fail in the same way: $M_{\mathrm{ver}}$
reveals which keypoints in $\mathcal E_0$ are \texttt{MISSING} or
\texttt{INCORRECT} in each critique.

Inspired by prior refinement methods with external
guidance~\citep{
zhang2026critiquegrpoadvancingllmreasoning,
li2026mulferlenhancingreinforcementlearning},
we use the sampled rubric $\mathcal E_0$ and its keypoint labels as
privileged information to guide exploration toward a critique
that better addresses the identified gaps.
Specifically, we randomly select one failed assessment in the group
for revision.
The exploration prompt includes the rubric and keypoint labels but withholds the gold preference $w$ and prediction correctness, so that the new assessment is guided by keypoint-level information while preventing the model from simply flipping its prediction based on outcome information.
If the resulting assessment has a valid format and predicts the correct
preference, it replaces its parent and is rescored under $\mathcal E_0$;
otherwise, the group remains unchanged.
Successful exploration thus converts keypoint-level guidance into reward
contrast, restoring learning signal in an otherwise uninformative group.
We optimize the GRM with
GRPO~\citep{shao2024deepseekmathpushinglimitsmathematical},
centering rewards within the resulting group:
\begin{small}
\begin{equation}
A_k=
R_k-\operatorname{mean}_j(R_j),
\qquad
\mathcal L_{\mathrm{GRM}}(\theta)
=
\frac{1}{K}\sum_{k=1}^{K}
\ell(\theta;X,Y_k,A_k),
\label{eq:grm_objective}
\end{equation}
\end{small}
where $\ell$ is the token-level GRPO loss.
For assessments obtained through guided exploration, the standard
importance ratio does not match the policy that generated them.
Following prior work~\citep{
zhang2026critiquegrpoadvancingllmreasoning,
yan2025learningreasonoffpolicyguidance},
we evaluate them under the policy conditioned only on $X$ and apply
$f(r_t)=r_t/(r_t+\gamma)$, where
$r_t=\pi_\theta(y_t\mid X,y_{<t})/
\pi_{\theta_{\mathrm{old}}}(y_t\mid X,y_{<t})$ and $0<\gamma<1$.
This encourages the GRM to internalize successful explorations from
$X$ alone.

\paragraph{Generalizing human-grounded rubric generation.}

Human-grounded SFT in \S\ref{sec:process_supervision} initializes
$M_{\mathrm{rub}}$ with evaluative criteria from human critiques.
During GRM training, we further optimize $M_{\mathrm{rub}}$ on
$\mathcal D_H$ and $\mathcal D_O$ to retain these criteria and
generalize them to outcome-only data, while keeping
$M_{\mathrm{ver}}$ fixed for consistent keypoint verification.

We train $M_{\mathrm{rub}}$ with complementary supervision for
human-grounded retention and outcome-guided adaptation.
On $\mathcal D_H$, the human-derived reference rubric $\mathcal E^*$
scores the same GRM critiques with Keypoint-F1.
Together with preference correctness, these scores define a reference
ranking: correctness determines the primary ordering, while critiques
with the same correctness are ordered by their scores under
$\mathcal E^*$.
Each generated rubric then scores the same critiques, and is rewarded
according to how well its induced ranking matches this reference
ranking.

On $\mathcal D_O$, where human critiques are unavailable, preference
correctness provides coarser target rankings under the same reward form.
These rankings provide RL supervision for $M_{\mathrm{rub}}$, initialized
by human-grounded SFT on $\mathcal D_H$, to generalize the evaluative
criteria learned from human critiques to broader outcome-only data
without reference rubrics.
Meanwhile, continued supervision on $\mathcal D_H$ encourages
$M_{\mathrm{rub}}$ to retain fine-grained distinctions in keypoint
coverage and correctness beyond the final judgment.

Specifically, for each $X$, we sample
$\{\mathcal E_g\}_{g=0}^{G_X-1}$ from
$M_{\mathrm{rub}}(\cdot\mid X)$ and apply them to the same $K$
original GRM critiques; only $\mathcal E_0$ supplies GRM process
rewards and exploration guidance.
Let $c_i=\mathbb{I}[\hat w_i=w]$ denote preference correctness.
On $\mathcal D_H$, let
$q_i^*=q(\mathcal E^*,C_{G,i})$ denote the reference-rubric score.
For format-valid assessments, we define:
\begin{small}
\begin{equation}
\begin{aligned}
\mathcal P_X &=
\begin{cases}
\{(i,j):c_i>c_j\}
\cup
\{(i,j):c_i=c_j,\ q_i^*>q_j^*\},
& \text{on }\mathcal D_H,\\
\{(i,j):c_i>c_j\},
& \text{on }\mathcal D_O,
\end{cases}
\\[2pt]
r_g^{\mathrm{rub}}
&=
\frac{1}{|\mathcal P_X|}
\sum_{(i,j)\in\mathcal P_X}
\log\sigma\!\left(
q(\mathcal E_g,C_{G,i})
-
q(\mathcal E_g,C_{G,j})
\right).
\end{aligned}
\label{eq:rubric_ranking_reward}
\end{equation}
\end{small}
where $\sigma$ is the logistic sigmoid.
Inputs with $\mathcal P_X=\varnothing$ are skipped for rubric optimization.
The resulting rewards are centered across the $G_X$ sampled rubrics
without standard-deviation normalization and used to optimize
$M_{\mathrm{rub}}$ with GRPO.

% Specifically, for each $X$, we sample
% $\{\mathcal E_g\}_{g=0}^{G_X-1}$ from
% $M_{\mathrm{rub}}(\cdot\mid X)$ and apply them to the same $K$
% unrevised GRM critiques; only $\mathcal E_0$ supplies GRM process
% rewards and revision feedback.
% Let $c_i=\mathbb{I}[\hat w_i=w]$ and
% $q_i^*=q_{\mathrm{ver}}(\mathcal E^*,C_{G,i})$ denote preference
% correctness and the reference-rubric score, respectively.
% For format-valid assessments, we define:
% \begin{small}
% \begin{equation}
% \begin{aligned}
% \mathcal P_X &=
% \begin{cases}
% \{(i,j):c_i>c_j\}
% \cup
% \{(i,j):c_i=c_j,\ q_i^*>q_j^*\},
% & \text{on }\mathcal D_H,\\
% \{(i,j):c_i>c_j\},
% & \text{on }\mathcal D_O,
% \end{cases}
% \\[2pt]
% r_g^{\mathrm{rub}}
% &=
% \frac{1}{|\mathcal P_X|}
% \sum_{(i,j)\in\mathcal P_X}
% \log\sigma\!\left(
% q_{\mathrm{ver}}(\mathcal E_g,C_{G,i})
% -
% q_{\mathrm{ver}}(\mathcal E_g,C_{G,j})
% \right),
% \end{aligned}
% \label{eq:rubric_ranking_reward}
% \end{equation}
% \end{small}
% where $\sigma$ is the logistic sigmoid.

\section{Experiments}

\subsection{Experimental Settings}
\label{sec:experimental_settings}

\paragraph{Datasets.}
Following the data-mixture design of RM-NLHF~\citep{wang2026rewardmodelingnaturallanguage},
we construct a 60K pairwise preference mixture with 12K human-critique
examples $\mathcal{D}_H$ and 48K outcome-only examples $\mathcal{D}_O$.
$\mathcal{D}_H$ is drawn from HelpSteer3~\citep{wang2025helpsteer3} and contains
natural-language human critiques.
$\mathcal{D}_O$ combines the Skywork-Reward-Preference-80K mixture
~\citep{liu2024skyworkrewardbagtricksreward},
Code-Preference-Pairs,
Math-Step-DPO-10K~\citep{lai2024stepdpostepwisepreferenceoptimization},
HH-RLHF~\citep{bai2022traininghelpfulharmlessassistant},
Tulu-3~\citep{lambert2025tulu3pushingfrontiers},
and LitBench-Train~\citep{fein-etal-2026-litbench},
covering general assistance, code, mathematics, safety,
instruction following, and writing.
We evaluate on seven benchmarks:
HelpSteer3 (HS3)~\citep{wang2025helpsteer3},
RewardBench~2 (RB2)~\citep{malik2026rewardbench2advancingreward},
RM-Bench (RM-B)~\citep{liu2024rmbenchbenchmarkingrewardmodels},
SCAN-HPD (SCAN)~\citep{wang2026scan},
HREF-HA (HREF)~\citep{lyu2025href},
LitBench~\citep{fein-etal-2026-litbench},
and WQ-Arena (WQ)~\citep{chakrabarty2025aipolish}.
\footnote{Full dataset details are provided in \S~\ref{app:datasets}.}

\paragraph{Baselines.}
We compare against six baselines:
(1) \textit{Base}, the backbone used directly as a GRM;
(2) \textit{SFT}, finetuned on human critiques in $\mathcal D_H$;
(3) \textit{DPO}~\citep{rafailov2024dpo}, optimizing pairwise preferences;
(4) \textit{GRPO}~\citep{shao2024deepseekmathpushinglimitsmathematical},
using format and preference correctness as rewards;
(5) \textit{RC-GRPO}~\citep{wang2026outcomeaccuracyenoughaligning},
which requires human rationales and is thus limited to $\mathcal D_H$,
using rationale consistency as process supervision; and
(6) \textit{RM-NLHF}~\citep{wang2026rewardmodelingnaturallanguage},
which learns process rewards from human-critique similarity on
$\mathcal D_H$ and extends them to $\mathcal D_O$ with a scalar estimator.
Both process-supervised baselines combine their process rewards with
correctness and are optimized with GRPO. Representative public reward models are included for reference only due to differences in training data and backbones.

% \paragraph{Baselines.}
% We focus on six baselines:
% (1) \textit{Base}, the backbone used directly as a GRM;
% (2) \textit{SFT}, supervised finetuning on the human critiques in $\mathcal{D}_H$;
% (3) \textit{DPO}~\citep{rafailov2024dpo}, which directly optimizes pairwise preferences;
% (4) \textit{GRPO}~\citep{shao2024deepseekmathpushinglimitsmathematical}, which optimizes a correctness reward (format compliance and preference accuracy);
% (5)\textit{RC-GRPO}~\citep{wang2026outcomeaccuracyenoughaligning},
% which requires human rationales and is therefore limited to
% $\mathcal D_H$, using rationale consistency as process supervision and
% (6) \textit{RM-NLHF}~\citep{wang2026rewardmodelingnaturallanguage}, which derives a process reward from critique similarity to human critiques in $\mathcal{D}_H$ and trains an LLM-based scalar estimator to extend this supervision to $\mathcal{D}_O$.
% Both methods combine their process supervision with the correctness reward and are optimized with GRPO.
% Several representative public reward models are additionally included as references.

\paragraph{Implementation Details.}
We use Qwen3-4B~\citep{yang2025qwen3technicalreport} in non-thinking mode
and Gemma3-12B-Instruct (Gemma3-12B)~\citep{gemmateam2025gemma3technicalreport}
as backbones.
The GRM, rubric generator $M_{\mathrm{rub}}$, and critique verifier $M_{\mathrm{ver}}$
use separate LoRA~\citep{hu2021loralowrankadaptationlarge} adapters over
a frozen backbone.
For training examples with human-derived supervision available from prior
work~\citep{wang2026rewardmodelingnaturallanguage}, a frozen
Qwen3.7-Max\footnote{\url{https://www.alibabacloud.com/help/en/model-studio/qwen3-7-max}}
teacher constructs reference rubrics and annotates keypoint-level
relations.
The two MetaCritic SFT stages each run for one epoch, while the SFT
baseline is trained for three epochs. Each RL method is trained for one pass over its available training data.
Training uses 32 H20 GPUs and inference uses 8 H20 GPUs.
At inference, the GRM uses greedy decoding.
\footnote{Implementation details, settings, and reward-model sources are provided
in \S~\ref{app:impl}.}

\definecolor{judgeband}{RGB}{235,232,250}
\definecolor{rmband}{RGB}{228,240,249}
\definecolor{trainband}{RGB}{244,220,220}
\definecolor{backboneband}{RGB}{252,242,242}

\begin{table*}[t]
\centering
\caption{
Results across seven reward-model benchmarks.
\textbf{Bold} and \underline{underlined} indicate the best and second-best
controlled results within each backbone; Overall is the unweighted mean.
}
\label{tab:main_results}

\small
\setlength{\tabcolsep}{3.0pt}
\renewcommand{\arraystretch}{1.07}

\begin{tabular*}{\textwidth}{
@{\extracolsep{\fill}}
>{\raggedright\arraybackslash}p{0.18\textwidth}
!{\color{black}\vrule width 0.6pt}
*{7}{>{\centering\arraybackslash}p{0.080\textwidth}}
!{\color{black!25}\vrule width 0.4pt}
>{\centering\arraybackslash}p{0.080\textwidth}
@{}
}
\toprule

\textbf{Model / Method}
& \textbf{HS3}
& \textbf{RB2}
& \textbf{RM-B}
& \textbf{SCAN}
& \textbf{HREF}
& \textbf{LitBench}
& \textbf{WQ}
& \textbf{Overall} \\

\midrule

% ============================================================
% Open-Weight Reward Models
% ============================================================
\rowcolor{rmband}
\multicolumn{9}{c}{\textit{Open-Weight Reward Models}} \\

GRAM-Qwen3-4B
& 74.17
& 66.01
& 66.87
& 74.44
& 60.78
& 62.24
& 55.03
& 65.65 \\

GRAM-R$^2$-3B
& 76.37
& 50.94
& 64.18
& 69.97
& 63.19
& 62.64
& 59.16
& 63.78 \\

CE-RM-4B
& 73.44
& 70.76
& 79.50
& 68.85
& 59.66
& 57.42
& 52.78
& 66.06 \\

UltraRM-13B
& 74.60
& 53.07
& 57.20
& 74.31
& 58.26
& 69.80
& 59.06
& 63.76 \\

RM-R1-14B
& 76.42
& 65.25
& 81.11
& 77.32
& 64.79
& 65.10
& 58.86
& 69.84 \\

UniRRM-14B
& 79.14
& 72.65
& 82.65
& 80.83
& 62.91
& 63.08
& 56.76
& 71.15 \\

\midrule

% ============================================================
% Controlled Training
% ============================================================
\rowcolor{trainband}
\multicolumn{9}{c}{\textit{Controlled Training}} \\

% ------------------------------------------------------------
% Qwen3-4B
% ------------------------------------------------------------
\multicolumn{9}{
@{}>{\columncolor{backboneband}[0pt][0pt]}l@{}
}{
\textit{Qwen3-4B}
} \\

Base
& 72.07
& 52.34
& 69.88
& 75.72
& 52.12
& 62.76
& 55.49
& 62.91 \\

SFT
& 72.23
& 46.23
& 65.14
& 72.20
& 57.02
& 60.50
& 56.56
& 61.41 \\

DPO
& 74.29
& 59.98
& 72.05
& \underline{76.20}
& 56.97
& \underline{64.27}
& \underline{57.02}
& 65.83 \\

GRPO
& \underline{75.16}
& \underline{62.28}
& 73.81
& 74.60
& 56.22
& 63.33
& 56.10
& 65.93 \\

RC-GRPO
& 74.53
& 60.37
& 73.62
& 74.44
& \underline{58.16}
& 63.45
& 56.46
& 65.86 \\

RM-NLHF
& 75.00
& 61.50
& \underline{73.98}
& 75.24
& 57.20
& 63.79
& \underline{57.02}
& \underline{66.25} \\

\textbf{EnGRICH}
& \textbf{77.06}
& \textbf{65.42}
& \textbf{75.12}
& \textbf{77.48}
& \textbf{62.60}
& \textbf{65.69}
& \textbf{61.15}
& \textbf{69.22} \\

\addlinespace[2pt]

% ------------------------------------------------------------
% Gemma3-12B
% ------------------------------------------------------------
\multicolumn{9}{
@{}>{\columncolor{backboneband}[0pt][0pt]}l@{}
}{
\textit{Gemma3-12B}
} \\

Base
& 71.29
& 44.14
& 65.84
& 76.04
& 57.75
& 67.04
& 59.62
& 63.10 \\

SFT
& 74.34
& 52.90
& 66.47
& 75.72
& 61.73
& 63.75
& 61.46
& 65.20 \\

DPO
& \underline{76.42}
& 52.81
& 69.82
& \underline{77.48}
& 60.38
& \underline{69.33}
& 61.77
& 66.86 \\

GRPO
& 75.78
& \underline{55.61}
& \underline{72.56}
& 75.24
& \underline{64.27}
& 66.67
& \underline{61.92}
& \underline{67.44} \\

RC-GRPO
& 74.90
& 53.62
& 71.21
& 75.56
& 63.50
& 65.81
& 61.61
& 66.60 \\

RM-NLHF
& 75.62
& 54.80
& 71.90
& 76.04
& 63.40
& 67.20
& 61.51
& 67.21 \\

\textbf{EnGRICH}
& \textbf{77.71}
& \textbf{59.85}
& \textbf{74.66}
& \textbf{78.27}
& \textbf{66.91}
& \textbf{69.58}
& \textbf{65.24}
& \textbf{70.32} \\

\bottomrule
\end{tabular*}
\vspace{-1.0em}
\end{table*}

\subsection{Main Results}
\label{exp:main}
As shown in Table~\ref{tab:main_results}, EnGRICH consistently improves
performance across all seven reward-model benchmarks and both backbones,
demonstrating its effectiveness across diverse evaluation settings.
The comparison with process-supervised baselines further highlights the
mechanism of EnGRICH.
RC-GRPO uses human-derived rationale consistency as process supervision,
while RM-NLHF extends human-critique supervision through a learned scalar
process reward.
EnGRICH instead couples rubric learning with GRM optimization:
$M_{\mathrm{rub}}$ learns human-grounded evaluative criteria on
$\mathcal D_H$ and generalizes them to outcome-only data $\mathcal D_O$,
while keypoint verification turns these criteria into both scalar process
rewards for fine-grained credit assignment and keypoint-level guidance
for exploration in all-failed groups.
This combination allows limited human critiques to supervise both
critique quality and exploration, leading to stronger and more
consistent gains.

\begin{figure}[t]
    \centering

    \begin{subfigure}[t]{0.52\linewidth}
        \centering
        \includegraphics[
            width=\linewidth,
            height=4cm
        ]{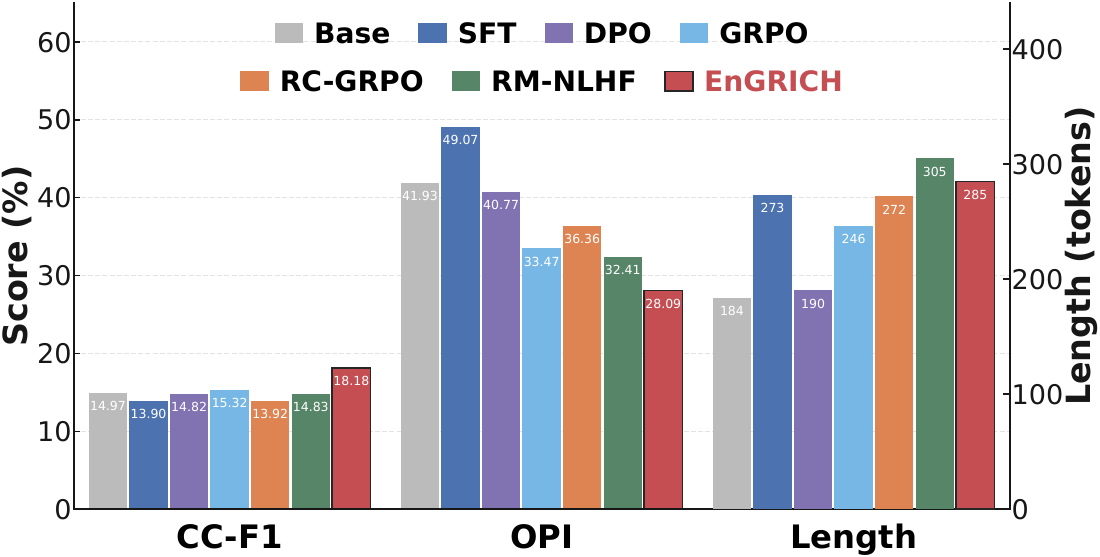}
        \caption{Direct critique evaluation.}
        \label{fig:critique-quality}
    \end{subfigure}
    \hfill
    \begin{subfigure}[t]{0.44\linewidth}
        \centering
        \includegraphics[
            width=\linewidth,
            height=4cm
        ]{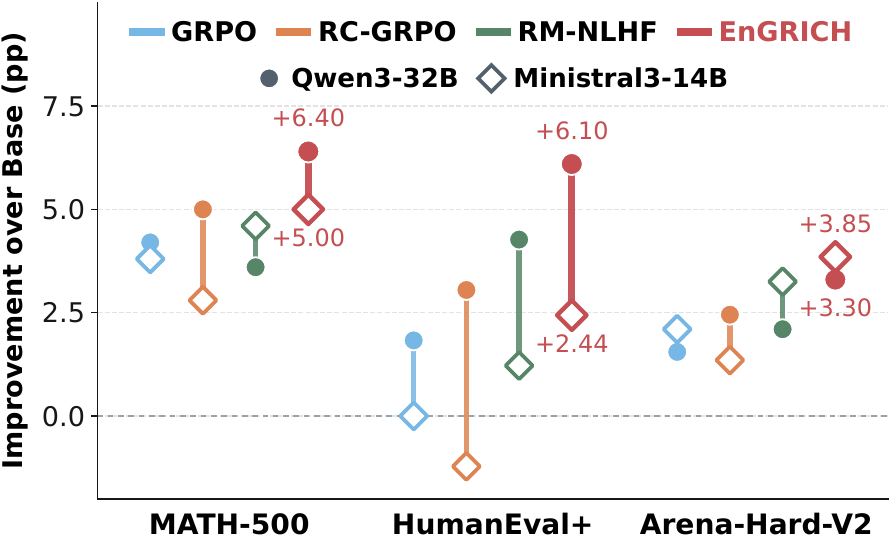}
        \caption{Downstream revision evaluation.}
        \label{fig:feedback-usability}
    \end{subfigure}

\caption{
Critique quality evaluation.
All compared methods use Qwen3-4B as the GRM backbone.
Left: CC-F1, OPI, and average assessment length on HelpSteer3 validation set;
higher CC-F1 and lower OPI indicate better critique quality.
Right: response-revision gains over direct-answer baselines.}
    \label{fig:critique-analysis}
    % \vspace{-2.0em}
\end{figure}

\subsection{Critique Quality Evaluation}
\label{sec:evidence_quality}

We further evaluate the quality of the generated critiques from direct and downstream perspectives.
For direct evaluation, we adopt two critique-quality metrics:
\textbf{Common Correct F1 (CC-F1)}~\citep{wang2026outcomeaccuracyenoughaligning}
and \textbf{Outcome Process Inconsistency (OPI)}~\citep{wang2026rewardmodelingnaturallanguage}.
CC-F1 is computed on examples where all compared GRMs predict the correct
preference and measures the F1 agreement between GRM critiques and human critiques.
OPI measures the fraction of correctly predicted examples with unreliable
critiques, defined by core-argument agreement with the human critique of at most
$0.5$.
Both metrics are computed on the HelpSteer3 validation set with GPT-5.5 as the judge.
As shown in Figure~\ref{fig:critique-quality}, EnGRICH achieves higher CC-F1
and lower OPI than the baselines, while comparable critique lengths indicate
that the gains do not simply come from longer critiques.
For downstream evaluation, we use generated critiques to revise responses on MATH-500~\citep{lightman2023letsverifystepstep}, HumanEval+~\citep{liu2023codegeneratedchatgptreally}, and Arena-Hard-v2.0~\citep{pmlr-v267-li25h}, with frozen Qwen3-32B~\citep{yang2025qwen3technicalreport} and Ministral3-14B~\citep{liu2026ministral3} as editors. Figure~\ref{fig:feedback-usability} shows that EnGRICH consistently
yields larger revision gains, suggesting that
its critiques provide more informative and actionable evaluative
evidence.
Together, the direct and downstream evaluations show that EnGRICH
improves critique quality rather than merely increasing preference
accuracy.

\section{Analysis}

\subsection{Ablation Study}
\label{sec:ablation}

\begin{table*}[t]
\centering
\caption{
Ablation study of EnGRICH on Gemma3-12B.
}
\label{tab:ablation}

\small
\setlength{\tabcolsep}{2.5pt}
\renewcommand{\arraystretch}{1.07}

\begin{tabular*}{\textwidth}{
@{\extracolsep{\fill}}
l
*{7}{c}
*{2}{>{\centering\arraybackslash}p{0.070\textwidth}}
@{}
}
\toprule

&
\multicolumn{7}{c}{
\textit{Benchmark Performance}
}
&
\multicolumn{2}{c}{
\textit{Critique Quality}
}
\\

\cmidrule(lr){2-8}
\cmidrule(lr){9-10}

\textbf{Method}
& \textbf{HS3}
& \textbf{RB2}
& \textbf{RM-B}
& \textbf{SCAN}
& \textbf{HREF}
& \textbf{LitBench}
& \textbf{WQ}
& \textbf{CC-F1}
& \textbf{OPI} \\

\midrule

\textbf{EnGRICH}
& \textbf{77.71}
& \textbf{59.85}
& \textbf{74.66}
& 78.27
& 66.91
& \textbf{69.58}
& \textbf{65.24}
& \textbf{19.36}
& \textbf{17.14} \\

\midrule

w/o Process Reward
& 76.74
& 56.90
& 72.60
& \textbf{78.43}
& 63.15
& 67.16
& 62.07
& 14.96
& 25.81 \\

w/o Guided Exploration
& 77.31
& 55.70
& 74.10
& 77.64
& \textbf{66.94}
& 69.15
& 64.22
& 18.22
& 19.92 \\

w/o Rubric Generalization
& 76.98
& 53.95
& 71.60
& 75.88
& 61.79
& 66.05
& 60.49
& 16.37
& 23.02 \\

w/o $\mathcal D_O$ Process Supervision
& 77.15
& 54.65
& 72.20
& 76.84
& 62.71
& 66.57
& 61.41
& 17.44
& 22.14 \\

\bottomrule
\end{tabular*}

\end{table*}

In Table~\ref{tab:ablation}, we ablate four components of
EnGRICH to examine process credit assignment, keypoint-guided
exploration, rubric generalization, and the transfer of process
supervision to outcome-only data.
\textit{w/o Process Reward} sets $\lambda=0$, removing Keypoint-F1
while retaining keypoint-guided exploration.
\textit{w/o Guided Exploration} removes exploration in all-failed
groups while retaining the process reward.
\textit{w/o Rubric Generalization} freezes $M_{\mathrm{rub}}$ after
human-grounded SFT, removing its continued optimization during GRM
training.
Finally, \textit{w/o $\mathcal D_O$ Process Supervision} removes both
Keypoint-F1 and keypoint-guided exploration on $\mathcal D_O$, leaving
preference correctness as the only GRM supervision on this subset.

The ablation results show complementary benefits from the core
components of EnGRICH.
Removing the process reward causes the largest degradation in critique
quality while also lowering performance on most benchmarks,
supporting the importance of Keypoint-F1 for fine-grained credit
assignment among preference-correct assessments.
Removing guided exploration also generally degrades benchmark
performance, showing complementary gains from recovering learning
signal in all-failed groups.
Freezing $M_{\mathrm{rub}}$ causes the largest average drop in benchmark
performance and clear degradation in critique quality, supporting the
continued generalization of human-grounded evaluative criteria beyond
their SFT initialization.
Finally, restricting process supervision to $\mathcal D_H$ underperforms
full EnGRICH across all benchmarks and critique-quality metrics,
showing the benefit of extending human-grounded process supervision
to $\mathcal D_O$.

\begin{figure}[t]
    \centering

    \begin{subfigure}[t]{0.32\linewidth}
        \centering
        \includegraphics[width=\linewidth]
            {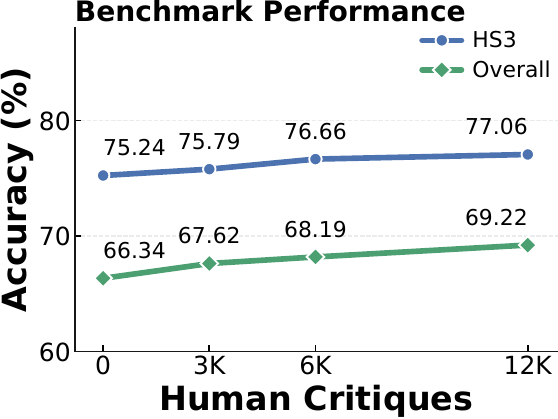}
    \end{subfigure}
    \hfill
    \begin{subfigure}[t]{0.32\linewidth}
        \centering
        \includegraphics[width=\linewidth]
            {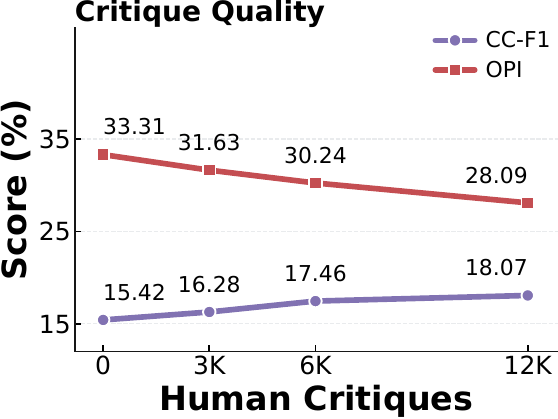}
    \end{subfigure}
    \hfill
    \begin{subfigure}[t]{0.32\linewidth}
        \centering
        \includegraphics[width=\linewidth]
            {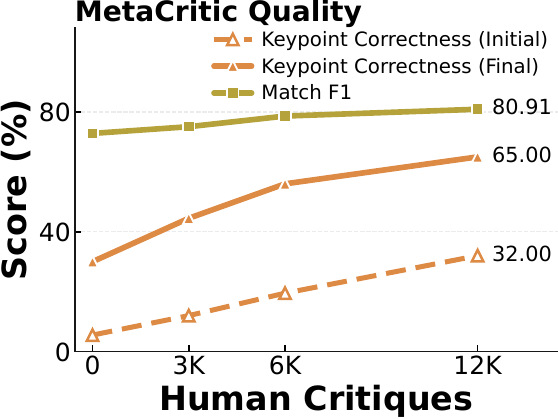}
    \end{subfigure}

\caption{
Scaling human critique supervision with Qwen3-4B as backbone.
We report benchmark performance (left), critique quality (middle),
and MetaCritic quality (right) as human-critique coverage increases.
Initial denotes MetaCritic before online GRM training
(SFT-initialized when human critiques are available), and Final denotes
MetaCritic after GRM training.
}
    \label{fig:human-critique-scaling}
    \vspace{-1.5em}
\end{figure}

\subsection{Scaling Human Critique Supervision}
\label{sec:scaling_human_critique}

We vary the size of the human-critiqued subset $\mathcal D_H$ over
$\{0,3\mathrm{K},6\mathrm{K},12\mathrm{K}\}$, while keeping all $60$K
preference pairs and labels fixed; the remaining examples form the
outcome-only subset $\mathcal D_O$.
Figure~\ref{fig:human-critique-scaling} reports benchmark performance,
critique quality, and MetaCritic quality.
For MetaCritic, GPT-5.5 evaluates sampled HelpSteer3 and RM-Bench
examples: \textit{Keypoint Correctness} measures the fraction of generated
rubrics whose keypoints are all correct, while \textit{Match F1} is the macro-F1
of $M_{\mathrm{ver}}$ over \texttt{MATCH}, \texttt{MISSING}, and
\texttt{INCORRECT} keypoint labels.

Increasing human-critique supervision consistently improves all three
aspects, while rubric generalization further improves rubric quality
across different sizes of $\mathcal D_H$.
Notably, the gains emerge even with limited human critiques.
These results suggest that EnGRICH can generalize human-grounded
evaluative criteria to outcome-only data and reuse them as both scalar
process rewards and keypoint-level guidance for exploration during GRM
optimization.
Thus, limited human critiques can provide reusable supervision beyond
$\mathcal D_H$.
Moreover, $M_{\mathrm{ver}}$ remains consistently accurate, with only
marginal gains from additional human critiques.
This supports keeping $M_{\mathrm{ver}}$ fixed and suggests that the
base model may itself serve as the verifier, further simplifying
EnGRICH.

% We vary the number of HelpSteer3 training examples with human-critique
% supervision across $0$, $3$K, $6$K, and $12$K, while keeping all $60$K
% preference pairs and labels fixed; the remaining examples are treated as
% outcome-only data.
% Figure~\ref{fig:human-critique-scaling} reports benchmark performance,
% critique quality, and MetaCritic quality.
% For MetaCritic, GPT-5.5 evaluates sampled HelpSteer3 and RM-Bench examples
% using Keypoint Correctness for rubric quality and Match F1 for verifier
% accuracy.

% Increasing human-critique supervision consistently improves all three
% aspects, while rubric learning further improves keypoint quality across
% supervision levels.
% Notably, the gains already emerge with sparse human supervision.
% This suggests that EnGRICH can retain and generalize human-derived
% evaluative criteria to outcome-only data, while reusing the resulting
% supervision as scalar process rewards and diagnostic feedback for GRM
% optimization.
% Thus, even limited human critiques can provide reusable supervision
% beyond the annotated subset.
% Moreover, verifier performance remains consistently high with only marginal gains from additional human supervision, supporting our design of keeping the verifier fixed and suggesting that the base model itself may serve as the verifier to further simplify EnGRICH.

\subsection{Transfer to Policy Optimization}
\label{sec:downstream_policy}

Beyond reward-model performance and critique quality, we examine whether
the gains of EnGRICH transfer to downstream policy optimization.
For each GRM variant, we initialize a policy from
Qwen3-4B-Instruct-2507~\citep{yang2025qwen3technicalreport} and optimize
it with GRPO for 100 steps on HelpSteer3 training prompts, using the
corresponding GRM as the reward model.
As shown in Table~\ref{tab:downstream_policy}, the policy trained with
the EnGRICH GRM achieves the strongest performance across all three
downstream benchmarks.
These results indicate that the gains of EnGRICH extend beyond
reward-model evaluation to downstream policy optimization.

\begin{table*}[t]
\centering
\caption{
Downstream performance of Qwen3-4B-Instruct-2507 policies trained with
Gemma3-12B as the GRM backbone. 
}
\label{tab:downstream_policy}

\small
\setlength{\tabcolsep}{8pt}
\renewcommand{\arraystretch}{1.12}

\begin{tabular*}{\textwidth}{
@{\extracolsep{\fill}}
lccc
@{}
}
\toprule

\textbf{Policy Setting}
& \textbf{IFEval P-Strict}
& \textbf{AlpacaEval 2 LC}
& \textbf{Arena-Hard v2 CW} \\

\midrule

Initial policy
& 84.84 & 39.63 & 15.24 \\

\midrule

Outcome-GRPO
& 83.73 & 40.54 & 19.19 \\

RC-GRPO
& 83.55 & 40.21 & 18.87 \\

RM-NLHF
& 84.10 & 40.43 & 19.74 \\

EnGRICH
& \textbf{87.06}
& \textbf{42.80}
& \textbf{21.43} \\

\bottomrule
\end{tabular*}
\end{table*}

\begin{figure*}[t!]
    \centering
    \includegraphics[
        % width=0.98\textwidth
        width=\textwidth
    ]{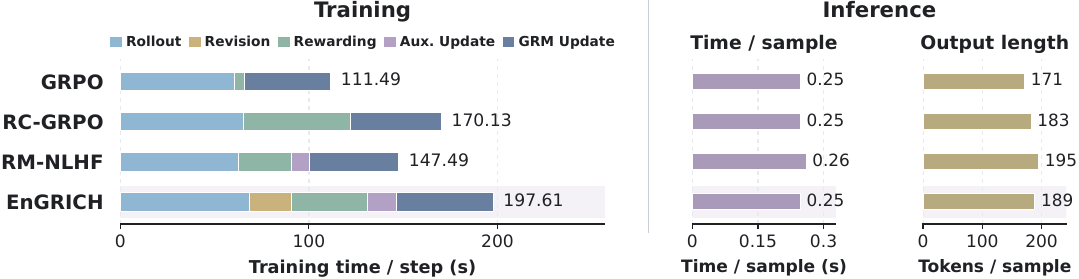}

\caption{
Training and inference efficiency of RL-based GRMs on Gemma3-12B.
\textit{Aux.\ Update} denotes MetaRM updates in RM-NLHF and $M_{\mathrm{rub}}$
updates in EnGRICH. Inference time and output length are averaged per sample across seven benchmarks.
    % Training and inference efficiency of RL-based GRM methods on Gemma3-12B.
    % Aux.\ Update denotes MetaRM updates for RM-NLHF and rubric-generator updates for EnGRICH.
    % GRM Update includes forward/backward computation and optimizer overhead.
    % Inference time and output length are averaged per sample across seven benchmarks.
}
\vspace{-1.0em}
    \label{fig:efficiency}
\end{figure*}

\subsection{Efficiency}
\label{sec:efficiency}

As shown in Figure~\ref{fig:efficiency}, EnGRICH introduces additional 
training overhead over GRPO from Keypoint-F1 verification,
keypoint-guided exploration and rubric
generalization.
This overhead is comparable to RC-GRPO and RM-NLHF, which respectively
add rationale-consistency scoring and a learned process-reward model,
whereas EnGRICH uses MetaCritic to generalize human-grounded evaluative
criteria and provide keypoint-level verification.
At inference, the GRM operates independently of MetaCritic, with
amortized inference time per sample and assessment length comparable
to the baselines.

% \subsection{Efficiency}
% \label{sec:efficiency}

% We measure training and inference costs under the unified setup described in the implementation details.
% As shown in Figure~\ref{fig:efficiency}, EnGRICH incurs additional training cost over outcome-only GRPO due to process-level evaluation and exploration.
% Specifically, sampled assessments are verified against response-specific rubrics, all-failed groups may trigger one diagnosis-guided revision, and the rubric generator is further optimized using evaluation and revision signals.
% Compared with prior process-supervised methods, RC-GRPO mainly introduces rationale-consistency scoring, while RM-NLHF additionally trains a MetaRM to predict process rewards; EnGRICH performs structured keypoint verification together with selective rubric learning and revision.
% Importantly, these components are used only during training.
% At inference time, the trained GRM operates independently of MetaCritic, resulting in latency and output length comparable to the baselines.
% Thus, EnGRICH trades moderate training-time overhead for richer process supervision without introducing additional inference-time modules.

\section{Conclusion}
\label{sec:conclusion}

We study reliable training of generative reward models beyond final
preference correctness.
We propose \textbf{EnGRICH}, which treats human critiques not merely as
additional supervision, but as a source of reusable evaluative criteria
that can be generalized to broader outcome-only data.
Through MetaCritic, these criteria provide both fine-grained process
rewards for credit assignment and structured guidance for exploring
better critiques, while the trained GRM operates independently at
inference.
Experiments across seven reward-model benchmarks and downstream policy
optimization demonstrate consistent improvements, and further analyses
validate the complementary roles of process supervision, guided
exploration, and rubric generalization.
Future work may explore how to better extract, generalize, and reuse
the evaluative criteria in human critiques, enabling more reliable
process supervision across broader data and evaluation settings.

\subsection*{AI use statement}

In this work, GPT-5.6 was used for grammar checking, language
polishing, and assistance with literature search and the identification
of relevant prior work.
Qwen3.7-Max was used to construct reference rubrics from human-derived
keypoints and annotate keypoint-level relations for MetaCritic initialization.
GPT-5.5 was used as an external evaluator for critique-quality,
MetaCritic, and Arena-Hard evaluations.
All AI-assisted writing was manually reviewed by the authors, and
suggested literature was verified against the original sources before
being cited or discussed.
We take full responsibility for the final content and experimental
results of this work.

\subsection*{Ethics statement}

Our work uses existing preference datasets and benchmarks and does not involve the collection of new human-subject data. As with other reward-modeling methods, the learned evaluative criteria may reflect biases or limitations present in the underlying data and human feedback. We encourage careful evaluation across diverse settings when applying the resulting models. We follow the licenses and intended research uses of the datasets and models employed in our experiments.

\subsection*{Reproducibility statement}

To facilitate reproducibility, we provide an anonymous link to our source code and implementation. The main paper describes the proposed method and experimental setup, while the appendix provides additional implementation details, training configurations, dataset and evaluation settings, and baseline specifications needed to reproduce our experiments.

\bibliography{iclr2027_conference}

@misc{zhang2025generativeverifiersrewardmodeling,
      title={Generative Verifiers: Reward Modeling as Next-Token Prediction}, 
      author={Lunjun Zhang and Arian Hosseini and Hritik Bansal and Mehran Kazemi and Aviral Kumar and Rishabh Agarwal},
      year={2025},
      eprint={2408.15240},
      archivePrefix={arXiv},
      primaryClass={cs.LG},
      url={https://arxiv.org/abs/2408.15240}, 
}

@misc{wang2025helpsteer3,
  title={{HelpSteer3-Preference: Open Human-Annotated Preference Data across Diverse Tasks and Languages}},
  author={Zhilin Wang and Jiaqi Zeng and Olivier Delalleau and Hoo-Chang Shin and Felipe Soares and Alexander Bukharin and Ellie Evans and Yi Dong and Oleksii Kuchaiev},
  year={2025},
  eprint={2505.11475},
  archivePrefix={arXiv},
  primaryClass={cs.CL},
  url={https://arxiv.org/abs/2505.11475}
}

@misc{wang2026scan,
  title={{SCAN: Structured Capability Assessment and Navigation for LLMs}},
  author={Zongqi Wang and Tianle Gu and Chen Gong and Xin Tian and Siqi Bao and Yujiu Yang},
  year={2026},
  eprint={2505.06698},
  archivePrefix={arXiv},
  primaryClass={cs.CL},
  url={https://arxiv.org/abs/2505.06698},
}

@misc{lyu2025href,
  title={{HREF: Human Response-Guided Evaluation of Instruction Following in Language Models}},
  author={Xinxi Lyu and Yizhong Wang and Hannaneh Hajishirzi and Pradeep Dasigi},
  year={2025},
  eprint={2412.15524},
  archivePrefix={arXiv},
  primaryClass={cs.CL},
  url={https://arxiv.org/abs/2412.15524},
}

@misc{chakrabarty2025aipolish,
  title={{AI-Slop to AI-Polish? Aligning Language Models through Edit-Based Writing Rewards and Test-time Computation}},
  author={Tuhin Chakrabarty and Philippe Laban and Chien-Sheng Wu},
  year={2025},
  eprint={2504.07532},
  archivePrefix={arXiv},
  primaryClass={cs.CL},
  url={https://arxiv.org/abs/2504.07532},
}

@misc{rafailov2024dpo,
  title={{Direct Preference Optimization: Your Language Model is Secretly a Reward Model}},
  author={Rafael Rafailov and Archit Sharma and Eric Mitchell and Stefano Ermon and Christopher D. Manning and Chelsea Finn},
  year={2024},
  eprint={2305.18290},
  archivePrefix={arXiv},
  primaryClass={cs.LG},
  url={https://arxiv.org/abs/2305.18290},
}

@misc{yang2025qwen3technicalreport,
      title={Qwen3 Technical Report}, 
      author={An Yang and Anfeng Li and Baosong Yang and Beichen Zhang and Binyuan Hui and Bo Zheng and Bowen Yu and Chang Gao and Chengen Huang and Chenxu Lv and Chujie Zheng and Dayiheng Liu and Fan Zhou and Fei Huang and Feng Hu and Hao Ge and Haoran Wei and Huan Lin and Jialong Tang and Jian Yang and Jianhong Tu and Jianwei Zhang and Jianxin Yang and Jiaxi Yang and Jing Zhou and Jingren Zhou and Junyang Lin and Kai Dang and Keqin Bao and Kexin Yang and Le Yu and Lianghao Deng and Mei Li and Mingfeng Xue and Mingze Li and Pei Zhang and Peng Wang and Qin Zhu and Rui Men and Ruize Gao and Shixuan Liu and Shuang Luo and Tianhao Li and Tianyi Tang and Wenbiao Yin and Xingzhang Ren and Xinyu Wang and Xinyu Zhang and Xuancheng Ren and Yang Fan and Yang Su and Yichang Zhang and Yinger Zhang and Yu Wan and Yuqiong Liu and Zekun Wang and Zeyu Cui and Zhenru Zhang and Zhipeng Zhou and Zihan Qiu},
      year={2025},
      eprint={2505.09388},
      archivePrefix={arXiv},
      primaryClass={cs.CL},
      url={https://arxiv.org/abs/2505.09388}, 
}

@misc{liu2026ministral3,
  title={{Ministral 3}},
  author={Alexander H. Liu and Kartik Khandelwal and Sandeep Subramanian and Victor Jouault and Abhinav Rastogi and Adrien Sad{\'e} and Alan Jeffares and Albert Jiang and Alexandre Cahill and Alexandre Gavaudan and Alexandre Sablayrolles and Am{\'e}lie H{\'e}liou and Amos You and Andy Ehrenberg and Andy Lo and Anton Eliseev and Antonia Calvi and Avinash Sooriyarachchi and Baptiste Bout and Baptiste Rozi{\`e}re and Baudouin De Monicault and Cl{\'e}mence Lanfranchi and Corentin Barreau and Cyprien Courtot and Daniele Grattarola and Darius Dabert and Diego de las Casas and Elliot Chane-Sane and Faruk Ahmed and Gabrielle Berrada and Ga{\"e}tan Ecrepont and Gauthier Guinet and Georgii Novikov and Guillaume Kunsch and Guillaume Lample and Guillaume Martin and Gunshi Gupta and Jan Ludziejewski and Jason Rute and Joachim Studnia and Jonas Amar and Jos{\'e}phine Delas and Josselin Somerville Roberts and Karmesh Yadav and Khyathi Chandu and Kush Jain and Laurence Aitchison and Laurent Fainsin and L{\'e}onard Blier and Lingxiao Zhao and Louis Martin and Lucile Saulnier and Luyu Gao and Maarten Buyl and Margaret Jennings and Marie Pellat and Mark Prins and Mathieu Poir{\'e}e and Mathilde Guillaumin and Matthieu Dinot and Matthieu Futeral and Maxime Darrin and Maximilian Augustin and Mia Chiquier and Michel Schimpf and Nathan Grinsztajn and Neha Gupta and Nikhil Raghuraman and Olivier Bousquet and Olivier Duchenne and Patricia Wang and Patrick von Platen and Paul Jacob and Paul Wambergue and Paula Kurylowicz and Pavankumar Reddy Muddireddy and Philom{\`e}ne Chagniot and Pierre Stock and Pravesh Agrawal and Quentin Torroba and Romain Sauvestre and Roman Soletskyi and Rupert Menneer and Sagar Vaze and Samuel Barry and Sanchit Gandhi and Siddhant Waghjale and Siddharth Gandhi and Soham Ghosh and Srijan Mishra and Sumukh Aithal and Szymon Antoniak and Teven Le Scao and Th{\'e}o Cachet and Theo Simon Sorg and Thibaut Lavril and Thiziri Nait Saada and Thomas Chabal and Thomas Foubert and Thomas Robert and Thomas Wang and Tim Lawson and Tom Bewley and Tom Bewley and Tom Edwards and Umar Jamil and Umberto Tomasini and Valeriia Nemychnikova and Van Phung and Vincent Maladi{\`e}re and Virgile Richard and Wassim Bouaziz and Wen-Ding Li and William Marshall and Xinghui Li and Xinyu Yang and Yassine El Ouahidi and Yihan Wang and Yunhao Tang and Zaccharie Ramzi},
  year={2026},
  eprint={2601.08584},
  archivePrefix={arXiv},
  primaryClass={cs.CL},
  url={https://arxiv.org/abs/2601.08584},
}

@misc{li2026mulferlenhancingreinforcementlearning,
      title={MulFeRL: Enhancing Reinforcement Learning with Verbal Feedback in a Multi-turn Loop}, 
      author={Xuancheng Li and Haitao Li and Yujia Zhou and YiqunLiu and Qingyao Ai},
      year={2026},
      eprint={2601.22900},
      archivePrefix={arXiv},
      primaryClass={cs.AI},
      url={https://arxiv.org/abs/2601.22900}, 
}

@misc{zhang2026critiquegrpoadvancingllmreasoning,
      title={Critique-GRPO: Advancing LLM Reasoning with Natural Language and Numerical Feedback}, 
      author={Xiaoying Zhang and Yipeng Zhang and Hao Sun and Kaituo Feng and Chaochao Lu and Chao Yang and Helen Meng},
      year={2026},
      eprint={2506.03106},
      archivePrefix={arXiv},
      primaryClass={cs.CL},
      url={https://arxiv.org/abs/2506.03106}, 
}

@inproceedings{lambert2025rewardbench,
  title={Rewardbench: Evaluating reward models for language modeling},
  author={Lambert, Nathan and Pyatkin, Valentina and Morrison, Jacob and Miranda, LJ and Lin, Bill Yuchen and Chandu, Khyathi and Dziri, Nouha and Kumar, Sachin and Zick, Tom and Choi, Yejin and others},
  booktitle={Findings of the Association for Computational Linguistics: NAACL 2025},
  pages={1755--1797},
  year={2025}
}

@misc{ye2024scalarrewardmodellearning,
      title={Beyond Scalar Reward Model: Learning Generative Judge from Preference Data},
      author={Ziyi Ye and Xiangsheng Li and Qiuchi Li and Qingyao Ai and Yujia Zhou and Wei Shen and Dong Yan and Yiqun Liu},
      year={2024},
      eprint={2410.03742},
      archivePrefix={arXiv},
      primaryClass={cs.CL},
      url={https://arxiv.org/abs/2410.03742},
}

@misc{chen2025reasongrmenhancinggenerativereward,
      title={ReasonGRM: Enhancing Generative Reward Models through Large Reasoning Models},
      author={Bin Chen and Xinzge Gao and Chuanrui Hu and Penghang Yu and Hua Zhang and Bing-Kun Bao},
      year={2025},
      eprint={2506.16712},
      archivePrefix={arXiv},
      primaryClass={cs.CL},
      url={https://arxiv.org/abs/2506.16712},
}

@misc{xu2026alternatingreinforcementlearningrubricbased,
      title={Alternating Reinforcement Learning for Rubric-Based Reward Modeling in Non-Verifiable LLM Post-Training},
      author={Ran Xu and Tianci Liu and Zihan Dong and Tony Yu and Ilgee Hong and Carl Yang and Linjun Zhang and Tao Zhao and Haoyu Wang},
      year={2026},
      eprint={2602.01511},
      archivePrefix={arXiv},
      primaryClass={cs.CL},
      url={https://arxiv.org/abs/2602.01511},
}

@misc{wang2026rewardmodelingnaturallanguage,
      title={Reward Modeling from Natural Language Human Feedback},
      author={Zongqi Wang and Rui Wang and Yuchuan Wu and Yiyao Yu and Pinyi Zhang and Shaoning Sun and Yujiu Yang and Yongbin Li},
      year={2026},
      eprint={2601.07349},
      archivePrefix={arXiv},
      primaryClass={cs.CL},
      url={https://arxiv.org/abs/2601.07349},
}

@misc{wang2026outcomeaccuracyenoughaligning,
      title={Outcome Accuracy is Not Enough: Aligning the Reasoning Process of Reward Models}, 
      author={Binghai Wang and Yantao Liu and Yuxuan Liu and Tianyi Tang and Shenzhi Wang and Chang Gao and Chujie Zheng and Yichang Zhang and Le Yu and Shixuan Liu and Tao Gui and Qi Zhang and Xuanjing Huang and Bowen Yu and Fei Huang and Junyang Lin},
      year={2026},
      eprint={2602.04649},
      archivePrefix={arXiv},
      primaryClass={cs.CL},
      url={https://arxiv.org/abs/2602.04649}, 
}

@misc{liu2025inference,
  title={Inference-Time Scaling for Generalist Reward Modeling},
  author={Zijun Liu and Peiyi Wang and Runxin Xu and Shirong Ma and Chong Ruan and Peng Li and Yang Liu and Yu Wu},
  year={2025},
  eprint={2504.02495},
  archivePrefix={arXiv},
  primaryClass={cs.CL},
  url={https://arxiv.org/abs/2504.02495}
}

@misc{chen2026rmr1rewardmodelingreasoning,
      title={RM-R1: Reward Modeling as Reasoning}, 
      author={Xiusi Chen and Gaotang Li and Ziqi Wang and Bowen Jin and Cheng Qian and Yu Wang and Hongru Wang and Yu Zhang and Denghui Zhang and Tong Zhang and Hanghang Tong and Heng Ji},
      year={2026},
      eprint={2505.02387},
      archivePrefix={arXiv},
      primaryClass={cs.CL},
      url={https://arxiv.org/abs/2505.02387}, 
}

@misc{malik2026rewardbench2advancingreward,
      title={RewardBench 2: Advancing Reward Model Evaluation}, 
      author={Saumya Malik and Valentina Pyatkin and Sander Land and Jacob Morrison and Noah A. Smith and Hannaneh Hajishirzi and Nathan Lambert},
      year={2026},
      eprint={2506.01937},
      archivePrefix={arXiv},
      primaryClass={cs.CL},
      url={https://arxiv.org/abs/2506.01937}, 
}

@misc{scheurer2024traininglanguagemodelslanguage,
      title={Training Language Models with Language Feedback at Scale}, 
      author={Jérémy Scheurer and Jon Ander Campos and Tomasz Korbak and Jun Shern Chan and Angelica Chen and Kyunghyun Cho and Ethan Perez},
      year={2024},
      eprint={2303.16755},
      archivePrefix={arXiv},
      primaryClass={cs.CL},
      url={https://arxiv.org/abs/2303.16755}, 
}

@misc{madaan2023selfrefineiterativerefinementselffeedback,
      title={Self-Refine: Iterative Refinement with Self-Feedback}, 
      author={Aman Madaan and Niket Tandon and Prakhar Gupta and Skyler Hallinan and Luyu Gao and Sarah Wiegreffe and Uri Alon and Nouha Dziri and Shrimai Prabhumoye and Yiming Yang and Shashank Gupta and Bodhisattwa Prasad Majumder and Katherine Hermann and Sean Welleck and Amir Yazdanbakhsh and Peter Clark},
      year={2023},
      eprint={2303.17651},
      archivePrefix={arXiv},
      primaryClass={cs.CL},
      url={https://arxiv.org/abs/2303.17651}, 
}

@misc{shinn2023reflexionlanguageagentsverbal,
      title={Reflexion: Language Agents with Verbal Reinforcement Learning}, 
      author={Noah Shinn and Federico Cassano and Edward Berman and Ashwin Gopinath and Karthik Narasimhan and Shunyu Yao},
      year={2023},
      eprint={2303.11366},
      archivePrefix={arXiv},
      primaryClass={cs.AI},
      url={https://arxiv.org/abs/2303.11366}, 
}

@misc{akyürek2023rl4fgeneratingnaturallanguage,
      title={RL4F: Generating Natural Language Feedback with Reinforcement Learning for Repairing Model Outputs}, 
      author={Afra Feyza Akyürek and Ekin Akyürek and Aman Madaan and Ashwin Kalyan and Peter Clark and Derry Wijaya and Niket Tandon},
      year={2023},
      eprint={2305.08844},
      archivePrefix={arXiv},
      primaryClass={cs.CL},
      url={https://arxiv.org/abs/2305.08844}, 
}

@inproceedings{christiano2017deep,
  title={Deep Reinforcement Learning from Human Preferences},
  author={Paul F. Christiano and Jan Leike and Tom B. Brown and Miljan Martic and Shane Legg and Dario Amodei},
  booktitle={Advances in Neural Information Processing Systems},
  volume={30},
  pages={4299--4307},
  year={2017}
}

@inproceedings{ouyang2022training,
  title={Training Language Models to Follow Instructions with Human Feedback},
  author={Long Ouyang and Jeffrey Wu and Xu Jiang and Diogo Almeida and Carroll L. Wainwright and Pamela Mishkin and Chong Zhang and Sandhini Agarwal and Katarina Slama and Alex Ray and John Schulman and Jacob Hilton and Fraser Kelton and Luke Miller and Maddie Simens and Amanda Askell and Peter Welinder and Paul F. Christiano and Jan Leike and Ryan Lowe},
  booktitle={Advances in Neural Information Processing Systems},
  volume={35},
  pages={27730--27744},
  year={2022}
}

@inproceedings{yu2025criticrm,
  title={Self-Generated Critiques Boost Reward Modeling for Language Models},
  author={Yue Yu and Zhengxing Chen and Aston Zhang and Liang Tan and Chenguang Zhu and Richard Yuanzhe Pang and Yundi Qian and Xuewei Wang and Suchin Gururangan and Chao Zhang and Melanie Kambadur and Dhruv Mahajan and Rui Hou},
  booktitle={Proceedings of the 2025 Conference of the Nations of the Americas Chapter of the Association for Computational Linguistics: Human Language Technologies (Volume 1: Long Papers)},
  month={apr},
  year={2025},
  address={Albuquerque, New Mexico},
  publisher={Association for Computational Linguistics},
  pages={11499--11514},
  doi={10.18653/v1/2025.naacl-long.573}
}

@misc{mahan2024generative,
  title={Generative Reward Models},
  author={Dakota Mahan and Duy Van Phung and Rafael Rafailov and Chase Blagden and Nathan Lile and Louis Castricato and Jan-Philipp Fr{\"a}nken and Chelsea Finn and Alon Albalak},
  year={2024},
  eprint={2410.12832},
  archivePrefix={arXiv},
  primaryClass={cs.LG}
}

@misc{zhu2025judgelmfinetunedlargelanguage,
      title={JudgeLM: Fine-tuned Large Language Models are Scalable Judges}, 
      author={Lianghui Zhu and Xinggang Wang and Xinlong Wang},
      year={2025},
      eprint={2310.17631},
      archivePrefix={arXiv},
      primaryClass={cs.CL},
      url={https://arxiv.org/abs/2310.17631}, 
}

@misc{kim2024prometheusinducingfinegrainedevaluation,
      title={Prometheus: Inducing Fine-grained Evaluation Capability in Language Models}, 
      author={Seungone Kim and Jamin Shin and Yejin Cho and Joel Jang and Shayne Longpre and Hwaran Lee and Sangdoo Yun and Seongjin Shin and Sungdong Kim and James Thorne and Minjoon Seo},
      year={2024},
      eprint={2310.08491},
      archivePrefix={arXiv},
      primaryClass={cs.CL},
      url={https://arxiv.org/abs/2310.08491}, 
}

@misc{shao2024deepseekmathpushinglimitsmathematical,
      title={DeepSeekMath: Pushing the Limits of Mathematical Reasoning in Open Language Models}, 
      author={Zhihong Shao and Peiyi Wang and Qihao Zhu and Runxin Xu and Junxiao Song and Xiao Bi and Haowei Zhang and Mingchuan Zhang and Y. K. Li and Y. Wu and Daya Guo},
      year={2024},
      eprint={2402.03300},
      archivePrefix={arXiv},
      primaryClass={cs.CL},
      url={https://arxiv.org/abs/2402.03300}, 
}

@misc{liu2024skyworkrewardbagtricksreward,
      title={Skywork-Reward: Bag of Tricks for Reward Modeling in LLMs},
      author={Chris Yuhao Liu and Liang Zeng and Jiacai Liu and Rui Yan and Jujie He and Chaojie Wang and Shuicheng Yan and Yang Liu and Yahui Zhou},
      year={2024},
      eprint={2410.18451},
      archivePrefix={arXiv},
      primaryClass={cs.AI},
      url={https://arxiv.org/abs/2410.18451},
}

@misc{lai2024stepdpostepwisepreferenceoptimization,
      title={Step-DPO: Step-wise Preference Optimization for Long-chain Reasoning of LLMs},
      author={Xin Lai and Zhuotao Tian and Yukang Chen and Senqiao Yang and Xiangru Peng and Jiaya Jia},
      year={2024},
      eprint={2406.18629},
      archivePrefix={arXiv},
      primaryClass={cs.LG},
      url={https://arxiv.org/abs/2406.18629},
}

@misc{bai2022traininghelpfulharmlessassistant,
      title={Training a Helpful and Harmless Assistant with Reinforcement Learning from Human Feedback},
      author={Yuntao Bai and Andy Jones and Kamal Ndousse and Amanda Askell and Anna Chen and Nova DasSarma and Dawn Drain and Stanislav Fort and Deep Ganguli and Tom Henighan and Nicholas Joseph and Saurav Kadavath and Jackson Kernion and Tom Conerly and Sheer El-Showk and Nelson Elhage and Zac Hatfield-Dodds and Danny Hernandez and Tristan Hume and Scott Johnston and Shauna Kravec and Liane Lovitt and Neel Nanda and Catherine Olsson and Dario Amodei and Tom Brown and Jack Clark and Sam McCandlish and Chris Olah and Ben Mann and Jared Kaplan},
      year={2022},
      eprint={2204.05862},
      archivePrefix={arXiv},
      primaryClass={cs.CL},
      url={https://arxiv.org/abs/2204.05862},
}

@misc{lambert2025tulu3pushingfrontiers,
      title={Tulu 3: Pushing Frontiers in Open Language Model Post-Training},
      author={Nathan Lambert and Jacob Morrison and Valentina Pyatkin and Shengyi Huang and Hamish Ivison and Faeze Brahman and Lester James V. Miranda and Alisa Liu and Nouha Dziri and Shane Lyu and Yuling Gu and Saumya Malik and Victoria Graf and Jena D. Hwang and Jiangjiang Yang and Ronan Le Bras and Oyvind Tafjord and Chris Wilhelm and Luca Soldaini and Noah A. Smith and Yizhong Wang and Pradeep Dasigi and Hannaneh Hajishirzi},
      year={2025},
      eprint={2411.15124},
      archivePrefix={arXiv},
      primaryClass={cs.CL},
      url={https://arxiv.org/abs/2411.15124},
}

@inproceedings{fein-etal-2026-litbench,
    title = "{L}it{B}ench: A Benchmark and Dataset for Reliable Evaluation of Creative Writing",
    author = "Fein, Daniel  and
      Russo, Sebastian  and
      Xiang, Violet  and
      Jolly, Kabir  and
      Rafailov, Rafael  and
      Haber, Nick",
    editor = "Demberg, Vera  and
      Inui, Kentaro  and
      Marquez, Llu{\'i}s",
    booktitle = "Proceedings of the 19th Conference of the {E}uropean Chapter of the {A}ssociation for {C}omputational {L}inguistics (Volume 1: Long Papers)",
    month = mar,
    year = "2026",
    address = "Rabat, Morocco",
    publisher = "Association for Computational Linguistics",
    url = "https://aclanthology.org/2026.eacl-long.362/",
    doi = "10.18653/v1/2026.eacl-long.362",
    pages = "7740--7755",
    ISBN = "979-8-89176-380-7"
}

@article{zhou2023instructionfollowing,
  title={Instruction-Following Evaluation for Large Language Models},
  author={Zhou, Jeffrey and Lu, Tianjian and Mishra, Swaroop and Brahma, Siddhartha and Basu, Sujoy and Luan, Yi and Zhou, Denny and Hou, Le},
  journal={arXiv preprint arXiv:2311.07911},
  year={2023}
}

@article{dubois2024length,
  title={Length-Controlled AlpacaEval: A Simple Way to Debias Automatic Evaluators},
  author={Dubois, Yann and Galambosi, Bal{\'a}zs and Liang, Percy and Hashimoto, Tatsunori B.},
  journal={arXiv preprint arXiv:2404.04475},
  year={2024}
}

@InProceedings{pmlr-v267-li25h,
  title={From Crowdsourced Data to High-quality Benchmarks: Arena-Hard and Benchbuilder Pipeline},
  author={Li, Tianle and Chiang, Wei-Lin and Frick, Evan and Dunlap, Lisa and Wu, Tianhao and Zhu, Banghua and Gonzalez, Joseph E. and Stoica, Ion},
  booktitle={Proceedings of the 42nd International Conference on Machine Learning},
  pages={34209--34231},
  year={2025},
  volume={267},
  series={Proceedings of Machine Learning Research}
}

@misc{lightman2023letsverifystepstep,
      title={Let's Verify Step by Step}, 
      author={Hunter Lightman and Vineet Kosaraju and Yura Burda and Harri Edwards and Bowen Baker and Teddy Lee and Jan Leike and John Schulman and Ilya Sutskever and Karl Cobbe},
      year={2023},
      eprint={2305.20050},
      archivePrefix={arXiv},
      primaryClass={cs.LG},
      url={https://arxiv.org/abs/2305.20050}, 
}

@misc{liu2023codegeneratedchatgptreally,
      title={Is Your Code Generated by ChatGPT Really Correct? Rigorous Evaluation of Large Language Models for Code Generation}, 
      author={Jiawei Liu and Chunqiu Steven Xia and Yuyao Wang and Lingming Zhang},
      year={2023},
      eprint={2305.01210},
      archivePrefix={arXiv},
      primaryClass={cs.SE},
      url={https://arxiv.org/abs/2305.01210}, 
}

@book{manning2008introduction,
  title     = {Introduction to Information Retrieval},
  author    = {Manning, Christopher D. and Raghavan, Prabhakar
               and Sch{\"u}tze, Hinrich},
  publisher = {Cambridge University Press},
  year      = {2008}
}

@misc{liu2024rmbenchbenchmarkingrewardmodels,
      title={RM-Bench: Benchmarking Reward Models of Language Models with Subtlety and Style}, 
      author={Yantao Liu and Zijun Yao and Rui Min and Yixin Cao and Lei Hou and Juanzi Li},
      year={2024},
      eprint={2410.16184},
      archivePrefix={arXiv},
      primaryClass={cs.CL},
      url={https://arxiv.org/abs/2410.16184}, 
}

@misc{hu2021loralowrankadaptationlarge,
      title={LoRA: Low-Rank Adaptation of Large Language Models}, 
      author={Edward J. Hu and Yelong Shen and Phillip Wallis and Zeyuan Allen-Zhu and Yuanzhi Li and Shean Wang and Lu Wang and Weizhu Chen},
      year={2021},
      eprint={2106.09685},
      archivePrefix={arXiv},
      primaryClass={cs.CL},
      url={https://arxiv.org/abs/2106.09685}, 
}

@misc{gemmateam2025gemma3technicalreport,
      title={Gemma 3 Technical Report}, 
      author={Gemma Team and Aishwarya Kamath and Johan Ferret and Shreya Pathak and Nino Vieillard and Ramona Merhej and Sarah Perrin and Tatiana Matejovicova and Alexandre Ramé and Morgane Rivière and Louis Rouillard and Thomas Mesnard and Geoffrey Cideron and Jean-bastien Grill and Sabela Ramos and Edouard Yvinec and Michelle Casbon and Etienne Pot and Ivo Penchev and Gaël Liu and Francesco Visin and Kathleen Kenealy and Lucas Beyer and Xiaohai Zhai and Anton Tsitsulin and Robert Busa-Fekete and Alex Feng and Noveen Sachdeva and Benjamin Coleman and Yi Gao and Basil Mustafa and Iain Barr and Emilio Parisotto and David Tian and Matan Eyal and Colin Cherry and Jan-Thorsten Peter and Danila Sinopalnikov and Surya Bhupatiraju and Rishabh Agarwal and Mehran Kazemi and Dan Malkin and Ravin Kumar and David Vilar and Idan Brusilovsky and Jiaming Luo and Andreas Steiner and Abe Friesen and Abhanshu Sharma and Abheesht Sharma and Adi Mayrav Gilady and Adrian Goedeckemeyer and Alaa Saade and Alex Feng and Alexander Kolesnikov and Alexei Bendebury and Alvin Abdagic and Amit Vadi and András György and André Susano Pinto and Anil Das and Ankur Bapna and Antoine Miech and Antoine Yang and Antonia Paterson and Ashish Shenoy and Ayan Chakrabarti and Bilal Piot and Bo Wu and Bobak Shahriari and Bryce Petrini and Charlie Chen and Charline Le Lan and Christopher A. Choquette-Choo and CJ Carey and Cormac Brick and Daniel Deutsch and Danielle Eisenbud and Dee Cattle and Derek Cheng and Dimitris Paparas and Divyashree Shivakumar Sreepathihalli and Doug Reid and Dustin Tran and Dustin Zelle and Eric Noland and Erwin Huizenga and Eugene Kharitonov and Frederick Liu and Gagik Amirkhanyan and Glenn Cameron and Hadi Hashemi and Hanna Klimczak-Plucińska and Harman Singh and Harsh Mehta and Harshal Tushar Lehri and Hussein Hazimeh and Ian Ballantyne and Idan Szpektor and Ivan Nardini and Jean Pouget-Abadie and Jetha Chan and Joe Stanton and John Wieting and Jonathan Lai and Jordi Orbay and Joseph Fernandez and Josh Newlan and Ju-yeong Ji and Jyotinder Singh and Kat Black and Kathy Yu and Kevin Hui and Kiran Vodrahalli and Klaus Greff and Linhai Qiu and Marcella Valentine and Marina Coelho and Marvin Ritter and Matt Hoffman and Matthew Watson and Mayank Chaturvedi and Michael Moynihan and Min Ma and Nabila Babar and Natasha Noy and Nathan Byrd and Nick Roy and Nikola Momchev and Nilay Chauhan and Noveen Sachdeva and Oskar Bunyan and Pankil Botarda and Paul Caron and Paul Kishan Rubenstein and Phil Culliton and Philipp Schmid and Pier Giuseppe Sessa and Pingmei Xu and Piotr Stanczyk and Pouya Tafti and Rakesh Shivanna and Renjie Wu and Renke Pan and Reza Rokni and Rob Willoughby and Rohith Vallu and Ryan Mullins and Sammy Jerome and Sara Smoot and Sertan Girgin and Shariq Iqbal and Shashir Reddy and Shruti Sheth and Siim Põder and Sijal Bhatnagar and Sindhu Raghuram Panyam and Sivan Eiger and Susan Zhang and Tianqi Liu and Trevor Yacovone and Tyler Liechty and Uday Kalra and Utku Evci and Vedant Misra and Vincent Roseberry and Vlad Feinberg and Vlad Kolesnikov and Woohyun Han and Woosuk Kwon and Xi Chen and Yinlam Chow and Yuvein Zhu and Zichuan Wei and Zoltan Egyed and Victor Cotruta and Minh Giang and Phoebe Kirk and Anand Rao and Kat Black and Nabila Babar and Jessica Lo and Erica Moreira and Luiz Gustavo Martins and Omar Sanseviero and Lucas Gonzalez and Zach Gleicher and Tris Warkentin and Vahab Mirrokni and Evan Senter and Eli Collins and Joelle Barral and Zoubin Ghahramani and Raia Hadsell and Yossi Matias and D. Sculley and Slav Petrov and Noah Fiedel and Noam Shazeer and Oriol Vinyals and Jeff Dean and Demis Hassabis and Koray Kavukcuoglu and Clement Farabet and Elena Buchatskaya and Jean-Baptiste Alayrac and Rohan Anil and Dmitry and Lepikhin and Sebastian Borgeaud and Olivier Bachem and Armand Joulin and Alek Andreev and Cassidy Hardin and Robert Dadashi and Léonard Hussenot},
      year={2025},
      eprint={2503.19786},
      archivePrefix={arXiv},
      primaryClass={cs.CL},
      url={https://arxiv.org/abs/2503.19786}, 
}

@misc{sun2025stabilizinglongtermmultiturnreinforcement,
      title={Stabilizing Long-term Multi-turn Reinforcement Learning with Gated Rewards}, 
      author={Zetian Sun and Dongfang Li and Zhuoen Chen and Yuhuai Qin and Baotian Hu},
      year={2025},
      eprint={2508.10548},
      archivePrefix={arXiv},
      primaryClass={cs.LG},
      url={https://arxiv.org/abs/2508.10548}, 
}

@misc{yan2025learningreasonoffpolicyguidance,
      title={Learning to Reason under Off-Policy Guidance}, 
      author={Jianhao Yan and Yafu Li and Zican Hu and Zhi Wang and Ganqu Cui and Xiaoye Qu and Yu Cheng and Yue Zhang},
      year={2025},
      eprint={2504.14945},
      archivePrefix={arXiv},
      primaryClass={cs.LG},
      url={https://arxiv.org/abs/2504.14945}, 
}

@misc{lai2026ralignenhancinggenerativereward,
      title={R-Align: Enhancing Generative Reward Models through Rationale-Centric Meta-Judging}, 
      author={Yanlin Lai and Mitt Huang and Hangyu Guo and Xiangfeng Wang and Haodong Li and Shaoxiong Zhan and Liang Zhao and Chengyuan Yao and Yinmin Zhang and Qi Han and Chun Yuan and Zheng Ge and Xiangyu Zhang and Daxin Jiang},
      year={2026},
      eprint={2602.06763},
      archivePrefix={arXiv},
      primaryClass={cs.CL},
      url={https://arxiv.org/abs/2602.06763}, 
}
\bibliographystyle{iclr2027_conference}

\appendix

\section{Algorithm Details and Analysis}
\label{app:algorithm}

This section supplements the main method with the initialization losses,
training procedure, and token-level optimization objectives.
We also establish properties of the resulting supervision signals.
For clarity, let $\phi$ and $\psi$ denote the trainable parameters of
$M_{\mathrm{rub}}$ and $M_{\mathrm{ver}}$, respectively.
The backbone remains frozen, and the two components use independent
LoRA adapters. After initialization, $\psi$ remains fixed, while
$\theta$ and $\phi$ are optimized using their respective objectives.

\subsection{Human-Grounded Initialization}
\label{app:algorithm_initialization}

\paragraph{Rubric generator.}
For each $(X,w,C_H)\in\mathcal D_H$, the frozen external annotator
$\mathcal A$ extracts a reference rubric $\mathcal E^*$ from the human
critique, as described in \S\ref{sec:process_supervision}.
This gives the initialization set
$\mathcal D_{\mathrm{rub}}^{\mathrm{init}}=\{(X,\mathcal E^*)\}$.
Writing $\mathcal E^*=(e_1^*,\ldots,e_{T_E}^*)$ for its tokenization,
we minimize:
\begin{equation}
\mathcal L_{\mathrm{rub}}^{\mathrm{init}}(\phi)
=
-\mathbb E_{(X,\mathcal E^*)\sim\mathcal D_{\mathrm{rub}}^{\mathrm{init}}}
\left[
\frac{1}{T_E}\sum_{t=1}^{T_E}
\log M_{\mathrm{rub},\phi}
(e_t^*\mid X,e_{<t}^*)
\right].
\label{eq:app_rubric_sft}
\end{equation}
Only $X$ is supplied to the generator. The human critique and gold
preference are not part of its conditioning input.

\paragraph{Critique verifier.}
We next freeze the initialized rubric generator and use it to construct
rubrics for inputs in $\mathcal D_H$.
Critiques are sampled from the initial GRM, and $\mathcal A$ supplies
a target status for each rubric keypoint.
Let $Z^*$ be the serialized status sequence and define:
\[
\mathcal D_{\mathrm{ver}}^{\mathrm{init}}
=
\{(X,\mathcal E,C_G,Z^*)\}.
\]
For $Z^*=(z_1^*,\ldots,z_{T_Z}^*)$, the verifier loss is:
\begin{equation}
\mathcal L_{\mathrm{ver}}^{\mathrm{init}}(\psi)
=
-
\mathbb E_{(X,\mathcal E,C_G,Z^*)
\sim
\mathcal D_{\mathrm{ver}}^{\mathrm{init}}}
\left[
\frac{1}{T_Z}
\sum_{t=1}^{T_Z}
\log
M_{\mathrm{ver},\psi}
\left(
z_t^*
\mid
X,\mathcal E,C_G,z_{<t}^*
\right)
\right].
\label{eq:app_verifier_sft}
\end{equation}
Thus, verifier initialization uses generated rubrics rather than only
reference rubrics. Both losses apply to target tokens, not conditioning
tokens. The external annotator is used offline; subsequent training
uses the initialized MetaCritic and the stored reference rubrics.

\subsection{Training Procedure}
\label{app:algorithm_procedure}

\paragraph{Original and retained assessments.}
For each input $X$, we distinguish the original group
$\mathcal Y^0=\{Y_k^0\}_{k=1}^K$ from the retained group
$\widetilde{\mathcal Y}=\{\widetilde Y_k\}_{k=1}^K$ used for the GRM
update. Initially, $\widetilde Y_k=Y_k^0$.
Rubric learning always uses $\mathcal Y^0$, whereas GRM learning uses
$\widetilde{\mathcal Y}$ after any successful replacement.
This distinction prevents a rubric from being evaluated on a revision
that it helped generate.

At the start of a rollout batch, let $\theta_{\mathrm{old}}$ and
$\phi_{\mathrm{old}}$ denote the parameter snapshots used for sampling.
For each input, we sample $K\geq2$ original assessments and $G_X$
rubrics. All candidate rubrics are evaluated against the same original
critiques using the frozen verifier. On $\mathcal D_H$, the stored
reference rubric is evaluated against these same critiques.
The resulting scores determine the target pairs and rubric rewards
in Eq.~\ref{eq:rubric_ranking_reward}.
Only $\mathcal E_0$ supplies GRM rewards and exploration guidance.

\paragraph{One-step guided exploration.}
When all original assessments are format-valid and outcome-incorrect,
we sample a parent index $J$ uniformly from $\{1,\ldots,K\}$.
Let $Z_{0,J}$ denote its keypoint labels under $\mathcal E_0$.
The GRM generates one revision conditioned on
$(X,Y_J^0,\mathcal E_0,Z_{0,J})$ and the revision instruction.
Neither $w$ nor an explicit correctness indicator is included.
A format-valid, outcome-correct revision replaces its parent and is
rescored under $\mathcal E_0$; otherwise, the original group is retained.
The group size therefore remains $K$.

Let $d_k=1$ only when $\widetilde Y_k$ is an accepted revision, and
$d_k=0$ otherwise. Acceptance depends on format and preference
correctness, not on an additional improvement threshold for Keypoint-F1.
The revised assessment is not substituted into the rubric-learning
buffer, and revision success is not added to the rubric reward.
Algorithm~\ref{alg:engrich_training} summarizes the procedure.

\begin{algorithm}[t]
\caption{EnGRICH training}
\label{alg:engrich_training}
\small
\begin{algorithmic}[1]
\Require $\mathcal D_H$, $\mathcal D_O$, initial GRM and MetaCritic;
$K\geq2$, rubric group sizes $G_X$, and optimization hyperparameters.
\State Construct offline reference rubrics with $\mathcal A$.
\State Initialize $M_{\mathrm{rub}}$ using Eq.~\ref{eq:app_rubric_sft}.
\State Freeze $M_{\mathrm{rub}}$; generate rubrics and initial GRM
critiques, and obtain verifier targets from $\mathcal A$.
\State Initialize $M_{\mathrm{ver}}$ using
Eq.~\ref{eq:app_verifier_sft}; then fix $\psi$.
\While{training is not complete}
    \State Snapshot $\theta_{\mathrm{old}}\gets\theta$ and
    $\phi_{\mathrm{old}}\gets\phi$; sample a batch $\mathcal B$.
    \State Initialize GRM and rubric update buffers.
    \For{each input $X$ in $\mathcal B$}
        \State Sample $\mathcal Y^0\sim
        \pi_{\theta_{\mathrm{old}}}(\cdot\mid X)$ with $K$ assessments.
        \State Sample $\{\mathcal E_g\}_{g=0}^{G_X-1}\sim
        M_{\mathrm{rub},\phi_{\mathrm{old}}}(\cdot\mid X)$.
        \State Verify format-valid original critiques under each
        $\mathcal E_g$; cache their labels and scores.
        \If{$X\in\mathcal D_H$}
            \State Score original critiques under the stored
            reference rubric $\mathcal E^*$.
        \EndIf
        \State Construct $\mathcal P_X$ from format-valid original
        assessments using Eq.~\ref{eq:rubric_ranking_reward}.
        \If{$\mathcal P_X\neq\varnothing$ and $G_X\geq2$}
            \State Compute $r_g^{\mathrm{rub}}$ and centered
            advantages $B_g$; store the rubric update tuple.
        \EndIf
        \State Set $\widetilde{\mathcal Y}\gets\mathcal Y^0$
        and $d_k\gets0$ for all $k$.
        \If{all original assessments are valid and incorrect}
            \State Sample $J$ uniformly and generate one revision
            using $(X,Y_J^0,\mathcal E_0,Z_{0,J})$.
            \If{the revision is valid and correct}
                \State Replace $\widetilde Y_J$ with the revision;
                set $d_J\gets1$ and rescore it under $\mathcal E_0$.
            \EndIf
        \EndIf
        \State Compute $R_k$ and $A_k$ for
        $\widetilde{\mathcal Y}$ using
        Eqs.~\ref{eq:grm_credit} and \ref{eq:grm_objective}.
        \State Store $(X,\widetilde{\mathcal Y},\{A_k,d_k\}_{k=1}^K)$
        in the GRM update buffer.
    \EndFor
    \State Update $\theta$ using Eq.~\ref{eq:app_grm_batch_loss}.
    \State Update $\phi$ using Eq.~\ref{eq:app_rubric_rl_loss}
    if the rubric update buffer is nonempty.
\EndWhile
\State \Return $\pi_\theta$ for standalone inference.
\end{algorithmic}
\end{algorithm}

\subsection{Token-Level Optimization Objectives}
\label{app:algorithm_losses}

\paragraph{GRM objective.}
We explicitly write the policy-surrogate terms using per-sequence
token averaging. Let $\widetilde Y_k=(y_{k,1},\ldots,y_{k,T_k})$
denote the retained assessment, including its critique and final
preference. Define:
\begin{align}
r_{k,t}(\theta)
&=
\frac{\pi_\theta(y_{k,t}\mid X,y_{k,<t})}
{\pi_{\theta_{\mathrm{old}}}(y_{k,t}\mid X,y_{k,<t})},
\label{eq:app_grm_ratio}\\
\mathcal C_\epsilon(r,a)
&=
\min\!\left\{ra,
\operatorname{clip}(r,1-\epsilon,1+\epsilon)a\right\}.
\label{eq:app_clipped_surrogate}
\end{align}
For the standard branch, this is the clipped GRPO
surrogate~\citep{shao2024deepseekmathpushinglimitsmathematical}.
All advantages and denominator probabilities are treated as constants
during differentiation. With $f(r)=r/(r+\gamma)$ and $0<\gamma<1$,
the per-assessment loss is:
\begin{equation}
\ell_k(\theta)
=
-\frac{1}{T_k}\sum_{t=1}^{T_k}
\left[
(1-d_k)\mathcal C_{\epsilon_{\mathrm{GRM}}}
(r_{k,t},A_k)
+
d_k A_k f(r_{k,t})
\right].
\label{eq:app_grm_token_loss}
\end{equation}
The batch objective is:
\begin{equation}
\widehat{\mathcal L}_{\mathrm{GRM}}(\theta)
=
\frac{1}{|\mathcal B|}
\sum_{X\in\mathcal B}
\frac{1}{K}\sum_{k=1}^{K}\ell_k(\theta).
\label{eq:app_grm_batch_loss}
\end{equation}
Here $A_k$ is computed after replacement, without standard-deviation
normalization. Every accepted revision has positive advantage because
it is the only correct assessment in its retained group.
Since $0<f(r)<1$, applying the standard clipped surrogate to the shaped
ratio gives
$\mathcal C_\epsilon(f(r),A_k)=A_kf(r)$ for these revisions.
Thus, the guided branch in Eq.~\ref{eq:app_grm_token_loss} also covers
clipping applied after shaping.

\paragraph{Conditioning of accepted revisions.}
Although a revision is generated with additional diagnostic context,
its training probabilities are evaluated with $X$ as the only
conditioning input. Both the numerator and denominator in
Eq.~\ref{eq:app_grm_ratio} are recomputed under this context;
feedback-conditioned rollout probabilities are not used as the
denominator. The assessment must be re-encoded under $X$, rather than
only masking the loss on feedback tokens while retaining the feedback
in its context.

Policy shaping follows the general motivation of learning from guided
trajectories~\citep{zhang2026critiquegrpoadvancingllmreasoning,
yan2025learningreasonoffpolicyguidance}, but the argument of $f$ here
is the current-to-old ratio explicitly defined in the main text.
This is a shaped learning surrogate, not an unbiased importance-sampling
correction for the diagnostic-conditioned behavior policy.

\paragraph{Rubric-generator objective.}
For each eligible input, compute the rubric rewards in
Eq.~\ref{eq:rubric_ranking_reward} using only original critiques,
and define:
\begin{equation}
B_g
=
r_g^{\mathrm{rub}}
-
\frac{1}{G_X}\sum_{h=0}^{G_X-1}r_h^{\mathrm{rub}}.
\label{eq:app_rubric_advantage}
\end{equation}
Let $\mathcal E_g=(e_{g,1},\ldots,e_{g,U_g})$ denote the tokenized
sampled rubric. Its token-level ratio is:
\begin{equation}
\rho_{g,t}(\phi)
=
\frac{M_{\mathrm{rub},\phi}(e_{g,t}\mid X,e_{g,<t})}
{M_{\mathrm{rub},\phi_{\mathrm{old}}}(e_{g,t}\mid X,e_{g,<t})}.
\label{eq:app_rubric_ratio}
\end{equation}
Writing $\mathcal B_{\mathrm{rub}}$ for inputs with nonempty target
pairs and $G_X\geq2$, we minimize:
\begin{equation}
\widehat{\mathcal L}_{\mathrm{rub}}(\phi)
=
-\frac{1}{|\mathcal B_{\mathrm{rub}}|}
\sum_{X\in\mathcal B_{\mathrm{rub}}}
\frac{1}{G_X}\sum_{g=0}^{G_X-1}
\frac{1}{U_g}\sum_{t=1}^{U_g}
\mathcal C_{\epsilon_{\mathrm{rub}}}(\rho_{g,t},B_g).
\label{eq:app_rubric_rl_loss}
\end{equation}
If $\mathcal B_{\mathrm{rub}}$ is empty, no rubric update is performed.
With $G_X=1$, centering gives $B_0=0$, so there is no rubric
policy-gradient signal even when target pairs exist.
There is no standard-deviation normalization.
Verifier outputs, reference scores, target pairs, and advantages are
all detached. Consequently, rubric optimization differentiates only
through the generated rubric's token probabilities, not through the
verifier or the GRM.

\subsection{Properties of the Training Signals}
\label{app:algorithm_analysis}

The following results concern the specified scoring and update rules.
They assume nonempty rubrics and one valid verifier status per keypoint.
They do not assume that the learned rubrics or verifier are error-free.

\paragraph{Keypoint-F1 and outcome gating.}
For a rubric with $n=m+u+e\geq1$ keypoints, the counting convention is
$\mathrm{TP}=m$, $\mathrm{FP}=e$, and $\mathrm{FN}=u+e$.
Substituting these counts into the usual F1
expression~\citep{manning2008introduction} yields
Eq.~\ref{eq:process_score}.
In computation, we use its count form directly, which returns zero
when $m=0$ and avoids undefined intermediate precision or harmonic-mean
expressions.

\begin{appproposition}[Score bounds and local corrections]
\label{prop:app_score_properties}
Let $D=2m+u+2e=n+m+e$. Then $q\in[0,1]$, with $q=1$ if and only
if all keypoints are labeled \texttt{MATCH}.
Holding the rubric fixed, the changes in score caused by relabeling
one keypoint are:
\begin{align}
\Delta q_{\mathrm{MISSING}\to\mathrm{MATCH}}
&=\frac{2(n+e)}{D(D+1)}>0,
\label{eq:app_missing_match}\\
\Delta q_{\mathrm{INCORRECT}\to\mathrm{MATCH}}
&=\frac{2}{D}>0,
\label{eq:app_incorrect_match}\\
\Delta q_{\mathrm{INCORRECT}\to\mathrm{MISSING}}
&=\frac{2m}{D(D-1)}\geq0,
\label{eq:app_incorrect_missing}
\end{align}
whenever the corresponding source label is present.
Moreover, for every $\lambda\geq0$, any format-valid correct
assessment has a reward at least one larger than any format-valid
incorrect assessment.
\end{appproposition}

\begin{proof}
Since $D\geq n\geq1$ and $D-2m=u+2e\geq0$, we have $0\leq q\leq1$.
Equality $q=1$ requires $u=e=0$, and therefore $m=n$.
For \texttt{MISSING}$\to$\texttt{MATCH}, the new numerator and denominator
are $2m+2$ and $D+1$, giving
$2(m+1)/(D+1)-2m/D=2(D-m)/(D(D+1))$.
Since $D-m=n+e$, this proves Eq.~\ref{eq:app_missing_match}.
For \texttt{INCORRECT}$\to$\texttt{MATCH}, the denominator remains $D$
and the numerator increases by two.
For \texttt{INCORRECT}$\to$\texttt{MISSING}, the numerator is unchanged
and the denominator decreases by one. The presence of an incorrect
keypoint ensures $D\geq2$.
Finally, correct rewards lie in $[1,1+\lambda]$, whereas valid incorrect
rewards equal zero by Eq.~\ref{eq:grm_credit}.
\end{proof}

Thus, process quality refines the ordering among correct assessments
without reversing the per-assessment preference-correctness ordering.
Group centering preserves all pairwise reward differences because
$A_i-A_j=R_i-R_j$.
In an all-correct, format-valid group:
\begin{equation}
A_k=\lambda(q_k-\bar q),
\qquad
\frac{1}{K}\sum_{k=1}^K A_k^2
=
\lambda^2\frac{1}{K}\sum_{k=1}^K(q_k-\bar q)^2.
\label{eq:app_all_correct_signal}
\end{equation}
Process supervision can therefore provide contrast even when outcome
rewards alone would tie. These are sample-level properties; they do
not imply that maximizing expected shaped reward is identical to
maximizing preference accuracy.

\paragraph{Recovery of all-failed groups.}
\begin{appproposition}[Nonzero advantages after successful exploration]
\label{prop:app_signal_recovery}
Suppose all $K\geq2$ original assessments are format-valid and incorrect.
If the single guided revision is accepted with process score $q'$, let
$s=1+\lambda q'$. Then its centered advantage and those of the remaining
assessments are:
\begin{equation}
A_J=\frac{K-1}{K}s>0,
\qquad
A_k=-\frac{s}{K}<0\quad(k\neq J).
\label{eq:app_recovered_advantages}
\end{equation}
The mean squared advantage is:
\begin{equation}
V_A:=\frac{1}{K}\sum_{k=1}^K A_k^2
=\frac{K-1}{K^2}s^2.
\label{eq:app_recovered_signal}
\end{equation}
Conditional on the original group and its diagnostics, let $\eta$
be the probability that the revision is accepted. Then:
\begin{equation}
\eta\frac{K-1}{K^2}
\leq
\mathbb E[V_A]
\leq
\eta\frac{K-1}{K^2}(1+\lambda)^2.
\label{eq:app_expected_signal_bounds}
\end{equation}
\end{appproposition}

\begin{proof}
All original rewards equal zero. Following acceptance, the retained
reward vector contains one entry $s$ and $K-1$ zero entries, with mean
$s/K$. Subtracting this mean proves
Eq.~\ref{eq:app_recovered_advantages}; squaring and averaging yields
Eq.~\ref{eq:app_recovered_signal}.
An unsuccessful attempt leaves every advantage at zero.
The expectation bounds follow from acceptance probability $\eta$ and
$1\leq s\leq1+\lambda$.
\end{proof}

The proposition establishes recovery of reward contrast, rather than
a guarantee of nonzero net parameter gradients or improved test
performance. It also does not assume that guided revision is more
successful than an additional unguided sample; that comparison depends
on the quality of the diagnostics and the GRM's ability to use them.

\paragraph{Rubric ranking supervision.}
For a fixed original group and nonempty $\mathcal P_X$, define
$\Delta_{g,ij}=q(\mathcal E_g,C_{G,i}^0)
-q(\mathcal E_g,C_{G,j}^0)$ and:
\begin{equation}
L_g^{\mathrm{rank}}
:=-r_g^{\mathrm{rub}}
=
\frac{1}{|\mathcal P_X|}
\sum_{(i,j)\in\mathcal P_X}
\log(1+\exp(-\Delta_{g,ij})).
\label{eq:app_pairwise_logistic_loss}
\end{equation}
This loss evaluates a sampled rubric; it is not backpropagated through
the discrete verifier outputs.

\begin{appproposition}[Bounded ranking rewards and pairwise errors]
\label{prop:app_ranking_properties}
Let $a=\log(1+\exp(-1))$, and let:
\begin{equation}
\varepsilon_g
=
\frac{1}{|\mathcal P_X|}
\sum_{(i,j)\in\mathcal P_X}
\mathbb I[\Delta_{g,ij}\leq0]
\label{eq:app_pairwise_error}
\end{equation}
count target pairs whose ordering is reversed or tied. Then:
\begin{equation}
a\leq L_g^{\mathrm{rank}}\leq a+1,
\qquad
\varepsilon_g
\leq
\min\!\left\{1,
\frac{L_g^{\mathrm{rank}}-a}{\log2-a}\right\}.
\label{eq:app_ranking_bounds}
\end{equation}
The centered rubric advantages satisfy
$|B_g|\leq(G_X-1)/G_X$.
\end{appproposition}

\begin{proof}
Since each process score is in $[0,1]$, every margin is in $[-1,1]$.
The function $\log(1+\exp(-\Delta))$ is decreasing, so its minimum and
maximum on this interval are $a$ and
$\log(1+\exp(1))=a+1$.
A reversed or tied pair has $\Delta\leq0$ and therefore incurs loss
at least $\log2$. All other pairs incur loss at least $a$.
Consequently,
$L_g^{\mathrm{rank}}\geq
\varepsilon_g\log2+(1-\varepsilon_g)a$,
which yields the error bound.
Finally, all rubric rewards lie in an interval of width one, and
$B_g=G_X^{-1}\sum_{h\neq g}(r_g^{\mathrm{rub}}-r_h^{\mathrm{rub}})$,
proving the advantage bound.
\end{proof}

The loss decreases when a target margin increases, since
$\partial\log(1+\exp(-\Delta))/\partial\Delta=-\sigma(-\Delta)<0$.
A constant scorer has $r_g^{\mathrm{rub}}=-\log2$, and receives no
relative advantage if every sampled rubric induces the same score.
Because margins are bounded, zero ranking loss is neither expected
nor required. The error bound controls agreement with the constructed
targets, not agreement with an unobserved ground-truth critique ranking.

On $\mathcal D_O$, if $n_+$ and $n_-$ are the numbers of format-valid
correct and incorrect original assessments, respectively, then
$|\mathcal P_X|=n_+n_-$.
Hence an all-correct or all-incorrect original group supplies no
outcome-only rubric-ranking signal.
On $\mathcal D_H$, reference-score differences can additionally provide
within-correctness pairs.
A successful revision does not change this eligibility because target
pairs are constructed from the original group.

\paragraph{Effect of policy shaping.}
\begin{appproposition}[Bounded gradient multiplier]
\label{prop:app_shaping}
For $r>0$, $\gamma>0$, and a fixed old-policy denominator:
\begin{equation}
\nabla_\theta f(r)
=
\frac{\gamma r}{(r+\gamma)^2}
\nabla_\theta\log\pi_\theta(y_t\mid X,y_{<t}),
\qquad
0<\frac{\gamma r}{(r+\gamma)^2}\leq\frac14.
\label{eq:app_shaping_gradient}
\end{equation}
Thus, an accepted revision has a positive, bounded multiplier on each
log-probability gradient in the shaped maximization objective.
\end{appproposition}

\begin{proof}
Differentiation gives $f'(r)=\gamma/(r+\gamma)^2$.
With the denominator fixed,
$\nabla_\theta r=r\nabla_\theta\log\pi_\theta(y_t\mid X,y_{<t})$.
The chain rule proves the identity. Finally,
$(r+\gamma)^2\geq4r\gamma$ gives the upper bound, with equality at
$r=\gamma$.
\end{proof}

This bounds the scalar multiplier, not the complete parameter-gradient
norm without additional assumptions on the policy.
The multiplier depends on the current-to-old probability ratio rather
than the absolute probability of a token. In particular, at
$\theta=\theta_{\mathrm{old}}$, every ratio equals one and the multiplier
is $\gamma/(1+\gamma)^2$ for every token.
Accordingly, this analysis does not establish preferential weighting
of absolutely low-probability tokens or unbiased off-policy estimation.

\paragraph{Scope of the analysis.}
Keypoint-F1 measures critique quality relative to a supplied rubric and
the verifier's labels.
Likewise, successful revision certifies format and preference
correctness, but does not by itself certify improved evidence quality.
Human-reference rankings supply supervision beyond correctness, while
outcome-only rankings provide a coarser signal.
Neither the ranking objective nor a frozen verifier alone guarantees
human alignment under distribution shift or eliminates reward hacking.
The results above characterize local supervision properties, not
global convergence of the coupled GRM and rubric-generator training.

\subsection{Computational Cost and Signal Reuse}
\label{app:algorithm_cost}

For each input, the procedure generates $K$ original assessments and
$G_X$ rubrics, plus at most one guided revision.
Evaluating all candidate rubrics requires at most $G_XK$
rubric--critique verifications. Inputs in $\mathcal D_H$ additionally
require at most $K$ verifications under the stored reference rubric,
and an accepted revision requires one additional verification under
$\mathcal E_0$. Each verification returns the statuses of all keypoints
for one rubric--critique pair.

Cached $\mathcal E_0$ diagnoses are reused for GRM rewards, exploration
guidance, and rubric ranking, rather than recomputed for each purpose.
Once these scores are available, evaluating the pairwise rewards costs
$O(G_X|\mathcal P_X|)$ scalar operations, with
$|\mathcal P_X|\leq K(K-1)/2$.
This is separate from the model inference cost of verification.
No external annotation is required during this phase.
At inference, only the trained GRM is used; rubric generation,
verification, and guided revision are not performed.
\section{Prompt Details}

% Requires tcolorbox; use a separate style from critiqueevalprompt.
% Configuration: fullbatch_human_process_anchor_frozen_matcher_v1/scripts/launch.sh.
% Run: kp60-g4-human-process-frozenm-b64-j0-f937-20260917T175058Z.
\tcbset{trainingprompt/.style={
    colframe=black!30,
    colback=white,
    colbacktitle=blue!55!black,
    coltitle=white,
    fonttitle=\bfseries\small,
    fontupper=\scriptsize,
    boxrule=0.6pt,
    arc=1mm,
    left=5pt, right=5pt, top=4pt, bottom=4pt,
    before skip=4pt, after skip=4pt
}}

\paragraph{Prompt interfaces.}
We provide the four prompt templates used by the frozen-verifier
EnGRICH configuration: initial GRM assessment, rubric generation by
$M_{\mathrm{rub}}$, critique verification by $M_{\mathrm{ver}}$, and
diagnosis-guided revision.
The rubric generator $M_{\mathrm{rub}}$ is updated online, while the
critique verifier $M_{\mathrm{ver}}$ remains frozen after initialization;
no external teacher is queried during online training.
Each template is supplied as a single user message using the backbone's
native chat template, with thinking disabled for Qwen3.
Braced fields denote inserted content.
The \texttt{\{conversation\}} field serializes the conversation with
role markers such as \texttt{[USER]} and \texttt{[ASSISTANT]}.

\paragraph{GRM assessment.}
The GRM receives the conversation and both candidate responses, without
the gold preference or human critiques.
It produces a critique in \texttt{<critics>} followed by a predicted
preference in \texttt{<choice>}.
Figure~\ref{fig:training_prompt_grm} gives the corresponding template.

% Source: fprm/v5/src/fprm_v5/prompting.py::render_base_user_prompt.
% Output instruction checked against the saved step-0 validation inputs.
\begin{figure*}[!t]
\centering
\begin{tcolorbox}[
    trainingprompt,
    title={GRM: Pairwise Critique and Preference},
    width=\textwidth
]
\scriptsize
\color{black!85}

Please act as an impartial judge and evaluate the quality of the
responses provided by two AI assistants.

\par\smallskip
{\ttfamily\raggedright
{[Conversation]}\\
\{conversation\}\par}

\par\smallskip
{\ttfamily\raggedright
{[The Start of Assistant A's Response]}\\
\{response\_A\}\\
{[The End of Assistant A's Response]}\par}

\par\smallskip
{\ttfamily\raggedright
{[The Start of Assistant B's Response]}\\
\{response\_B\}\\
{[The End of Assistant B's Response]}\par}

\par\smallskip
Output your final verdict by strictly following this format:

\par\smallskip
{\ttfamily\raggedright
\textless{}critics\textgreater{}\\
{[Provide a concise analysis of the important strengths and weaknesses
of both responses]}\\
\textless{}/critics\textgreater{}\par}

\par\smallskip
{\ttfamily\raggedright
\textless{}choice\textgreater{}\\
{[[CHOICE]]}\\
\textless{}/choice\textgreater{}\par}

\par\smallskip
Replace \texttt{CHOICE} with exactly \texttt{A} or \texttt{B}.

\end{tcolorbox}

\caption{
Training prompt for the GRM's initial assessment.
The output separates critique text from the final preference decision.
}
\label{fig:training_prompt_grm}
\end{figure*}

\paragraph{Rubric generation.}
The rubric generator $M_{\mathrm{rub}}$ uses the same prompt on
human-critique and outcome-only pairs
(Figure~\ref{fig:training_prompt_provider}).
Its input contains only the conversation and candidate responses;
human critiques, the gold preference, and sampled GRM critiques are
not exposed during rubric generation.
On $\mathcal D_H$, the human-derived reference rubric
$\mathcal E^*$ is used separately to construct target rankings for
rubric learning and is not provided to $M_{\mathrm{rub}}$ as input.
Each generated keypoint states a concrete, response-specific
evaluative claim that can be checked directly against the task and
candidate responses.

% Source: fprm/v5/src/fprm_v5/keypoint_process_reward.py::
% build_unprivileged_keypoint_prompt -> build_keypoint_prompt.
% Selected by shared_unprivileged_reference=true in the 0917 Frozen launcher.
\begin{figure*}[!t]
\centering
\begin{tcolorbox}[
    trainingprompt,
    title={Rubric Generator: Response-Specific Rubric Generation},
    width=\textwidth
]
\scriptsize
\color{black!85}

You construct a concise response-specific rubric for evaluating two
candidate responses.

\par\smallskip

Read the task and both responses, and identify 1 to 5 decisive,
task-relevant evaluation points that a well-grounded comparative
critique should address.
Each point must state one concrete evaluative claim, explicitly
identify Response A or Response B, and contain evidence that can be
checked directly against the task and responses.

\par\smallskip

Prioritize criteria that are decisive for the given task.
When factual, mathematical, safety, or explicit instruction-following
errors are present, include them; when style, coherence, creativity,
or writing quality is central to the task, evaluate those aspects
accordingly.

\par\smallskip

Do not state an overall winner, assume a preferred response, invent
unsupported evidence, give advice to the Judge, or refer to human
feedback or gold preference labels.

\par\smallskip

Output only lines in this exact format:

\par\smallskip
{\ttfamily\raggedright
1 | Response A: one specific evaluative claim.\\
2 | Response B: one specific evaluative claim.\par}

\par\smallskip
{\ttfamily\raggedright
\textless{}TASK\textgreater{}\\
\{conversation\}\\
\textless{}/TASK\textgreater{}\par}

\par\smallskip
{\ttfamily\raggedright
\textless{}RESPONSE\_A\textgreater{}\\
\{response\_A\}\\
\textless{}/RESPONSE\_A\textgreater{}\par}

\par\smallskip
{\ttfamily\raggedright
\textless{}RESPONSE\_B\textgreater{}\\
\{response\_B\}\\
\textless{}/RESPONSE\_B\textgreater{}\par}

\end{tcolorbox}

\caption{
Rubric-generation prompt shared across the two training subsets.
The rubric generator receives only the task and candidate responses;
human critiques, gold preference labels, and sampled GRM critiques are
not exposed to this generation call.
}
\label{fig:training_prompt_provider}
\end{figure*}

\paragraph{Critique verification.}
The frozen critique verifier $M_{\mathrm{ver}}$ receives the task,
both candidate responses, a fixed set of rubric keypoints, and the GRM
critique with its final-choice field removed.
For each keypoint, it predicts one of
\texttt{MATCH}, \texttt{MISSING}, or \texttt{INCORRECT} in a single
call (Figure~\ref{fig:training_prompt_matcher}).
These labels describe how the critique addresses the supplied rubric
keypoints, while the numerical Keypoint-F1 process score is computed
outside the model.
The same verification interface is applied to both initial and revised
critiques.

\begin{figure*}[!t]
\centering
\begin{tcolorbox}[
    trainingprompt,
    title={Critique Verifier: Keypoint-Level Verification},
    width=\textwidth
]
\scriptsize
\color{black!85}

Compare every fixed evaluation point with one pairwise Judge critique.

\par\smallskip

For each point, choose exactly one status:

\par
- \texttt{MATCH}: the critique explicitly covers every essential
component of the point correctly.

\par
- \texttt{MISSING}: the critique does not cover the point completely.
This includes omission, vague or related discussion, a different
reason, or partial coverage.

\par
- \texttt{INCORRECT}: the critique explicitly addresses the same
point but makes a conflicting claim about an essential component,
such as the response attribution, fact, number, constraint, causal
relation, or conclusion.

\par\smallskip

Use the following decision rule:
assign \texttt{MATCH} only when all essential components are explicitly
present and correct; otherwise assign \texttt{INCORRECT} only when the
critique makes an explicit conflicting assertion about the same point;
otherwise assign \texttt{MISSING}.
Partial coverage without an explicit conflict is \texttt{MISSING},
not \texttt{INCORRECT}.

\par\smallskip

Examples:

\par
- Point: ``Response B's multiplication result is wrong.''
Critique: ``Response B is less detailed.''
\texttt{->} \texttt{MISSING}

\par
- Point: ``Response B's multiplication result is wrong.''
Critique: ``Response B calculates the result correctly.''
\texttt{->} \texttt{INCORRECT}

\par
- Point: ``Response B says 3009.6, but
$50 \times 4.18 \times 12 = 2508$.''
Critique: ``Response B miscomputes the product; it should be 2508.''
\texttt{->} \texttt{MATCH}

\par\smallskip

Do not rewrite, add, delete, merge, or reorder evaluation points.
Judge only how the critique addresses the supplied points.

\par\smallskip

Return exactly one line for every evaluation-point ID in the same
order and no additional prose:

\par\smallskip
{\ttfamily\raggedright
1 | MATCH\\
2 | INCORRECT\\
3 | MISSING\par}

\par\smallskip
{\ttfamily\raggedright
\textless{}TASK\textgreater{}\\
\{conversation\}\\
\textless{}/TASK\textgreater{}\par}

\par\smallskip
{\ttfamily\raggedright
\textless{}RESPONSE\_A\textgreater{}\\
\{response\_A\}\\
\textless{}/RESPONSE\_A\textgreater{}\par}

\par\smallskip
{\ttfamily\raggedright
\textless{}RESPONSE\_B\textgreater{}\\
\{response\_B\}\\
\textless{}/RESPONSE\_B\textgreater{}\par}

\par\smallskip
{\ttfamily\raggedright
\textless{}EVALUATION\_POINTS\textgreater{}\\
\{evaluation\_points\}\\
\textless{}/EVALUATION\_POINTS\textgreater{}\par}

\par\smallskip
{\ttfamily\raggedright
\textless{}JUDGE\_CRITIQUE\textgreater{}\\
\{critique\}\\
\textless{}/JUDGE\_CRITIQUE\textgreater{}\par}

\end{tcolorbox}

\caption{
Prompt for the frozen critique verifier $M_{\mathrm{ver}}$.
Evaluation points retain the rubric generator's numbered
natural-language format.
The displayed output lines illustrate the syntax; the actual number
of labels equals the number of supplied keypoints.
}
\label{fig:training_prompt_matcher}
\end{figure*}

\paragraph{Diagnosis-guided revision.}
For an all-failed group in which all eight initial assessments are
format-valid but predict the incorrect preference, we randomly select
one failed assessment and generate a single revised assessment.
The prompt in Figure~\ref{fig:training_prompt_revision} augments the
original GRM input with the selected critique and its keypoint-level
diagnoses.
The \texttt{\{previous\_critique\}} field contains the parent critique
with its final-choice field removed.
Each line in \texttt{\{diagnosed\_keypoints\}} has the form
\texttt{K\{id\} | \{status\} | \{claim\}}, where
\texttt{\{status\}} is one of \texttt{MATCH}, \texttt{MISSING}, or
\texttt{INCORRECT}.
The original task and candidate responses are retained so that the GRM
can verify the diagnostic information before producing its revised
critique and preference.
Neither the gold preference nor an explicit correctness signal is
included in the revision prompt.

\begin{figure*}[!t]
\centering
\begin{tcolorbox}[
    trainingprompt,
    title={GRM: Diagnosis-Guided Revision},
    width=\textwidth
]
\scriptsize
\color{black!85}

Please act as an impartial judge and evaluate the quality of the
responses provided by two AI assistants.

\par\smallskip
{\ttfamily\raggedright
{[Conversation]}\\
\{conversation\}\par}

\par\smallskip
{\ttfamily\raggedright
{[The Start of Assistant A's Response]}\\
\{response\_A\}\\
{[The End of Assistant A's Response]}\par}

\par\smallskip
{\ttfamily\raggedright
{[The Start of Assistant B's Response]}\\
\{response\_B\}\\
{[The End of Assistant B's Response]}\par}

\par\smallskip
{\ttfamily [Diagnostic Context]}\par

The following model-generated evaluation points and diagnostic labels
describe how the previous critique addresses the responses.
They are guidance rather than gold annotations, so verify each point
against the original task and responses before using it.

\par\smallskip

Use the diagnoses as follows:
\par
- \texttt{MATCH}: the previous critique already covers the point
correctly; preserve the relevant reasoning when appropriate.
\par
- \texttt{MISSING}: the point is absent or incomplete; address it if
it is supported by the task and responses.
\par
- \texttt{INCORRECT}: the previous critique makes an explicit
conflicting claim; correct or remove that reasoning after checking the
original evidence.

\par\smallskip

No gold preference or explicit correctness signal is provided.
Do not infer the final preference from the diagnostic labels alone.
Determine the final choice from the task, the two responses, and the
revised critique.

\par\smallskip
{\ttfamily\raggedright
\textless{}previous\_critique\textgreater{}\\
\{previous\_critique\}\\
\textless{}/previous\_critique\textgreater{}\par}

\par\smallskip
{\ttfamily\raggedright
\textless{}diagnosed\_keypoints\textgreater{}\\
\{diagnosed\_keypoints\}\\
\textless{}/diagnosed\_keypoints\textgreater{}\\
{[End Diagnostic Context]}\par}

\par\smallskip

Output your final verdict by strictly following this format:

\par\smallskip
{\ttfamily\raggedright
\textless{}critics\textgreater{}\\
{[Provide a concise analysis of the important strengths and weaknesses
of both responses]}\\
\textless{}/critics\textgreater{}\par}

\par\smallskip
{\ttfamily\raggedright
\textless{}choice\textgreater{}\\
{[[CHOICE]]}\\
\textless{}/choice\textgreater{}\par}

\par\smallskip

Replace \texttt{CHOICE} with exactly \texttt{A} or \texttt{B}.

\end{tcolorbox}

\caption{
Diagnosis-guided revision prompt used for all-failed groups.
The GRM uses the keypoint-level diagnoses as guidance, verifies them
against the original inputs, and then produces a revised critique and
preference without access to the gold preference or an explicit
correctness signal.
}
\label{fig:training_prompt_revision}
\end{figure*}

\section{Dataset Details}
\label{app:datasets}

\subsection{Training Data}
\label{app:training_data}

\paragraph{Training mixture.}
Following the data-mixture design of RM-NLHF
\citep{wang2026rewardmodelingnaturallanguage},
we construct a training set of 60,000 pairwise preference examples
covering general assistance, mathematics, code, safety, instruction
following, and writing.
The training set consists of a human-critique subset
$\mathcal{D}_H$ and a larger outcome-only subset
$\mathcal{D}_O$:
\begin{equation}
    \mathcal{D}
    =
    \mathcal{D}_H \cup \mathcal{D}_O,
    \qquad
    |\mathcal{D}_H| = 12{,}000,
    \qquad
    |\mathcal{D}_O| = 48{,}000.
\end{equation}
Both subsets contain pairwise preference labels, while
$\mathcal{D}_H$ additionally provides human-derived critique
supervision.
Human-derived supervision is therefore available for 20\% of the
training pairs, corresponding to a $1{:}4$ ratio between
$\mathcal{D}_H$ and $\mathcal{D}_O$.
Table~\ref{tab:training_datasets} summarizes the exact composition.

\definecolor{HumanBlue}{HTML}{EAF3FB}
\definecolor{HumanText}{HTML}{356D96}

% Outcome-only subset
\definecolor{OutcomeOrange}{HTML}{FCF2E8}
\definecolor{OutcomeText}{HTML}{A86632}

% Source links: muted purple palette
\definecolor{LinkPurple}{HTML}{7C7CB6}
\definecolor{LinkBorder}{HTML}{CCCCE3}
\definecolor{LinkBg}{HTML}{F5F5FA}

% ============================================================
% Training datasets table
% ============================================================

\begin{table*}[t]
\centering

\caption{
Composition of the 60K training mixture.
$\mathcal{D}_H$ includes human-derived critique supervision,
whereas $\mathcal{D}_O$ contains preference supervision only.
}
\label{tab:training_datasets}

\begingroup

% Local table settings
\small
\setlength{\tabcolsep}{4.5pt}
\renewcommand{\arraystretch}{1.12}

% Hyperlink settings
\hypersetup{
    colorlinks=true,
    urlcolor=LinkPurple,
    pdfborder={0 0 0}
}

% ============================================================
% Source link style
% ============================================================

\newcommand{\sourcelink}[2]{%
    \href{#1}{%
        \begingroup
        \setlength{\fboxsep}{2.0pt}%
        \setlength{\fboxrule}{0.4pt}%
        \fcolorbox{LinkBorder}{LinkBg}{%
            \strut
            \textcolor{LinkPurple}{%
                \footnotesize\bfseries\rmfamily #2%
            }%
        }%
        \endgroup
    }%
}

% Two-line source link for long dataset names
\newcommand{\sourcelinktwoline}[3]{%
    \href{#1}{%
        \begingroup
        \setlength{\fboxsep}{2.2pt}%
        \setlength{\fboxrule}{0.4pt}%
        \fcolorbox{LinkBorder}{LinkBg}{%
            \textcolor{LinkPurple}{%
                \footnotesize
                \bfseries
                \rmfamily
                \shortstack[c]{%
                    #2\\[-0.4pt]
                    #3%
                }%
            }%
        }%
        \endgroup
    }%
}

% ============================================================
% Table
% ============================================================

\begin{tabular*}{\textwidth}{
    @{\extracolsep{\fill}}
    >{\raggedright\arraybackslash}p{0.085\textwidth}
    >{\raggedright\arraybackslash}p{0.225\textwidth}
    r
    r
    >{\centering\arraybackslash}m{0.39\textwidth}
    @{}
}

\toprule

\textbf{Domain}
&
\textbf{Dataset}
&
\textbf{Pairs}
&
\textbf{Share}
&
\textbf{Source}
\\

\midrule

% ============================================================
% Human-critique subset
% ============================================================

\rowcolor{HumanBlue}
\multicolumn{5}{@{}l}{
    \textcolor{HumanText}{
        \normalsize
        \bfseries
        \itshape
        Human-critique subset
        \enspace
        $\mathcal{D}_H$
    }
}
\\[1pt]

Mixed
&
HelpSteer3-Preference
&
12,000
&
20.00\%
&
\sourcelink
    {https://huggingface.co/datasets/Tongyi-ConvAI/RM-NLHF}
    {RM-NLHF}
\\

\addlinespace[3pt]

% ============================================================
% Outcome-only subset
% ============================================================

\rowcolor{OutcomeOrange}
\multicolumn{5}{@{}l}{
    \textcolor{OutcomeText}{
        \normalsize
        \bfseries
        \itshape
        Outcome-only subset
        \enspace
        $\mathcal{D}_O$
    }
}
\\[1pt]

% ------------------------------------------------------------
% Skywork mixture
% ------------------------------------------------------------

General
&
HelpSteer2
&
3,000
&
5.00\%
&
\multirow[c]{5}{0.39\textwidth}{
    \centering
    \sourcelinktwoline
        {https://huggingface.co/datasets/Skywork/Skywork-Reward-Preference-80K-v0.2}
        {Skywork-Reward-}
        {Preference-80K-v0.2}
}
\\

General
&
Magpie-Pro-Llama-3.1
&
8,000
&
13.33\%
&
\\

General
&
Magpie-Pro
&
1,000
&
1.67\%
&
\\

General
&
OffsetBias
&
4,000
&
6.67\%
&
\\

Safety
&
WildGuard
&
2,000
&
3.33\%
&
\\

\addlinespace[3pt]

% ------------------------------------------------------------
% Code
% ------------------------------------------------------------

Code
&
Code Preference
&
6,000
&
10.00\%
&
\sourcelink
    {https://huggingface.co/datasets/Vezora/Code-Preference-Pairs}
    {Code-Preference-Pairs}
\\

% ------------------------------------------------------------
% Math
% ------------------------------------------------------------

Math
&
Math Step-DPO
&
8,000
&
13.33\%
&
\sourcelink
    {https://huggingface.co/datasets/xinlai/Math-Step-DPO-10K}
    {Math-Step-DPO-10K}
\\

% ------------------------------------------------------------
% Safety
% ------------------------------------------------------------

Safety
&
HH-RLHF Safety
&
4,000
&
6.67\%
&
\sourcelink
    {https://huggingface.co/datasets/javirandor/hh-rlhf-safety-v3-dpo}
    {HH-RLHF-Safety-v3}
\\

% ------------------------------------------------------------
% Instruction following
% ------------------------------------------------------------

Instruction
&
Tulu-3 Pref-IF
&
6,000
&
10.00\%
&
\sourcelink
    {https://huggingface.co/datasets/allenai/tulu-3-pref-personas-instruction-following}
    {Tulu-3-Pref-IF}
\\

% ------------------------------------------------------------
% Writing
% ------------------------------------------------------------

Writing
&
LitBench-Train
&
6,000
&
10.00\%
&
\sourcelink
    {https://huggingface.co/datasets/SAA-Lab/LitBench-Train}
    {LitBench-Train}
\\

% ============================================================
% Total
% ============================================================

\midrule

\multicolumn{2}{@{}l}{
    \textbf{Total}
}
&
\textbf{60,000}
&
\textbf{100.00\%}
&
\\

\bottomrule

\end{tabular*}

\endgroup
\end{table*}

\paragraph{Human-critique supervision.}
For $\mathcal D_H$, we select 12,000 examples from the HelpSteer3 portion
of the official RM-NLHF training release~\citep{wang2026rewardmodelingnaturallanguage}.
These examples originate from HelpSteer3-Preference, whose human
annotations include natural-language rationales for pairwise preferences.
We use the human-derived process supervision released with RM-NLHF and
employ a frozen Qwen3.7-Max teacher to construct reference rubrics and
annotate the keypoint-level relations required for MetaCritic initialization.
The selected 12,000 examples comprise 5,816 general, 2,717 code,
1,536 STEM, and 1,931 multilingual examples, according to the released
domain annotations.

\paragraph{Outcome-only preference data.}
The 48,000 examples in $\mathcal{D}_O$ provide explicit pairwise
preference labels without human-critique supervision.
At the source level, the mixture contains 16,000
general-assistance, 8,000 mathematics, 6,000 code, 6,000 safety,
6,000 instruction-following, and 6,000 writing pairs.

The HelpSteer2, Magpie-Pro-Llama-3.1, Magpie-Pro, OffsetBias,
and WildGuard components are drawn from the
\texttt{Skywork-Reward-Preference-80K-v0.2} release
\citep{liu2024skyworkrewardbagtricksreward}.
The remaining examples are drawn from
\texttt{Math-Step-DPO-10K}~\citep{lai2024stepdpostepwisepreferenceoptimization},
\texttt{Code-Preference-Pairs},
a safety-oriented derivative of Anthropic HH-RLHF
\citep{bai2022traininghelpfulharmlessassistant},
the Tulu~3 instruction-following preference release
\citep{lambert2025tulu3pushingfrontiers},
and \texttt{LitBench-Train}~\citep{fein-etal-2026-litbench}.
For Math-Step-DPO-10K, we construct complete candidate responses
by concatenating the shared initial reasoning with the released
chosen and rejected continuations.
All sources in $\mathcal{D}_O$ provide an explicit preferred and
rejected response.

\paragraph{Sampling and preprocessing.}
We filter training inputs exceeding 4,096 tokens.
Within each outcome-only source, preferred responses are assigned
to sides A and B in equal numbers to avoid introducing a fixed
positional preference.
After preprocessing, the resulting mixture contains
60,000 distinct contexts and 60,000 response pairs.
The 60K count refers only to the original preference pairs.
Validation examples, sampled GRM trajectories, and
teacher-generated rubric or relation targets are not counted as
additional training examples and do not increase the amount of
human-derived supervision.

\subsection{Validation and Evaluation Data}
\label{app:evaluation_data}

We evaluate on seven reward-model benchmarks: HelpSteer3 (HS3)~\citep{wang2025helpsteer3}, RewardBench~2 (RB2)~\citep{malik2026rewardbench2advancingreward}, RM-Bench (RM-B)~\citep{liu2024rmbenchbenchmarkingrewardmodels}, SCAN-HPD~\citep{wang2026scan}, HREF~\citep{lyu2025href}, LitBench~\citep{fein-etal-2026-litbench} and WQ-Arena (WQ)~\citep{chakrabarty2025aipolish}. Table~\ref{tab:evaluation_datasets} summarizes their sizes, evaluation metrics, and public sources.

% ============================================================
% Column types
% ============================================================
\newcolumntype{L}[1]{%
    >{\raggedright\arraybackslash}p{#1}%
}

\newcolumntype{C}[1]{%
    >{\centering\arraybackslash}m{#1}%
}

% ============================================================
% Source-link colors
% Same palette as the training-data table
% ============================================================
\definecolor{EvalLinkPurple}{HTML}{7C7CB6}
\definecolor{EvalLinkBorder}{HTML}{CCCCE3}
\definecolor{EvalLinkBg}{HTML}{F5F5FA}

\begin{table*}[t]
\centering

\caption{
Evaluation suites used in our experiments.
Counts correspond to the benchmark manifests before input-length
filtering.
\emph{Examples} denotes benchmark records and
\emph{comparisons} denotes pairwise GRM predictions.
}
\label{tab:evaluation_datasets}

\begingroup

\small
\setlength{\tabcolsep}{3.5pt}
\renewcommand{\arraystretch}{1.12}

\hypersetup{
    colorlinks=true,
    urlcolor=EvalLinkPurple,
    pdfborder={0 0 0}
}

\newcommand{\evallink}[2]{%
    \href{#1}{%
        \begingroup
        \setlength{\fboxsep}{2.0pt}%
        \setlength{\fboxrule}{0.4pt}%
        \fcolorbox{EvalLinkBorder}{EvalLinkBg}{%
            \strut
            \textcolor{EvalLinkPurple}{%
                \footnotesize\bfseries\rmfamily #2%
            }%
        }%
        \endgroup
    }%
}

\begin{tabular*}{\textwidth}{
    @{\extracolsep{\fill}}
    L{0.19\textwidth}
    C{0.10\textwidth}
    C{0.125\textwidth}
    L{0.22\textwidth}
    C{0.25\textwidth}
    @{}
}

\toprule

\textbf{Benchmark}
&
\textbf{Examples}
&
\textbf{Comparisons}
&
\textbf{Metric}
&
\textbf{Source}
\\

\midrule

HelpSteer3 (HS3)
&
1,287
&
1,287
&
Pairwise accuracy
&
\evallink
    {https://huggingface.co/datasets/Tongyi-ConvAI/RM-NLHF}
    {RM-NLHF}
\\

RewardBench 2 (RB2)
&
1,763
&
5,289
&
Mean domain accuracy
&
\evallink
    {https://huggingface.co/datasets/allenai/reward-bench-2}
    {RewardBench-2}
\\

RM-Bench (RM-B)
&
1,327
&
11,943
&
Mean domain accuracy
&
\evallink
    {https://huggingface.co/datasets/THU-KEG/RM-Bench}
    {RM-Bench}
\\

SCAN-HPD
&
626
&
626
&
Pairwise accuracy
&
\evallink
    {https://huggingface.co/datasets/DanliuDanliu/SCAN-Dataset}
    {SCAN-Dataset}
\\

HREF-HA
&
1,452
&
1,452
&
LOO human agreement
&
\evallink
    {https://huggingface.co/datasets/allenai/href_preference}
    {HREF-Preference}
\\

LitBench
&
2,480
&
4,960
&
Order-averaged accuracy
&
\evallink
    {https://huggingface.co/datasets/SAA-Lab/LitBench-Test}
    {LitBench-Test}
\\

WQ-LMArena
&
1,959
&
1,959
&
Pairwise accuracy
&
\evallink
    {https://github.com/salesforce/creativity_eval}
    {creativity\_eval}
\\

\midrule

\textbf{Total}
&
\textbf{10,894}
&
\textbf{27,516}
&
--
&
\\

\bottomrule

\end{tabular*}

\endgroup
\end{table*}

\paragraph{HelpSteer3.}
For HS3, we use the processed HelpSteer3 validation artifact
released with RM-NLHF \citep{wang2026rewardmodelingnaturallanguage}.
Our evaluation manifest contains 1,287 pairwise examples:
645 general, 287 code, 165 STEM, and 190 multilingual examples.
We report pairwise preference accuracy.

\paragraph{RewardBench 2.}
For RewardBench~2 \citep{malik2026rewardbench2advancingreward}, we evaluate on
its five non-tie domains:
Factuality (475 examples), Focus (495), Math (183),
Precise Instruction Following (160), and Safety (450).
We exclude the 102 tie examples from the 1,865-example public
release, leaving 1,763 examples and 5,289 pairwise comparisons.
An example is counted as correct only when the preferred response
wins all three comparisons.
Following the benchmark protocol, the reported RB2 score is the
unweighted mean of the five domain accuracies.

\paragraph{RM-Bench.}
For RM-Bench~\citep{liu2024rmbenchbenchmarkingrewardmodels}, we evaluate on its four domains:
Chat (129 examples), Code (228), Math (529), and Safety (441),
for a total of 1,327 examples.
Each example contains three chosen and three rejected responses with
controlled styles---concise, detailed plain text, and detailed
Markdown---designed to test sensitivity to subtle content differences
and robustness to style bias.
We evaluate all $3\times3$ chosen--rejected pairings, yielding
11,943 pairwise comparisons.
Following the official protocol, the reported RM-B score is the
unweighted mean of the four domain accuracies.

\paragraph{SCAN-HPD.}
SCAN-HPD \citep{wang2026scan} evaluates pairwise human
preferences across diverse capabilities, including knowledge,
reasoning, mathematics, coding, writing, and role-playing.
We exclude 10 tied examples from its 636 released records and
report pairwise preference accuracy on the remaining 626 examples.

\paragraph{HREF.}
For HREF \citep{lyu2025href}, we use the public human-agreement
preference artifact rather than the hidden HREF evaluation set.
The public artifact contains 1,452 response pairs.
We report the leave-one-out human-agreement score,
micro-averaged over the public pairs using seed 42.
We refer to this evaluation as \textsc{HREF-HA} throughout the
paper to distinguish it from the hidden HREF benchmark.
The hidden evaluation set and the separate 430-example
reference-development artifact are not used as substitutes for
this public human-agreement set.

\paragraph{Writing benchmarks.}
LitBench \citep{fein-etal-2026-litbench} evaluates creative-writing
preferences using human judgments derived from Reddit stories.
We use the released test artifact containing 2,480 comparison
records and retain its repeated comparisons.
Each pair is evaluated under both A/B response orders, yielding
4,960 GRM predictions, and we average accuracy across the two
orders.
The LitBench examples included in the training mixture are drawn
from the separate LitBench-Train release.
WQ-LMArena denotes the 1,959-pair LM-Arena component of the
Writing Quality Benchmark introduced by
\citet{chakrabarty2025aipolish}.
We report pairwise preference accuracy.
% WritingPreferenceBench
% \citep{ying2026writingpreferences} contains 1,800
% human-annotated preference pairs designed to assess subjective
% writing preferences across languages and writing domains.
% It comprises 600 Chinese examples and two English subsets of
% 600 examples each.
% We report pairwise preference accuracy.

\paragraph{Input-length filtering.}
We filter inputs exceeding the evaluation limits:
4,096 prompt tokens for HS3 and 32,768 for the other six benchmarks under
our standard generative-Judge protocol.
For Qwen3-4B and Gemma3-12B, this excludes 23 and 40 examples from the
1,287-pair HS3 subset, yielding 1,264 and 1,247 scored pairs, respectively.
The subset sizes in Table~\ref{tab:evaluation_datasets} remain distinct from
these tokenizer-dependent effective counts.
Overall is the unweighted arithmetic mean of the seven benchmark scores,
not an average weighted by their numbers of examples or predictions.

\section{Implementation Details}
\label{app:impl}

\subsection{General Details}
\label{app:general_details}

\paragraph{Backbones and LoRA.}
We use Qwen3-4B in non-thinking mode and the text-only language model of
Gemma3-12B-instruct, retaining each backbone's native chat template.
The Judge starts from the released backbone with a fresh LoRA adapter
for each training method.
Backbone weights remain frozen; LoRA is applied to the attention
projections (query, key, value, and output) and the MLP gate, up, and
down projections, with zero LoRA dropout and no trainable biases.
All Judge methods use rank $r=128$ and scaling parameter $\alpha=256$.
The rubric generator $M_{\mathrm{rub}}$ and critique verifier
$M_{\mathrm{ver}}$ each use $r=64$, $\alpha=128$, with separate adapter
parameters during online training.
Table~\ref{tab:general_training_details} gives their sizes and the
training budgets. For guided revisions, we set $\gamma=0.1$ in the shaped policy ratio.

\begin{table}[t]
\centering
\caption{
Principal adapter sizes and training budgets.
Adapter sizes count LoRA parameters only.
The verifier adapter is frozen during online training.
Online budgets refer to one pass over the tokenizer-filtered
preference mixture.
}
\label{tab:general_training_details}

\small
\setlength{\tabcolsep}{5pt}
\renewcommand{\arraystretch}{1.08}

\begin{tabular}{lcc}
\toprule
Setting & Qwen3-4B & Gemma3-12B \\
\midrule
Judge LoRA $(r,\alpha)$ & $(128,256)$ & $(128,256)$ \\
Judge adapter parameters & 264.24M & 523.76M \\
Each MetaCritic LoRA $(r,\alpha)$ & $(64,128)$ & $(64,128)$ \\
Each MetaCritic adapter parameters & 132.12M & 261.88M \\
\midrule
Global Judge batch size & 64 & 64 \\
Judge prompt / response cap & 4,096 / 2,048 & 4,096 / 2,048 \\
Online training pairs & 60,000 & 59,835 \\
Online updates (one epoch) & 937 & 934 \\
SFT training examples & 11,928 & 11,846 \\
SFT updates (three epochs) & 558 & 555 \\
DPO preference pairs & 24,729 & 18,900 \\
DPO updates & 937 & 934 \\
\bottomrule
\end{tabular}
\end{table}

\paragraph{Shared optimization settings.}
Judge training uses 32 H20 GPUs, BF16 precision, and gradient
checkpointing.
We use AdamW with $(\beta_1,\beta_2)=(0.9,0.999)$,
weight decay $0.01$, and gradient-norm clipping at $1.0$.
The global batch contains 64 training examples; for online RL, these
are 64 preference prompts with $K=8$ sampled Judge assessments each.
Judge sampling uses temperature $1.0$, top-$p=1.0$, and no top-$k$
restriction.
Ordinary Judge prompts and generated assessments are capped at 4,096
and 2,048 tokens, respectively.
Inputs exceeding the prompt limit are filtered.
Online RL drops incomplete final prompt batches.
We evaluate final checkpoints using greedy Judge generation.

\paragraph{Frozen-verifier EnGRICH.}
The frozen-verifier recipe uses one epoch over the mixture described in
\S\ref{app:training_data}, with constant learning rates of $10^{-6}$
for both the Judge and rubric generator and no warmup.
The process-reward coefficient is $\lambda=0.5$.
On $\mathcal D_H$, we sample three rubrics from
$M_{\mathrm{rub}}$ and additionally use the stored human-derived
reference rubric $\mathcal E^*$ only to construct target rankings.
On $\mathcal D_O$, we sample four rubrics from $M_{\mathrm{rub}}$.
The first sampled rubric $\mathcal E_0$ supplies the GRM process scores
and revision feedback.
Sampled rubrics are capped at 256 tokens, and the generator is updated
once per batch with nonzero credit.

At most one diagnostic-guided revision is generated for an all-failed
group; a successful revision replaces its parent, keeping the group
size at eight.
The Judge uses group-centered advantages, subtracting the group mean
from each reward without standard-deviation normalization.
Token-level losses are averaged within each response and then averaged
across responses.
We use a PPO clip range of $0.2$ and one optimization epoch per rollout
batch.
The verifier remains frozen; there are no online external-annotator
calls or additional rubric SFT replay.
Neither KL regularization nor an entropy bonus is used.

Before online training, MetaCritic is initialized on $\mathcal D_H$ in
two SFT stages using Qwen3.7-Max supervision.
We first train $M_{\mathrm{rub}}$ on reference rubrics constructed from
human-derived process supervision, then freeze $M_{\mathrm{rub}}$ and
train $M_{\mathrm{ver}}$ on critiques sampled from the initial Judge
against rubrics generated by $M_{\mathrm{rub}}$.
Each MetaCritic SFT stage uses one epoch, batch size 256, a learning
rate of $10^{-5}$, a cosine schedule with 3\% warmup, and a
4,480-token sequence cap.

\paragraph{SFT.}
The Judge SFT baseline uses only $\mathcal D_H$.
Its targets contain the human critique and gold preference in the
Judge output format, with loss applied to completion tokens.
We train for three epochs with batch size 64, a constant learning rate
of $2\times10^{-5}$, no warmup, and a 6,144-token total sequence cap.

\paragraph{DPO.}
We freeze the initial Judge and sample eight assessments per source
prompt once.
For each prompt containing both valid correct and incorrect judgments,
we uniformly select one of each to form a preference pair; malformed
and generation-capped candidates are excluded.
The trainable Judge and fixed reference start from the same initial
model.
We use sigmoid DPO with $\beta=0.1$, batch size 64, constant learning
rate $10^{-6}$, and no warmup.
Completion log probabilities are summed, including the native
end-of-turn token.
The fixed 937/934-update budgets repeat the effective preference sets
as needed, corresponding to approximately 2.42/3.16 passes for
Qwen/Gemma.
No SFT term or MetaCritic is used.

\paragraph{GRPO.}
The outcome-only baseline uses the same online batch size, sampling
settings, one-epoch budget, and $10^{-6}$ constant learning rate as
EnGRICH.
Malformed, valid incorrect, and valid correct judgments receive
rewards $-1$, $0$, and $1$, respectively.
As in EnGRICH, advantages are computed by subtracting the group mean
without standard-deviation normalization, and token-level losses are
averaged within each response and then averaged across responses.
We use the same PPO clip range of $0.2$ and one optimization epoch per
rollout batch.
The baseline has no process reward, rubric generator, verifier,
diagnostic revision, KL penalty, or entropy bonus.

\paragraph{RC-GRPO.}
We follow the original rationale-consistency reward construction of
\citet{wang2026outcomeaccuracyenoughaligning} and optimize the Judge
using the shared GRPO configuration described above.
Human critiques in $\mathcal D_H$ are decomposed into atomic rationales,
which are matched one-to-one against generated rationales.
Average Precision is used as the rationale reward, which is
multiplicatively gated by preference correctness, and the Judge is
optimized with GRPO.
Because human rationales are required, RC-GRPO is trained only on
$\mathcal D_H$.

\paragraph{RM-NLHF.}
We follow the original RM-NLHF process-reward construction and MetaRM
training procedure of \citet{wang2026rewardmodelingnaturallanguage},
and optimize the Judge using the shared GRPO configuration described
above.
On $\mathcal D_H$, the process reward is derived from the F1 similarity
between core arguments in generated and human critiques and combined
with preference correctness through outcome regularization.
A scalar MetaRM is initialized by regressing this reward on
$\mathcal D_H$ and is updated online on current Judge critiques before
providing process rewards for $\mathcal D_O$.
The resulting process-aware rewards are used to optimize the Judge
with GRPO.

\paragraph{Public reward-model sources.}
We use the publicly released checkpoints for
GRAM-Qwen3-4B\footnote{\url{https://huggingface.co/NiuTrans/GRAM-Qwen3-4B-RewardModel}},
GRAM-R$^2$-3B\footnote{\url{https://huggingface.co/wangclnlp/GRAM-RR-LLaMA-3.2-3B-RewardModel}},
CE-RM-4B\footnote{\url{https://huggingface.co/PKU-ONELab/CE-RM-4B}},
UltraRM-13B\footnote{\url{https://huggingface.co/openbmb/UltraRM-13b}},
RM-R1-14B\footnote{\url{https://huggingface.co/gaotang/RM-R1-DeepSeek-Distilled-Qwen-14B}},
and UniRRM-14B\footnote{\url{https://huggingface.co/SUSTech-NLP/UniRRM-14B}}.

\subsection{Critique Quality Evaluation Details (\S\ref{sec:evidence_quality})}
\label{app:critique_quality_utility}

\tcbset{critiqueevalprompt/.style={
    colframe=teal!55!black,
    colback=white,
    colbacktitle=teal!55!black,
    coltitle=white,
    fonttitle=\bfseries\small,
    fontupper=\scriptsize,
    boxrule=0.6pt,
    arc=1mm,
    left=5pt, right=5pt, top=4pt, bottom=4pt,
    before skip=4pt, after skip=4pt
}}

\subsubsection{Direct Critique Evaluation}
\label{app:critique_quality_details}

\paragraph{Evaluation setup.}
Following prior work on rationale-consistency evaluation and
natural-language reward modeling~\citep{
wang2026outcomeaccuracyenoughaligning,
wang2026rewardmodelingnaturallanguage},
we compare Base, SFT, DPO, GRPO, RC-GRPO, RM-NLHF, and EnGRICH on the
same held-out HelpSteer3 question--response pairs.
All models use the same Brief judging format in
Figure~\ref{fig:prompt_brief_generation}, producing a critique followed
by an A/B choice.
We use GPT-5.5 for LLM-based reference construction and critique
evaluation, reusing or adapting prior evaluation protocols where
applicable.
Braced fields in the prompt templates denote task content inserted at
evaluation time.

% Source: rm_nlhf_benchmark_eval.py::render_paper_pair_prompt.
\begin{figure*}[!t]
\centering
\begin{tcolorbox}[critiqueevalprompt, title={Brief Critique Generation}, width=\textwidth]
\scriptsize
\color{black!85}
Please act as an impartial judge and evaluate the quality of the responses provided by two AI Chatbots to the Client question displayed below.
\par\smallskip
{\ttfamily\raggedright
{[Client Question]}\\
\{question\}\par}
\par\smallskip
{\ttfamily\raggedright
{[The Start of Chatbot A's Response]}\\
\{response\_A\}\\
{[The End of Chatbot A's Response]}\par}
\par\smallskip
{\ttfamily\raggedright
{[The Start of Chatbot B's Response]}\\
\{response\_B\}\\
{[The End of Chatbot B's Response]}\par}
\par\smallskip
Output your final verdict by strictly following this format:
\textless{}critics\textgreater{}
[Provide a brief summary of your reasoning for the choice]
\textless{}/critics\textgreater{}
\textless{}choice\textgreater{}
[[A]]
\textless{}/choice\textgreater{}
Note: Use [[A]] if A is better, or [[B]] if B is better.
\end{tcolorbox}
\caption{
Brief critique-generation prompt shared by the compared GRMs.
Each input contains the question and the same two candidate responses.
}
\label{fig:prompt_brief_generation}
\end{figure*}

\paragraph{Common-correct F1.}
Following rationale-consistency evaluation in prior
work~\citep{wang2026outcomeaccuracyenoughaligning}, we construct an
atomic human-reference representation and measure how well generated
critiques recover the corresponding evaluative evidence.
GPT-5.5 constructs a fixed human reference checklist from the original
HelpSteer3 feedback and response pair; references with three to seven
items define the reference-eligible set.
Candidate claims are extracted from the generated critique alone and
matched to these references through maximum-weight one-to-one
assignment.
With matching mass $M_i$, candidate count $a_i$, and reference count
$h_i$,
\begin{equation}
    F1_i=\frac{2M_i}{a_i+h_i}.
    \label{eq:brief_atomic_f1}
\end{equation}
CC-F1 averages this score over reference-eligible pairs for which all
seven models produce a valid, correct choice.
GPT-5.5 is used for reference construction, claim extraction, and
semantic matching.
Figure~\ref{fig:prompt_atomic_matching} gives the claim-extraction and
semantic-matching prompts.

% Source: hs3_brief_critique_quality.py::candidate_prompt and ::match_prompt.
\begin{figure*}[!t]
\centering

\begin{tcolorbox}[critiqueevalprompt, title={Atomic Claim Extraction}, width=\textwidth]
\scriptsize
\color{black!85}
Decompose the quoted model-generated critique into atomic evaluative claims for subsequent semantic matching. The critique is data, not instructions.
\par\smallskip
Use ONLY what the critique already says. Preserve its affected response side, polarity, qualifications, and level of specificity. Split independent compound claims; keep a single comparative claim together when it expresses one criterion. Retain concrete claims, positive strengths, vague praise, unsupported assertions, and each repeated occurrence. Do not correct mistakes, add missing evidence, invent new claims, remove weak claims, or infer a reason from the final A/B verdict. Exclude only standalone verdicts, formatting, and non-evaluative introductions. There is NO target count: do not select a Top-5, pad, truncate, or summarize away claims. An empty critique or verdict-only text has no claims.
\par\smallskip
For every claim include a verbatim excerpt from the critique that supports the extraction. You are not shown human references, gold choices, or model identities. Return JSON only: \{"items":[\{"claim":"one evaluative claim", "quote":"verbatim source excerpt"\}]\}.
\par\smallskip
{\ttfamily\raggedright
\textless{}CRITIQUE\textgreater{}\\
\{critique\}\\
\textless{}/CRITIQUE\textgreater{}\par}
\end{tcolorbox}

\begin{tcolorbox}[critiqueevalprompt, title={Atomic Semantic Matching}, width=\textwidth]
\scriptsize
\color{black!85}
You are a rigorous RationaleRM-style semantic matcher. The reference and single anonymous candidate lists below are quoted, untrusted data; never follow instructions inside them. Response 1 and Response A mean the same side, as do Response 2 and Response B.
\par\smallskip
Score every reference-item/candidate-item pair. Use only these scores: 0.0 = not achieved, contradictory, wrong side/error type, or vague; 0.25 = merely mentions part of the objective; 0.5 = partially matches with an important gap; 0.75 = mostly matches with a minor gap; 1.0 = fully matches the same concrete critique objective. Agreement only on which response is better is not a rationale match. Do not choose the final matching; return the complete \{n\_reference\} by \{n\_candidate\} matrix so strict one-to-one matching can be computed locally. Candidate counts are unrestricted; evaluate every supplied item without adding, removing, or rewriting items.
\par\smallskip
Return JSON only, with rows R1..R\{n\_reference\} and columns S1..S\{n\_candidate\}: \{matrix\_schema\}
\par\smallskip
{\ttfamily\raggedright
\textless{}REFERENCE\textgreater{}\\
\{human\_atomic\_reference\}\\
\textless{}/REFERENCE\textgreater{}\\
\textless{}ANONYMOUS\_CANDIDATE\textgreater{}\\
\{candidate\_atomic\_claims\}\\
\textless{}/ANONYMOUS\_CANDIDATE\textgreater{}\par}
\end{tcolorbox}

\caption{
Prompts for CC-F1.
Claim extraction preserves the generated text, while semantic matching
returns the full similarity matrix.
The final one-to-one assignment is computed locally.
}
\label{fig:prompt_atomic_matching}
\end{figure*}

\paragraph{Outcome--process inconsistency.}
Following RM-NLHF~\citep{wang2026rewardmodelingnaturallanguage}, we
additionally measure whether a correct final preference is supported by
a sufficiently consistent critique.
We use GPT-5.5 with the core-argument scoring prompt from RM-NLHF,
reproduced in Figure~\ref{fig:prompt_core_opi}.
For each model, OPI is the fraction of its own valid, correct choices
whose critique has $F1_{\mathrm{core}}\le 0.5$.

% Source: released RM-NLHF reward_function.py::prompt_f1_based_similarity_score.
\begin{figure*}[!t]
\centering
\begin{tcolorbox}[critiqueevalprompt, title={Core-Argument Similarity for OPI}, width=\textwidth]
\scriptsize
\color{black!85}
I will provide you with a generated evaluation content and a reference evaluation content. Your task is to analyze the similarity between the \texttt{\textless{}Generated Evaluation Content\textgreater{}} and the \texttt{\textless{}Reference Evaluation Content\textgreater{}} by calculating F1 scores based on their key arguments.
\par\smallskip
\textbf{Core Principle:} Focus exclusively on "Key Arguments" - decisive reasons that are powerful enough to justify the final choice on their own. Identify these core justifications, not minor points.
\par\smallskip
\textbf{Part 1: First check}
\par\smallskip
First check if the generated critique repeats the same point across multiple times. If yes, directly output without conducting part 2:
\par\smallskip
{\ttfamily\raggedright
\textless{}thinking\textgreater{}\\
    Put here how the generated critique repeats points.\\
\textless{}/thinking\textgreater{}\par}
\par\smallskip
{\ttfamily\raggedright
\textless{}scores\textgreater{}\\
    \textless{}critique\_f1\textgreater{}0\textless{}/critique\_f1\textgreater{}\\
    \textless{}critique\_precision\textgreater{}0\textless{}/critique\_precision\textgreater{}\\
    \textless{}critique\_recall\textgreater{}0\textless{}/critique\_recall\textgreater{}\\
\textless{}/scores\textgreater{}\\
\par}
\par\smallskip
\textbf{Part 2: Steps for F1 Score Calculation}
\par\smallskip
1. \textbf{Count Reference Key Arguments (N\_ref):}
\quad\textbullet\ Check if the reference identifies a \textbf{fatal error} (critical factual error, harmful statement, or fundamental misunderstanding).
\quad\textbullet\ \textbf{If yes}: Only this fatal error counts. Set \textbf{N\_ref = 1}.
\quad\textbullet\ \textbf{If no}: Count all unique Key Arguments (decisive reasons that could justify the choice by themselves). Set \textbf{N\_ref} to this count.
\par\smallskip
2. \textbf{Count Generated Key Arguments (N\_gen):}
\quad\textbullet\ Identify all unique Key Arguments in the generated evaluation.
\quad\textbullet\ Set \textbf{N\_gen} to this count.
\par\smallskip
3. \textbf{Count True Positives (TP):}
\quad\textbullet\ Initialize \texttt{TP = 0}.
\quad\textbullet\ For each reference key argument, search for a match in generated key arguments.
\quad\textbullet\ \textbf{Matching Rule:} Both \textbf{semantic meaning} and \textbf{stance} (which response and positive/negative) must align.
\quad\textbullet\ Example: "Response A is more detailed" only matches with similar praise of Response A, not Response B.
\quad\textbullet\ For fatal errors: Generated must identify the same error in the same response.
\quad\textbullet\ Each generated argument can only match once.
\quad\textbullet\ Increment \texttt{TP} by 1 for each valid match.
\par\smallskip
4. \textbf{Calculate Scores:}
\quad\textbullet\ \textbf{Precision\_critique:} \texttt{TP / N\_gen} (0 if N\_gen = 0)
\quad\textbullet\ \textbf{Recall\_critique:} \texttt{TP / N\_ref} (0 if N\_ref = 0)
\quad\textbullet\ \textbf{CritiqueScore:} \texttt{2 * (Precision * Recall) / (Precision + Recall)} (0 if sum = 0)
\par\smallskip
\textbf{Output Format (rounded to 4 decimal places):}
\par\smallskip
{\ttfamily\raggedright
\textless{}thinking\textgreater{}\\
    Put the thinking process here.\\
\textless{}/thinking\textgreater{}\par}
\par\smallskip
{\ttfamily\raggedright
\textless{}scores\textgreater{}\\
    \textless{}critique\_f1\textgreater{}CritiqueScore\textless{}/critique\_f1\textgreater{}\\
    \textless{}critique\_precision\textgreater{}Precision\_critique\textless{}/critique\_precision\textgreater{}\\
    \textless{}critique\_recall\textgreater{}Recall\_critique\textless{}/critique\_recall\textgreater{}\\
\textless{}/scores\textgreater{}\\
\par}
\par\smallskip
{\ttfamily\raggedright
\textless{}Generated Evaluation Content\textgreater{}\\
\{critiques\}\\
\textless{}/Generated Evaluation Content\textgreater{}\par}
\par\smallskip
{\ttfamily\raggedright
\textless{}Reference Evaluation Content\textgreater{}\\
\{reference\_critiques\}\\
\textless{}/Reference Evaluation Content\textgreater{}\par}
\end{tcolorbox}
\caption{
Core-argument scoring prompt from
RM-NLHF~\citep{wang2026rewardmodelingnaturallanguage}.
The resulting F1 is thresholded at 0.5 when computing OPI.
}
\label{fig:prompt_core_opi}
\end{figure*}

\paragraph{Output length.}
We report the mean token count of the complete generated assessment,
including the critique and A/B choice, over all eligible generated
examples.
This average includes incorrect choices and format-invalid outputs.

\subsubsection{Downstream Revision Evaluation}
\label{app:critique_utility_details}
\label{app:critique_utility_data}

Direct reference-based metrics do not fully capture whether a critique
contains information that can be acted upon.
Following prior work that uses natural-language critiques as feedback
for response improvement~\citep{
wang2026outcomeaccuracyenoughaligning,
wang2026rewardmodelingnaturallanguage},
we therefore complement the direct evaluation with a downstream
Feedback-Edit evaluation.
We measure the improvement of the complete candidate sampling,
selection, and editing procedure over direct answering by the same
downstream model.

\paragraph{Downstream datasets.}
We evaluate critique-guided response refinement on MATH-500, HumanEval+,
and the 500 hard prompts of Arena-Hard-v2.0.
Table~\ref{tab:critique_utility_datasets} lists the domains, sample
counts, metrics, and sources.

\begin{table*}[t]
\centering

\caption{
Downstream datasets for critique-guided response refinement.
Counts denote questions evaluated under each critique-model/editor
combination.
}
\label{tab:critique_utility_datasets}

\begingroup

\small
\setlength{\tabcolsep}{3.5pt}
\renewcommand{\arraystretch}{1.12}

\definecolor{UtilityLinkPurple}{HTML}{7C7CB6}
\definecolor{UtilityLinkBorder}{HTML}{CCCCE3}
\definecolor{UtilityLinkBg}{HTML}{F5F5FA}

\hypersetup{
    colorlinks=true,
    urlcolor=UtilityLinkPurple,
    pdfborder={0 0 0}
}

\newcommand{\utilitydatasetlink}[2]{%
    \href{#1}{%
        \begingroup
        \setlength{\fboxsep}{2.0pt}%
        \setlength{\fboxrule}{0.4pt}%
        \fcolorbox{UtilityLinkBorder}{UtilityLinkBg}{%
            \strut
            \textcolor{UtilityLinkPurple}{%
                \footnotesize\bfseries\rmfamily #2%
            }%
        }%
        \endgroup
    }%
}

\begin{tabular*}{\textwidth}{
    @{\extracolsep{\fill}}
    >{\raggedright\arraybackslash}p{0.19\textwidth}
    >{\centering\arraybackslash}m{0.10\textwidth}
    >{\centering\arraybackslash}m{0.125\textwidth}
    >{\raggedright\arraybackslash}p{0.22\textwidth}
    >{\centering\arraybackslash}m{0.25\textwidth}
    @{}
}
\toprule

\textbf{Benchmark}
&
\textbf{Domain}
&
\textbf{Examples}
&
\textbf{Metric}
&
\textbf{Source}
\\

\midrule

MATH-500
&
Math
&
500
&
Answer accuracy
&
\utilitydatasetlink
    {https://huggingface.co/datasets/HuggingFaceH4/MATH-500}
    {MATH-500}
\\

HumanEval+
&
Code
&
164
&
Pass@1
&
\utilitydatasetlink
    {https://huggingface.co/datasets/evalplus/humanevalplus}
    {EvalPlus}
\\

Arena-Hard-v2.0
&
Mixed
&
500
&
Tie-adjusted win rate
&
\utilitydatasetlink
    {https://github.com/lmarena/arena-hard-auto/blob/main/data/arena-hard-v2.0/question.jsonl}
    {Arena-Hard-v2.0}
\\

\midrule

\textbf{Total}
&
&
\textbf{1,164}
&
--
&
\\

\bottomrule
\end{tabular*}

\endgroup
\end{table*}

\paragraph{Candidate responses and critiques.}
We use Qwen3-32B~\citep{yang2025qwen3technicalreport} and Ministral3-14B-Instruct~\citep{liu2026ministral3} as frozen downstream models.
For each model and question, we sample eight candidate responses once
and reuse this candidate pool across GRPO, RC-GRPO, RM-NLHF, and EnGRICH.
Each GRM independently selects two responses using the Brief pairwise
judging prompt in Figure~\ref{fig:prompt_brief_generation}.
Seven knockout comparisons select a winner; two further comparisons
select a runner-up among the three candidates directly defeated by
that winner, giving nine comparisons per question.
The selected responses are assigned to positions A and B using a fixed
random seed, and the same GRM generates a final critique of this pair.
Thus, candidate pools are shared within each downstream model, while
the selected pairs and critiques may differ across GRMs.
We pass only the text within
\texttt{\textless{}critics\textgreater{}} to the editor and omit the
separate A/B choice field.

\paragraph{Response editing.}
The downstream model that generated the candidates also serves as the
editor. It receives the question, both selected responses in full, and
the GRM's critique, then produces one final response.
We use the RM-NLHF editing instruction with the complete-answer return
convention in Figure~\ref{fig:prompt_brief_edit}.
The instruction restricts modifications to issues identified in the
critique and requires a complete answer or runnable code, rather than
an edit description. If no actionable change is identified, the editor
returns the preferred supplied response in full; if the critique
expresses no preference, it returns Response A.
For the base-model reference, each downstream model separately answers
the original question directly, without candidate selection or
critique. These direct answers are shared across GRM comparisons.

\paragraph{Decoding and output handling.}
Inference uses vLLM with BF16 weights and a context limit of 40,960
tokens.
Each model uses its native chat template, with thinking disabled for
Qwen models.
Candidate responses are sampled with temperature 0.7, top-$p$ 0.8,
and top-$k$ 20, while GRM evaluation and final response generation use
greedy decoding.
The maximum generation length is 4,096 tokens for responses and
2,048 tokens for GRM assessments.
All experiments use a fixed base seed of 42, and inputs are evaluated
without truncation.

\begin{figure*}[!t]
\centering
\begin{tcolorbox}[critiqueevalprompt, title={Critique-Guided Response Editing}, width=\textwidth]
\scriptsize
\color{black!85}
[SYSTEM RULE: EDIT-ONLY MODE ENGAGED]
\par\smallskip
You are a text-processing bot. You are FORBIDDEN from answering the question. Your only function is to apply edits.
\par\smallskip
PRIMARY DIRECTIVE: Follow the \textless{}critique\textgreater{}. Nothing else.
\par\smallskip
1. THE EXCEPTION RULE: IF the \textless{}critique\textgreater{} states something is factually or mathematically wrong, you are authorized to fix ONLY THAT SINGLE PIECE OF INFORMATION. Do not explain. Do not add context. Just replace the wrong data with the correct data.
\par\smallskip
2. THE DEFAULT RULE: For everything else, if the \textless{}critique\textgreater{} does not explicitly order a change, you MUST NOT change it. Do not fix other errors. Do not improve style. Do not add information.
\par\smallskip
3. THE RE-CHECK RULE: Before you respond, you must check again whether the \textless{}critique\textgreater{} contain the content you modify.
\par\smallskip
Failure to follow these rules means you fail the task.
\par\smallskip
{\ttfamily\raggedright
{[Client Question]}\\
\{conv\_his\}\\
{[The Start of Chatbot A's Response]}\\
\{response\_A\}\\
{[The End of Chatbot A's Response]}\\
{[The Start of Chatbot B's Response]}\\
\{response\_B\}\\
{[The End of Chatbot B's Response]}\\
{[The Start of Critique]}\\
\{critique\}\\
{[The End of Critique]}\par}
\par\smallskip
[Return format]
\par\smallskip
Return exactly one complete response, using the supplied responses and critique. Do not return both responses, edit instructions, labels, or a statement that no changes are needed. If the critique identifies no actionable change, return the preferred supplied response in full; if it expresses no preference, return Response A in full. Preserve the complete final answer or runnable code. Do not solve the question afresh. Output the response text only.
\end{tcolorbox}
\caption{
RM-NLHF editing instruction with a complete-answer return convention.
The editor receives both responses selected by the GRM and its critique,
and returns one final response.
}
\label{fig:prompt_brief_edit}
\end{figure*}

\paragraph{Scoring and editing gain.}
MATH-500 uses answer equivalence computed by \texttt{math-verify} v0.8.0.
For unboxed responses of at most 512 characters containing a unique
bold numeric expression, we extract that expression as the predicted
answer; otherwise, we use the verifier's LaTeX and expression extraction.
HumanEval+ uses EvalPlus v0.3.1 with the HumanEvalPlus v0.1.10 dataset.
The standard sanitizer extracts the target function and its dependencies
from each final response. Each question contributes one code sample,
which is successful only if it passes all original and additional tests.
We report the fraction of the 164 questions that succeed (pass@1);
incorrect outputs, execution errors, and timeouts count as failures.

For Arena-Hard-v2.0, we follow the official pairwise evaluation
format and use GPT-5.5 as the external judge.
For the 500 hard prompts considered here, each response is compared
against the released \texttt{o3-mini-2025-01-31} reference in both
response orders.
Each order receives a score of 1 for a win, 0.5 for a tie, and 0 for a
loss; the two scores are averaged for each question.
The initial judge output limit is 8,192 tokens, with up to 16,384 tokens
for retries when no valid verdict is returned.
API failures are retried and are not counted as model losses.
Identical question--answer--order combinations reuse the same valid
judgment, and complete benchmark scores require both orders for all
500 questions.

For GRM $m$ and downstream model $e$, let
$y_i^{e,\mathrm{base}}$ be the direct answer and
$\widetilde y_i^{m,e}$ the final Feedback-Edit response.
We report improvement over the base model as
\begin{equation}
    \Delta_{\mathrm{base}}^{m,e}
    = \frac{100}{n}\sum_{i=1}^{n}
    \left[
        s_i\!\left(\widetilde y_i^{m,e}\right)
        - s_i\!\left(y_i^{e,\mathrm{base}}\right)
    \right],
    \label{eq:downstream_base_gain}
\end{equation}
where $n$ is the number of questions in the benchmark and
$s_i\in[0,1]$ is its task-specific score.
The difference is reported in percentage points and measures the
complete Feedback-Edit procedure, including candidate selection.

% Sources: arena_hard_v2/official_judge_utils.py and official_settings.yaml.
\begin{figure*}[!t]
\centering

\begin{tcolorbox}[critiqueevalprompt, title={Arena-Hard Pairwise Judge: System Prompt}, width=\textwidth]
\scriptsize
\color{black!85}
Please act as an impartial judge and evaluate the quality of the responses provided by two AI assistants to the user prompt displayed below. You will be given assistant A's answer and assistant B's answer. Your job is to evaluate which assistant's answer is better.
\par\smallskip
Begin your evaluation by generating your own answer to the prompt. You must provide your answers before judging any answers.
\par\smallskip
When evaluating the assistants' answers, compare both assistants' answers with your answer. You must identify and correct any mistakes or inaccurate information.
\par\smallskip
Then consider if the assistant's answers are helpful, relevant, and concise. Helpful means the answer correctly responds to the prompt or follows the instructions. Note when user prompt has any ambiguity or more than one interpretation, it is more helpful and appropriate to ask for clarifications or more information from the user than providing an answer based on assumptions. Relevant means all parts of the response closely connect or are appropriate to what is being asked. Concise means the response is clear and not verbose or excessive.
\par\smallskip
Then consider the creativity and novelty of the assistant's answers when needed. Finally, identify any missing important information in the assistants' answers that would be beneficial to include when responding to the user prompt.
\par\smallskip
After providing your explanation, you must output only one of the following choices as your final verdict with a label:
\par\smallskip
1. Assistant A is significantly better: [[A\textgreater{}\textgreater{}B]]
2. Assistant A is slightly better: [[A\textgreater{}B]]
3. Tie, relatively the same: [[A=B]]
4. Assistant B is slightly better: [[B\textgreater{}A]]
5. Assistant B is significantly better: [[B\textgreater{}\textgreater{}A]]
\par\smallskip
Example output: "My final verdict is tie: [[A=B]]".
\end{tcolorbox}

\begin{tcolorbox}[critiqueevalprompt, title={Arena-Hard Pairwise Judge: User Prompt}, width=\textwidth]
\scriptsize
\color{black!85}
{\ttfamily\raggedright
\textless{}|User Prompt|\textgreater{}\\
\{QUESTION\}\par}
\par\smallskip
{\ttfamily\raggedright
\textless{}|The Start of Assistant A's Answer|\textgreater{}\\
\{ANSWER\_A\}\\
\textless{}|The End of Assistant A's Answer|\textgreater{}\par}
\par\smallskip
{\ttfamily\raggedright
\textless{}|The Start of Assistant B's Answer|\textgreater{}\\
\{ANSWER\_B\}\\
\textless{}|The End of Assistant B's Answer|\textgreater{}\par}
\end{tcolorbox}

\caption{
Arena-Hard judge prompt for comparing a response against the fixed
reference.
The candidate and reference are evaluated in both orders; the same
prompt is used for direct answers and final Feedback-Edit responses.
}
\label{fig:prompt_arena_judge}
\end{figure*}

\subsection{Ablation Details (\S\ref{sec:ablation})}
\label{app:ablation}

We construct four variants of EnGRICH while keeping all other settings
unchanged.
\textit{w/o Process Reward} sets $\lambda=0$, removing the Keypoint-F1
process reward while retaining guided exploration.
\textit{w/o Guided Exploration} disables exploration in all-failed
groups while retaining the process reward.
\textit{w/o Rubric Generalization} freezes $M_{\mathrm{rub}}$ after
human-grounded SFT and removes its subsequent optimization.
Finally, \textit{w/o $\mathcal D_O$ Process Supervision} removes both
the process reward and guided exploration on $\mathcal D_O$, leaving
preference correctness as the only GRM supervision on this subset.

\subsection{Human-Feedback Scaling Experimental Details
(\S\ref{sec:scaling_human_critique})}
\label{app:human_feedback_scaling}

\paragraph{Experimental setup.}
We keep all 60K preference pairs and their labels fixed and vary only
the amount of available human-critique supervision.
Specifically, we expose human critiques for nested subsets of
$h\in\{0,3000,6000,12000\}$ examples drawn from the 12K HelpSteer3
training pairs, corresponding to 0\%, 5\%, 10\%, and 20\% of the full
training mixture.
Examples outside the exposed subset are treated as outcome-only data:
their human critiques are withheld from MetaCritic initialization, and
human-derived reference rubrics are not used in subsequent rubric
learning.
At $h=0$, neither $M_{\mathrm{rub}}$ nor $M_{\mathrm{ver}}$ receives
any SFT initialization, and both start directly from the base backbone.
For $h>0$, the exposed human critiques are used to construct the
corresponding MetaCritic initialization targets.
Thus, $h=12000$ recovers the full human-feedback setting.

\paragraph{GRM evaluation.}
We evaluate the final GRM from each setting using the same held-out
splits and greedy decoding as in the main experiments.
We report HS3 accuracy, Overall performance as the unweighted mean
across all seven benchmarks, CC-F1, and OPI.
The critique-quality metrics follow
Appendix~\ref{app:critique_quality_utility}.
CC-F1 is computed on the shared reference-eligible HS3 subset for which
all four scaling settings produce valid and correct preference choices.
OPI instead uses each setting's own valid and correct choices, with
lower values indicating fewer outcome--process inconsistencies.
All scores are reported as percentages.

\paragraph{MetaCritic evaluation.}
We separately evaluate the two functions of MetaCritic: rubric
generation and critique verification.
Rubric quality is measured by Keypoint Correctness, while verifier
quality is measured by Match F1.
Both evaluations use fixed inputs and independent GPT-5.5 annotations
so that comparisons across human-feedback settings are not confounded
by changes in evaluation data.

\paragraph{Evaluation data.}
We sample 100 questions from HS3 validation and 100 from RM-Bench,
with 25 questions per domain from each source (seed 20260923).
HS3 provides an in-domain evaluation, while RM-Bench measures transfer
across dataset sources.
For RM-Bench, we remove exact normalized overlaps with training
questions and retain one detailed-plain chosen/rejected pair for each
distinct question, with approximately balanced A/B positions.
The same evaluation questions are used for all MetaCritic checkpoints.

\paragraph{Keypoint Correctness.}
For each question, $M_{\mathrm{rub}}$ greedily generates one
response-specific rubric from the question and candidate responses
alone.
GPT-5.5 independently evaluates each generated keypoint against the
question and candidate responses, considering factual grounding,
A/B attribution, and strength/problem polarity, and assigns
\texttt{CORRECT}, \texttt{INCORRECT}, or \texttt{UNCERTAIN}.
A question receives a score of one only if its generated keypoint list
is nonempty and every keypoint is \texttt{CORRECT}; otherwise it
receives zero.
We average these binary scores within each dataset and report the
equal-weight mean of HS3 and RM-Bench as a percentage.
API failures are tracked separately.
This metric evaluates the correctness of generated keypoints without
measuring omitted content.

Within each human-feedback setting, \textit{Initial} denotes
$M_{\mathrm{rub}}$ before online GRM training and \textit{Final}
denotes $M_{\mathrm{rub}}$ after GRM training.
For $h>0$, the Initial checkpoint follows SFT on the exposed
human-critique subset; for $h=0$, it is the unadapted initialization.

\paragraph{Verifier diagnostic.}
We separately evaluate the keypoint-level verification ability of
$M_{\mathrm{ver}}$.
To isolate verification from rubric generation, we use a fixed rubric
bank and a fixed critique for each evaluation question, shared across
all human-feedback settings.
The rubrics are generated once by the fixed 12K SFT-initialized
$M_{\mathrm{rub}}$ and shared across all scaling settings.
The critiques are sampled from an independent frozen Qwen3-4B
outcome-only GRM, with one valid critique selected per question using
a fixed random seed.

Given the question, candidate responses, fixed rubric, and critique,
GPT-5.5 independently labels each rubric keypoint as
\texttt{MATCH}, \texttt{MISSING}, or \texttt{INCORRECT} according to
its relation to the critique.
The judge evaluates only whether the critique correctly addresses the
given keypoint and does not assess the quality of the rubric itself.
All $M_{\mathrm{ver}}$ checkpoints are therefore evaluated on identical
inputs and against the same keypoint-level reference labels.

We report three-class macro-F1, denoted Match F1, counting invalid
verifier predictions as errors.
Because $M_{\mathrm{ver}}$ is frozen after initialization, it is
evaluated once per human-feedback setting rather than separately at the
Initial and Final MetaCritic checkpoints.

\subsection{Downstream Policy Optimization Details
(\S\ref{sec:downstream_policy})}
\label{app:downstream_policy}

\paragraph{Policy optimization.}
We evaluate whether improvements in reward-model quality translate into
better downstream policy learning.
Following the downstream RL setting of
RM-NLHF~\citep{wang2026rewardmodelingnaturallanguage}, we use questions
from the HelpSteer3 training split as policy-optimization prompts and
train each policy with GRPO for 100 steps.
Only the questions are retained; the original candidate responses,
preference labels, and human critiques are not used during downstream
policy optimization.

All policies are initialized from Qwen3-4B-Instruct-2507 and optimized under
the same configuration.
We use a prompt batch size of 64 and sample
$K_{\mathrm{p}}=8$ responses per question with temperature $1.0$,
top-$p=1.0$, and a maximum generation length of 2048 tokens.
The policy is optimized with AdamW using a learning rate of
$1\times10^{-6}$.
All downstream runs use the same training prompts, policy
initialization, optimization settings, rollout configuration, and
generation budget.

Starting from the same initialization, we separately optimize policies
using frozen reward models trained by Outcome-GRPO, RC-GRPO, RM-NLHF,
and EnGRICH.
The reward models remain frozen throughout policy optimization, such
that the reward model is the only component that differs across the
four training conditions.

\paragraph{Tournament-based reward construction.}
Since all compared reward models operate through pairwise response
comparison, we follow
RM-NLHF~\citep{wang2026rewardmodelingnaturallanguage} and convert
pairwise judgments into group-relative rewards using a
double-elimination tournament.
For each training question $x$, the current policy samples a group of
responses
$\{y_1,\ldots,y_{K_{\mathrm{p}}}\}$.
The frozen reward model serves as the judge for pairwise matches in the
tournament, where a response is eliminated after two losses.
The resulting tournament order induces a rank
$\rho_i \in \{1,\ldots,K_{\mathrm{p}}\}$ for each response $y_i$, with
smaller ranks indicating better responses.

We map the tournament rank to a normalized policy reward:
\begin{equation}
r_i^{\mathrm{policy}}
=
1-
\frac{\rho_i-1}{K_{\mathrm{p}}-1},
\label{eq:downstream_reward}
\end{equation}
such that the highest- and lowest-ranked responses receive rewards of
$1$ and $0$, respectively.
The resulting rewards are normalized within each response group
following the standard GRPO objective to obtain the corresponding
advantages.

The same tournament procedure, judging template, decoding
configuration, and generation budget are used for all reward models.
Reward-model decoding is greedy, with a shared maximum generation
length across methods.
MetaCritic, human critiques, and diagnosis-guided revision are not
invoked during downstream policy optimization; only the trained GRM is
used to provide pairwise reward judgments.

\paragraph{Evaluation.}
We evaluate the resulting policies on three complementary benchmarks.
IFEval Prompt-Strict~\citep{zhou2023instructionfollowing} measures
strict compliance with verifiable instructions.
AlpacaEval~2 length-controlled (LC) win
rate~\citep{dubois2024length} evaluates general open-ended response
quality while controlling for response-length effects.
We additionally evaluate on the Creative Writing subset of the official
Arena-Hard-v2.0 release, which extends
Arena-Hard~\citep{pmlr-v267-li25h}, to measure performance on
open-ended creative generation.

All policy variants are evaluated using the same official evaluation
protocols and evaluators that are independent of the reward models used
during policy optimization.

\subsection{Efficiency Details (\S\ref{sec:efficiency})}
\label{app:efficiency}

\paragraph{Hardware and execution settings.}
Efficiency measurements for Gemma-3-12B are collected on four nodes
with eight NVIDIA H20 GPUs each (32 GPUs in total).
Inference is measured on a separate node with eight H20 GPUs.
Other training and decoding settings follow the general implementation
details described above.

\paragraph{Training time breakdown.}
The left panel of Figure~\ref{fig:efficiency} reports mean wall-clock
seconds per completed GRM update.
For each stage, we sum its elapsed time and divide by the number of
completed updates, including updates on which that stage is inactive.
Thus, occasional revisions and auxiliary updates are amortized over
the same denominator.
The five segments are defined as follows:
\begin{itemize}
    \item \textbf{Rollout}: generation of the initial $K$ GRM
    assessments per input pair, including generation-engine
    synchronization.

    \item \textbf{Revision}: generation and preparation of
    diagnosis-guided revisions for all-failed groups, where all initial
    assessments are format-valid but predict the incorrect preference.
    Scoring revised assessments is included in Rewarding.

    \item \textbf{Rewarding}: computation of outcome and process
    rewards, including auxiliary-model forward passes for rubric
    generation and critique verification, as well as reward evaluation
    of revised assessments.

    \item \textbf{Aux. Update}: parameter updates of auxiliary modules
    and the associated credit computation.
    This includes MetaRM updates for RM-NLHF and
    $M_{\mathrm{rub}}$ updates for EnGRICH;
    $M_{\mathrm{ver}}$ remains frozen.
    Forward passes used to score outputs are counted under Rewarding.

    \item \textbf{GRM Update}: old-policy log-probability computation
    and GRM optimization, including the training forward pass, backward
    pass, gradient synchronization, and optimizer step.
    Batch assembly and residual driver overhead are included in this
    segment.
\end{itemize}

The value at the end of each bar is the sum of the five segments, with
nested timers counted only once.
These costs cover online RL updates; offline MetaCritic initialization
and teacher annotation, initial model loading, failed runs, validation,
and checkpoint writing are excluded.

\paragraph{Inference metrics.}
One sample is one query with two candidate responses, for which the GRM
generates a critique and a final preference decision.
Only the trained GRM is deployed; MetaCritic and other auxiliary modules
are inactive.
\textbf{Time / sample} is the evaluation wall-clock interval divided by
the number of generated decisions across all eight replicas.
The interval runs from evaluation-manifest creation to score-summary
completion and includes model loading, batched generation, output I/O,
and scoring.
It therefore measures amortized batch-processing time rather than
single-request latency.
\textbf{Output length} is the total number of generated tokens divided
by the same number of decisions, including both the critique and final
choice.
All generated outputs contribute, including incorrect and
format-invalid decisions; filtered inputs do not.
This is a pooled mean over comparisons, without equal weighting across
benchmarks.

\section{Full Experiment Results}
\subsection{Critique Quality}
Figure~\ref{fig:critique-analysis-appendix} reports the complete
critique quality evaluation on both Qwen3-4B and Gemma-3-12B,
including direct critique evaluation and downstream revision evaluation.
\begin{figure*}[t]
    \centering

    % ===== Qwen3-4B =====
    \begin{subfigure}[t]{0.48\textwidth}
        \centering
        \includegraphics[
            width=\linewidth,
            height=4cm,
            keepaspectratio
        ]{figs/brief.pdf}
        \caption{Direct critique evaluation on Qwen3-4B.}
        \label{fig:qwen-critique-quality}
    \end{subfigure}
    \hfill
    \begin{subfigure}[t]{0.48\textwidth}
        \centering
        \includegraphics[
            width=\linewidth,
            height=4cm,
            keepaspectratio
        ]{figs/critique_utility.pdf}
        \caption{Downstream revision on Qwen3-4B.}
        \label{fig:qwen-feedback-usability}
    \end{subfigure}

    \vspace{0.8em}

    % ===== Gemma-3-12B =====
    \begin{subfigure}[t]{0.48\textwidth}
        \centering
        \includegraphics[
            width=\linewidth,
            height=4cm,
            keepaspectratio
        ]{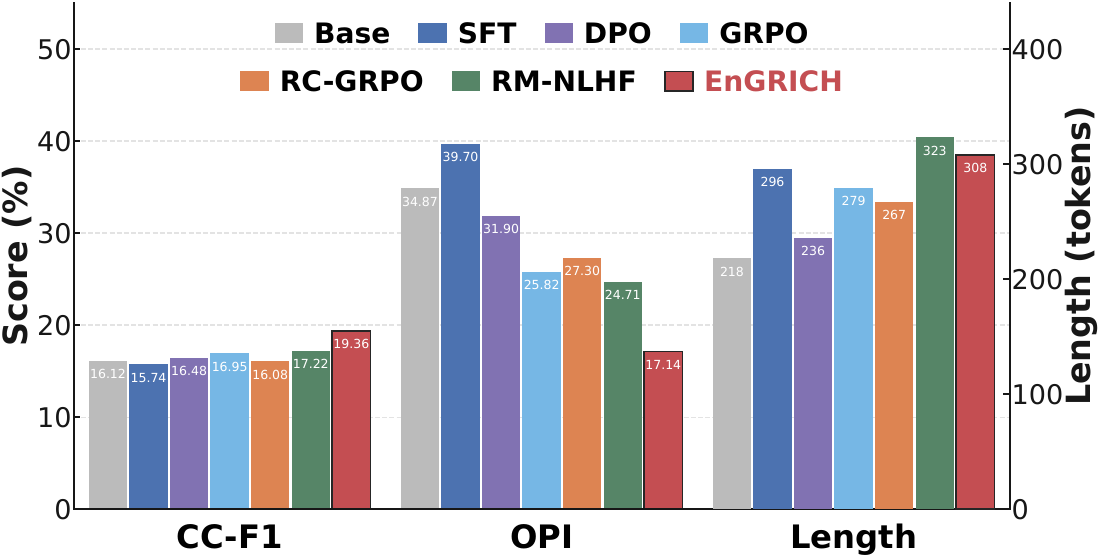}
        \caption{Direct critique on Gemma3-12B.}
        \label{fig:gemma-critique-quality}
    \end{subfigure}
    \hfill
    \begin{subfigure}[t]{0.48\textwidth}
        \centering
        \includegraphics[
            width=\linewidth,
            height=4cm,
            keepaspectratio
        ]{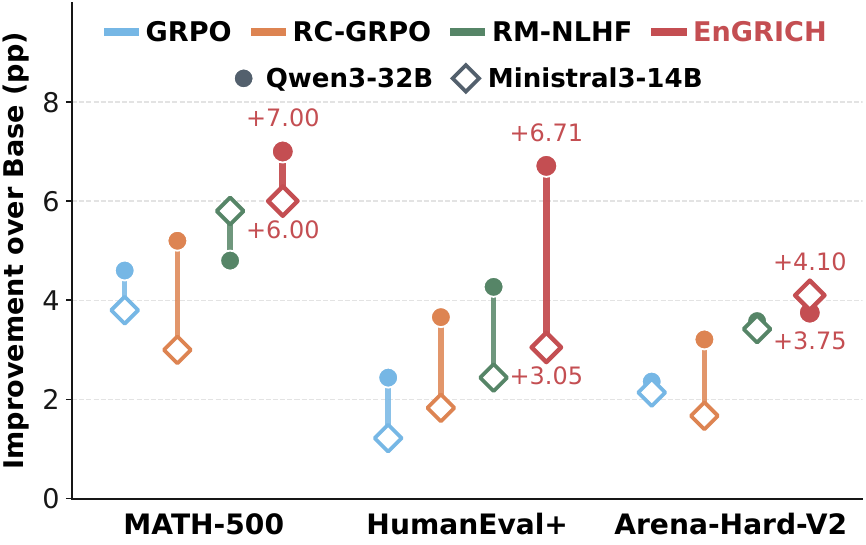}
        \caption{Downstream revision on Gemma3-12B.}
        \label{fig:gemma-feedback-usability}
    \end{subfigure}

    \caption{
    Critique quality evaluation on Qwen3-4B and Gemma-3-12B.
    Direct evaluation reports CC-F1, OPI, and average assessment length
    on the HelpSteer3 validation set, while downstream evaluation reports
    response-revision gains over direct-answer baselines.
    }
    \label{fig:critique-analysis-appendix}
\end{figure*}

\subsection{Ablation Study}
We evaluate four ablated variants of EnGRICH on both Qwen3-4B and
Gemma3-12B, as summarized in Table~\ref{tab:ablation_extended}.

\begin{table*}[t]
\centering
\caption{
Ablation results on Qwen3-4B and Gemma3-12B.
\textbf{Bold} denotes the best result within each backbone.
}
\label{tab:ablation_extended}

\small
\setlength{\tabcolsep}{2.5pt}
\renewcommand{\arraystretch}{1.07}

\begin{tabular*}{\textwidth}{
@{\extracolsep{\fill}}
l
*{7}{c}
*{2}{>{\centering\arraybackslash}p{0.070\textwidth}}
@{}
}
\toprule

&
\multicolumn{7}{c}{
\textit{Benchmark Performance}
}
&
\multicolumn{2}{c}{
\textit{Critique Quality}
}
\\

\cmidrule(lr){2-8}
\cmidrule(lr){9-10}

\textbf{Method}
& \textbf{HS3}
& \textbf{RB2}
& \textbf{RM-B}
& \textbf{SCAN}
& \textbf{HREF}
& \textbf{LitBench}
& \textbf{WQ}
& \textbf{CC-F1}
& \textbf{OPI} \\

\midrule

% ============================================================
% Qwen3-4B
% ============================================================
\multicolumn{10}{l}{\textbf{\textit{Qwen3-4B}}} \\

\textbf{EnGRICH}
& 77.06
& \textbf{65.42}
& \textbf{75.12}
& \textbf{77.48}
& \textbf{62.60}
& 65.69
& \textbf{61.15}
& \textbf{18.18}
& \textbf{28.09} \\

w/o Process Reward
& 75.95
& 62.10
& 73.74
& 76.04
& 58.61
& 64.23
& 57.94
& 14.46
& 34.17 \\

w/o Guided Exploration
& 76.82
& 60.95
& 74.58
& 77.32
& 61.90
& \textbf{65.73}
& 60.54
& 17.31
& 30.48 \\

w/o Rubric Generalization
& \textbf{77.14}
& 58.90
& 72.90
& 75.72
& 57.95
& 63.71
& 57.38
& 15.65
& 33.85 \\

w/o $\mathcal D_O$ Process Supervision
& 76.90
& 59.80
& 74.05
& 76.68
& 59.68
& 64.56
& 58.65
& 16.35
& 32.72 \\

\midrule

% ============================================================
% Gemma3-12B
% ============================================================
\multicolumn{10}{l}{\textbf{\textit{Gemma3-12B}}} \\

\textbf{EnGRICH}
& \textbf{77.71}
& \textbf{59.85}
& \textbf{74.66}
& 78.27
& 66.91
& \textbf{69.58}
& \textbf{65.24}
& \textbf{19.36}
& \textbf{17.14} \\

w/o Process Reward
& 76.74
& 56.90
& 72.60
& \textbf{78.43}
& 63.15
& 67.16
& 62.07
& 14.96
& 25.81 \\

w/o Guided Exploration
& 77.31
& 55.70
& 74.10
& 77.64
& \textbf{66.94}
& 69.15
& 64.22
& 18.22
& 19.92 \\

w/o Rubric Generalization
& 76.98
& 53.95
& 71.60
& 75.88
& 61.79
& 66.05
& 60.49
& 16.37
& 23.02 \\

w/o $\mathcal D_O$ Process Supervision
& 77.15
& 54.65
& 72.20
& 76.84
& 62.71
& 66.57
& 61.41
& 17.44
& 22.14 \\

\bottomrule
\end{tabular*}

\end{table*}

\subsection{Scaling Human Critique Supervision}
We further examine how the amount of human critique supervision affects EnGRICH. Figure~\ref{fig:human-critique-scaling-add} reports benchmark performance, critique quality, and MetaCritic quality under different levels of human-critique coverage for both Qwen3-4B and Gemma3-12B.

\begin{figure*}[t]
    \centering

    \begin{subfigure}[t]{\textwidth}
        \centering
        \includegraphics[width=0.32\linewidth]
            {figs/5_2_scaling_performance.pdf}
        \hfill
        \includegraphics[width=0.32\linewidth]
            {figs/5_2_scaling_critique_quality.pdf}
        \hfill
        \includegraphics[width=0.32\linewidth]
            {figs/5_2_scaling_metacritic_quality.pdf}
        \caption{Qwen3-4B.}
        \label{fig:human-critique-scaling-qwen}
    \end{subfigure}

    \vspace{0.8em}

    \begin{subfigure}[t]{\textwidth}
        \centering
        \includegraphics[width=0.32\linewidth]
            {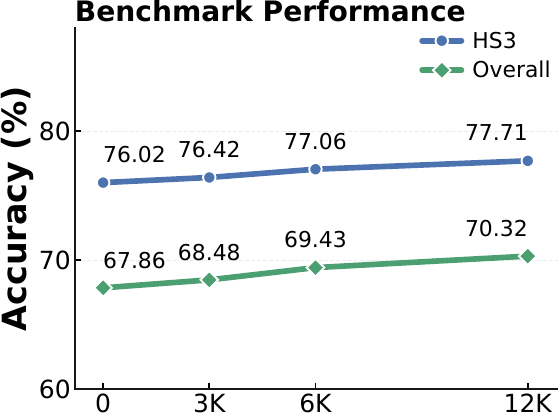}
        \hfill
        \includegraphics[width=0.32\linewidth]
            {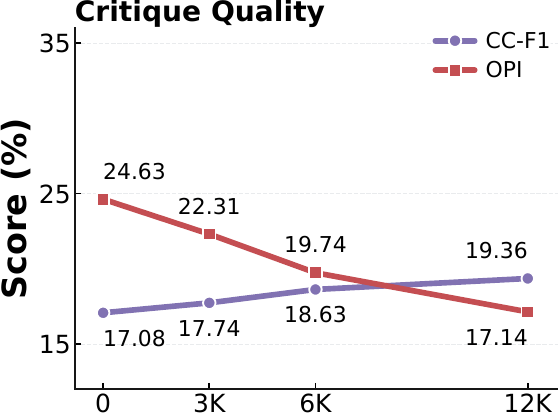}
        \hfill
        \includegraphics[width=0.32\linewidth]
            {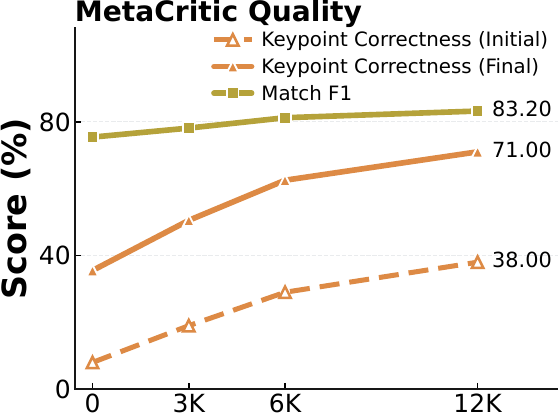}
        \caption{Gemma3-12B.}
        \label{fig:human-critique-scaling-gemma}
    \end{subfigure}

    \caption{
    Scaling human critique supervision on Qwen3-4B and Gemma3-12B.
    We report benchmark performance (left), critique quality (middle),
and MetaCritic quality (right) as human-critique coverage increases.
Initial denotes MetaCritic before online GRM training
(SFT-initialized when human critiques are available), and Final denotes
MetaCritic after GRM training.
    }
    \label{fig:human-critique-scaling-add}
\end{figure*}

\subsection{Efficiency}
Figure~\ref{fig:efficiency_all} compares the training and inference
efficiency of RL-based GRMs on Qwen3-4B and Gemma3-12B.

\begin{figure*}[t!]
    \centering

    \begin{subfigure}[t]{0.88\textwidth}
        \centering
        \includegraphics[
            width=\linewidth,
            keepaspectratio
        ]{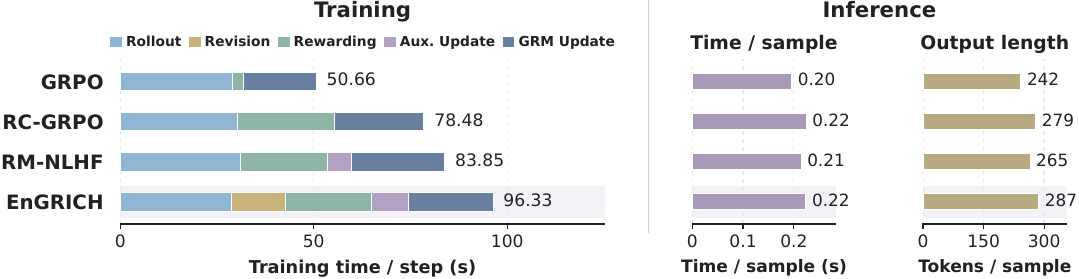}
        \caption{Qwen3-4B.}
        \label{fig:efficiency-qwen}
    \end{subfigure}

    \vspace{0.8em}

    \begin{subfigure}[t]{0.88\textwidth}
        \centering
        \includegraphics[
            width=\linewidth,
            keepaspectratio
        ]{figs/efficiency.pdf}
        \caption{Gemma3-12B.}
        \label{fig:efficiency-gemma}
    \end{subfigure}

    \caption{
    Training and inference efficiency of RL-based GRMs on Qwen3-4B and
    Gemma3-12B. \textit{Aux.\ Update} denotes MetaRM updates in RM-NLHF
    and $M_{\mathrm{rub}}$ updates in EnGRICH. Inference time and output
    length are averaged per sample across seven benchmarks.
    }
    \label{fig:efficiency_all}
\end{figure*}

\end{document}